\documentclass[11pt]{amsart}

\usepackage[margin=1in]{geometry}

\usepackage{setspace}
\usepackage{subfiles}

\usepackage{amstext}
\usepackage{amsthm}
\usepackage{amssymb}
\usepackage[colorlinks=true, linkcolor=blue, citecolor=green, urlcolor=cyan]{hyperref}
\newcommand{\nocolorref}[2]{\hypersetup{linkcolor=black}\hyperref[#1]{#2}\hypersetup{linkcolor=blue}}

\usepackage{tikz}
\usetikzlibrary{mindmap}

\usepackage{subfigure}

\usepackage{ulem}

\usepackage[linesnumbered,ruled]{algorithm2e}

\newtheorem{theor}{Theorem}[section]
\newtheorem*{theorem*}{Theorem}

\newtheorem{assump}[theor]{Assumption}
\newtheorem{lemma}[theor]{Lemma}
\newtheorem*{lemma*}{Lemma}
\newtheorem{cor}[theor]{Corollary}
\newtheorem{defi}[theor]{Definition}

\theoremstyle{definition}
\newtheorem{rem}[theor]{Remark}

\theoremstyle{plain}

\newcommand{\R}{\mathbb{R}}

\usepackage{xcolor}
\definecolor{b}{HTML}{4472c4}
\definecolor{o}{HTML}{ED7D31}
\definecolor{g}{HTML}{70ad47}
\definecolor{t}{RGB}{40,154,150}

\def\dist{{\rm dist}}

\def\R{{\mathbb R}}

\newcommand{\Var}{{\rm Var}}

\def\dist{{\rm dist}}

\date{\today}

\def\cal{\mathcal}

\newcommand{\bbE}{\mathbb{E}}

\newcommand{\bfV}{\mathbf{V}}
\newcommand{\bfU}{\mathbf{U}}

\newcommand{\rmp}{\mathrm{p}}
\newcommand{\rmq}{\mathrm{q}}
\newcommand{\wX}{\widetilde{X}}
\newcommand{\rmrp}{\mathrm{r}_{\mathrm{p}}}

\usepackage{pifont}

\newcommand{\os}{\curlyvee} 

\newcommand{\step}[1]{\medskip  \noindent \underline{#1.\,}}

\newcommand{\dm}[1]{\varpi_{#1}} 

\newcommand{\rinj}{r_{\rm inj}} 
\newcommand{\ang}[4]{ \measuredangle^{#1}({#2}^{#3}_{#4}) } 
\newcommand{\D}[2]{|{#1}\,{#2}|} 
\newcommand{\Deta}{\Delta}

\newcommand{\tv}[1]{\log_p(#1)}

\newcommand{\MAssump}{Let $(M,g)$ be a Riemannian manifold whose sectional curvatures are uniformly bounded in absolute value by a constant $\kappa > 0$.\,}

\newcommand{\cn}[2]{\nocolorref{_eq:def_normalized_number_of_common_neighbors}{{\rm N}_{#1}(#2)}} 

\newcommand{\acn}[2]{\nocolorref{_def: normalized_number_of_common_neighbors}{{\rmp}_{#1}(#2)}}

\newcommand{\Ept}[1]{\nocolorref{lem:Ept}{{\cal E}_{\tr{pt}}(#1)}} 
 
\newcommand{\Enavi}[2]{\nocolorref{_lem: navi}{{\cal E}_{\tr{navi}}(#1,#2)}}

\newcommand{\rG}{{\rm r}_G}
\newcommand{\rM}{{\rm r}_M}

\newcommand{\fe}[1]{\nocolorref{_eq: def_fe}{\varepsilon_{#1}}}

\renewcommand{\sp}{\mathsf{s}_n}

\newcommand{\tr}[1]{\text{\tiny{\rm #1}}} 
\newcommand{\tref}[1]{\text{\tiny{\ref{#1}}}}

\newcommand{\Evt}[1]{\mathcal{E}_{\tiny \mathrm{#1}}} 

\newcommand{\ring}[1]{\mathcal{R}_{#1}} 
\newcommand{\meso}{\mathcal{P}_{\tiny \mathrm{meso}}} 
\newcommand{\tar}{\mathcal{P}_{\tiny \mathrm{target}}}

\newcommand{\Ksg}{{\rm K}_{\mathrm{sg}}} 
\newcommand{\xiq}{\xi_{\frak q}} 
\newcommand{\Cs}[1]{{\mathcal C}_{#1}}

\author{Han Huang, Pakawut Jiradilok, and Elchanan Mossel}
\date{\today}

\address{Department of Mathematics, University of Missouri, Columbia, MO 65203}
\email[H.~Huang]{hhuang@missouri.edu}

\address{Science Division, Mahidol University International College, Nakhon Pathom 73170, Thailand}
\email[P.~Jiradilok]{pakawut.jir@mahidol.edu}

\address{Department of Mathematics, Massachusetts Institute of Technology, Cambridge, MA 02139}
\email[E.~Mossel]{elmos@mit.edu}

\title{Denoising distances beyond the volumetric barrier}

\begin{document}
\maketitle 

\begin{abstract}
We study the problem of reconstructing the latent geometry of a $d$-dimensional Riemannian manifold from a random geometric graph. While recent works have made significant progress in manifold recovery from random geometric graphs, and more generally from noisy distances, the precision of pairwise distance estimation has been fundamentally constrained by the volumetric barrier, namely the natural sample-spacing scale $n^{-1/d}$ coming from the fact that a generic point of the manifold typically lies at distance of order $n^{-1/d}$ from the nearest sampled point. In this paper, we introduce a novel 
approach, Orthogonal Ring Distance Estimation Routine (ORDER), 
which achieves a pointwise distance estimation precision of order $n^{-2/(d+5)}$ up to polylogarithmic factors in $n$ in polynomial time. This strictly beats the volumetric barrier for dimensions $d > 5$. 

As a consequence of obtaining pointwise precision better than $n^{-1/d}$, we prove that the Gromov--Wasserstein distance between the reconstructed metric measure space and the true latent manifold is of order $n^{-1/d}$. This matches the Wasserstein convergence rate of empirical measures, demonstrating that our reconstructed graph metric is asymptotically as good as having access to the full pairwise distance matrix of the sampled points. Our results are proven in a very general setting which includes general models of noisy pairwise distances, sparse random geometric graphs, and unknown connection probability functions.  

\end{abstract}

\section{Introduction}
\label{sec:introduction}

The reconstruction of the latent geometry of a Riemannian manifold from a random geometric graph is a novel problem at the intersection of probability theory, computational complexity, and inference. 
The input to the problem is generated as follows: 
points are sampled independently from a latent manifold and connected by edges with a probability determined by a link function $\rmp$ evaluated on their pairwise geodesic or embedded distances. 
Recent works have demonstrated that it is possible to reconstruct the underlying 
manifold geometry from these noisy, discrete graph observations with high probability. 
They do so by providing very accurate estimates for the geometric pairwise distance between each pair of the points corresponding to vertices of the graph. These works can be viewed as denoising the original given noisy distances to obtain much more accurate distances. 

However, existing approaches face a significant limitation regarding the precision of pairwise distance estimation. 
In previous results, the pointwise distance error is fundamentally constrained by the volumetric barrier, namely the natural scale $n^{-1/d}$ set by the typical distance from a target point on a $d$-dimensional manifold to the nearest sampled point among $n$ samples. 
This barrier is also natural from an information-theoretic perspective: even with direct access to the latent sample points, the empirical measure is typically $n^{-1/d+o(1)}$ away from the true measure in Wasserstein distance~\cite{WB19}. Since a random geometric graph reveals strictly less information than the full latent point cloud, it is natural to view $n^{-1/d}$ as a benchmark barrier for distance recovery from graph data.
Works by Huang, Jiradilok, and Mossel \cite{HJM24, HJM25} investigated reconstructing Riemannian geometry from random geometric graphs, establishing methods for distance recovery with errors $n^{-\Theta(1/d)}$ and $n^{-1/(d+1)}$ (up to logarithmic factors). 
More recent work of Fefferman, Marty, and Ren \cite{FMR25} generalizes the results, by considering a more abstract setting of noisy distances that includes random geometric graphs as a special case and by considering the case of an unknown link function and obtained distance estimations with error 
$n^{-1/2d}$ up to logarithmic factors. 
The works \cite{HJM24, HJM25,FMR25} ultimately encounter the same volumetric limitations and 
are unable to recover pairwise distances to accuracy better than $n^{-1/d}$. This scale was also identified in both \cite{HJM24} and \cite{FMR25} as a natural lower-bound benchmark for the problem, and an actual lower bound $n^{-6/d}$ was established in \cite{HJM25}, which confirms that the error must be at least $n^{-\Omega(1/d)}$ for arbitrary algorithms, but it is much weaker than the volumetric barrier.

In this paper, we introduce a novel approach, {\em Orthogonal Ring Distance Estimation Routine (ORDER)}, which successfully beats the volumetric barrier. 
By extracting and analyzing ring shaped sets, our polynomial-time algorithm achieves a pointwise distance estimation precision of order $n^{-2/(d+5)}$ (up to polylogarithmic factors in $n$). This error rate strictly surpasses the volumetric barrier of $n^{-1/d}$ for dimensions $d > 5$. 

This improved precision has interesting implications for the problem of global recovery of the underlying metric measure space. 
We prove that the Gromov--Wasserstein distance between the reconstructed metric space and the true latent measure is of order $n^{-1/d}$. 
Crucially, this matches the optimal Wasserstein convergence rate for empirical measures sampled from the manifold. 
In terms of approximating the metric measure space, our recovered random geometric graph is therefore asymptotically as good as having the full, exact distance information between all sampled points.
Our results also extend to the case of unknown link function.

\subsection{Formal statements of the main results}
\paragraph{Setting.}
We briefly describe the setting in the soft random geometric graph model, which is a special case of our more general graph model. 
\begin{enumerate}
\item We begin with a latent metric space which is a $d$-dimensional Riemannian manifold with bounded sectional curvature and positive injectivity radius (Assumption~\ref{_assump: manifold}).
\item 
We assume that the metric space is equipped with a probability measure $\mu$ satisfying the lower Ahlfors regularity assumption (Assumption~\ref{_def: mu}); i.e., the measure of balls of radius $r \leq r_{\mu}$ is of order at least $r^d$,
\item The random graph is generated by first sampling $n$ points $X_1,\ldots,X_n$ independently according to the measure $\mu$,
\item There is a monotone decreasing {\em link function} $\rmp : (0,\infty) \to [0,1]$ which is bi-Lipchitz in a neighborhood of $0$ (see Assumption~\ref{_def: distance-probability}), and
\item For any $u, v \in [n]$ with $u \neq v$, we include an (undirected) edge $\{u,v\}$ in the graph with probability
$\sp \rmp(|X_u X_v|)$, where $|X_u X_v|$ is the geodesic distance between $X_u$ and $X_v$, and 
$\sp n \ge (\log n)^{\log\log n}$ is the standing sparsity assumption. The choices for different edges are independent. 
\end{enumerate}
In fact, our graph model is slightly more general and also covers the
setting of~\cite{FMR25}; see Definition~\ref{def:graph_model}. For example, it
includes the case in which one observes a weighted graph whose edge weights are
noisy versions of the latent pairwise distances, or equivalently, the full
collection of such noisy pairwise distances.

We write \(\rG\) (see Definition~\ref{def:rG}) for the common local scale
below which the manifold admits uniformly bi-Lipschitz local coordinates,
the lower Ahlfors regularity estimate applies, and \(\rmp\) is bi-Lipschitz.

\medskip

\paragraph{Cluster generation} 
A major step in the previous works~\cite{HJM24,HJM25,FMR25} was to extract clusters from the graph.  
More formally an $(r_0,\lambda_0,p_0)$ {\em cluster generation algorithm}, takes the graph as the input and with probability at least $1-p_0$ for each $v \in [n]$ finds a $B_v \subset [n]$ such that:
$B_v \subseteq \{v' \in \bfV : \D{X_v}{X_{v'}} \le r_0\}$, and 
$|B_v| \ge \lambda \, n \, \mu_{\min}(r_0)$, where $\mu_{\min}(r_0)$ is the minimal $\mu$-probability of a ball of radius $r_0$ in the manifold, which is by assumption of order $r_0^d$.
See also Definition~\ref{def:cluster_generating_algorithm}.

We briefly summarize the key features of cluster extraction algorithms in previous work.  
We note that in fact most of these work find cluster of radius $r_0 = n^{-\Theta(1/d)}$ but all that is needed for our main results is $r_0 = O(1)$. The table below briefly summarizes previous work on 
cluster generation according to the following axis: 
\begin{itemize}
\item Density of the graph. Do the results hold only for dense graph, i.e., when $\sp = 1$? or also when the graph is sparser.
\item Efficiency: Does the clustering algorithm runs in time polynomial in $n$?
\item Metric Compatibility: Does the algorithm require the metric compatibility assumption?
\item Knowledge of link function. Does the algorithm require to know the link function $\rmp$?
\end{itemize}

\begin{table}[htbp]
\def\arraystretch{1.5}
\centering
\begin{tabular}{|c|c|c|c|c|}
\hline
\textbf{Reference}  & $\boldsymbol{\sp}$ & \textbf{Polynomial Time?} & \textbf{Metric Compatible?} & \textbf{unknown} $\boldsymbol{\rmp}$\textbf{?} \\
\hline
\cite{HJM24} &  $\Theta(1)$ & \ding{51} & \ding{55} Embedded Manifold & \ding{55} \\
\hline
\cite{HJM25} &  $\omega(\log n/\sqrt{n})$ & \ding{51}  & \ding{51}& \ding{55} \\
\hline
\cite{FMR25} &  $\Theta(1)$ & \ding{51} & \ding{51}  & \ding{51} \\
\hline

\end{tabular}
\vspace{0.25 cm}
\caption{Comparison of cluster extraction algorithms}
\label{tab:model_comparison}
\end{table}

In our main result we consider the setting above and the existence of a cluster generating algorithm and show to obtain very accurate distance estimates. Note that the requirements from the cluster generating algorithms are very weak in the sense that the radius $r_0$ is $\Theta(1)$. 

\begin{theor}[Distance Denoising]
\label{thm:bootstrap} 
Consider the Riemannian graph model (Definition \ref{def:graph_model}) where $(M,\mu,\rmp,\sp)$ satisfies the Assumptions~\ref{_assump: manifold}, \ref{_def: mu}, and \ref{_def: distance-probability}, and $\sp n \ge (\log n)^{\log\log n}$.
Suppose there exists an $(r_0,\lambda_0,p_0)$ cluster generating algorithm for this model with $r_0 \le c_{\rm cn}\rG$, where $c_{\rm cn}>0$ is some absolute constant.

Then, given as input the sampled graph \(G\) on \(\bfV\), all auxiliary inputs required by the chosen cluster generating algorithm, and the parameters $\sp$, \(\rG\), bi-Lipschitz constants of $\rmp$,
there is an algorithm which outputs with probability at least \(1-100p_0-n^{-\omega(1)}\) a distance function \(\hat{d}\) on \(\bfV\) such that for every \(v,w\in \bfV\), we have 
\begin{align*}
\left| \hat{d}(v,w) - R\D{X_v}{X_w} \right| \le
\log^C n 
\left(\frac{1}{\sp n}\right)^{\frac{2}{d+5}}, 
\end{align*}
and $R$ is an unknown dilation factor with $\max\{R,1/R\} \le C \max(L_{\rmp},1/\ell_\rmp)$, where $C$ is a universal constant. 
If $\rmp$ is known, we can take $R = 1$. Moreover, the running time of this algorithm is the running time of the chosen cluster generating algorithm plus \(O(n^3)\).

\end{theor}

\begin{rem}
    In the corresponding settings, an $(r_0,\lambda_0,p_0)$ cluster generating algorithm can be extracted from the results of \cite{HJM25, FMR25}, since those works recover latent distances with error \(n^{-\Theta(1/d)}\), which is much smaller than the required scale \(r_0\) in Theorem~\ref{thm:bootstrap}. As indicated in Table~\ref{tab:model_comparison}, the main caveat is that \cite{HJM25} assumes that \(\rmp\) is known, whereas \cite{FMR25} allows \(\rmp\) to be unknown but requires a dense graph. We expect that the arguments in both papers can be adapted to the setting of Theorem~\ref{thm:bootstrap} when \(\sp=\omega(\log n/\sqrt{n})\) and \(\rmp\) is unknown, which would extend the applicability of Theorem~\ref{thm:bootstrap} to that regime as well; however, we do not pursue this here.
\end{rem}

\begin{rem}[Comparison with the volumetric barrier]
Up to polylogarithmic factors, the error bound in Theorem~\ref{thm:bootstrap} is of order \((\sp n)^{-2/(d+5)}\). In the dense case \(\sp=1\), this becomes \(n^{-2/(d+5)}\), which improves on the volumetric barrier \(n^{-1/d}\) whenever \(d>5\). More generally, if
\[
\sp \ge n^{-1/2+\varepsilon}
\qquad\text{for some }\varepsilon>0,
\]
then the bound is at most \(n^{-(1+2\varepsilon)/(d+5)}\), which is smaller than \(n^{-1/d}\) once \(d\) is sufficiently large as a function of \(\varepsilon\).
\end{rem}

\begin{rem}[Lower bound on the error of distance estimation]
    It was also shown in \cite{HJM25} that, the same result cannot hold with $n^{-6/d+o(1)}$ for the error of distance estimation, which is strictly worse than the volumetric barrier. 
\end{rem}

\begin{rem}[Optimal limitation on the unknown $\rmp$ case]
In the case of unknown $\rmp$, the distance function $\hat{d}$ is recovered up to an unknown dilation factor $R$. This is optimal in the sense that, without knowing $\rmp$, one cannot distinguish the true distance from a dilated version of it, as we can dilate the $\rmp$ function and the Riemannian metric simultaneously without changing the distribution of the graph, when only an upper and lower Lipschitz bound on $\rmp$ is given.
\end{rem}

\subsubsection{Application to Gromov--Wasserstein distance of the underlying metric measure spaces}

For later reference, we record the standard definitions. If \(\mu\) and \(\nu\)
are probability measures on spaces \(M\) and \(N\), respectively, we write
\[
\Pi(\mu,\nu)
\]
for the set of all couplings of \(\mu\) and \(\nu\), namely all probability
measures on \(M\times N\) whose marginals are \(\mu\) and \(\nu\).

\begin{defi}[$p$-Wasserstein distance]
Let \((M,{\rm d})\) be a metric space, let \(\mu,\nu\) be probability measures
on \(M\), and let \(p\ge 1\). The \(p\)-Wasserstein distance between \(\mu\) and
\(\nu\) is
\[
W_p(\mu,\nu)
:=
\inf_{\pi\in\Pi(\mu,\nu)}
\left(
\int_{M\times M} {\rm d}(x,y)^p \, d\pi(x,y)
\right)^{1/p}.
\]
\end{defi}

Thus \(W_p\) compares two metric probability spaces only when the underlying
metric space is the same. When the underlying metric spaces are different, one
instead uses the Gromov--Wasserstein distance, which compares the two spaces
through their internal pairwise distances under a coupling.

\begin{defi}[$p$-Gromov--Wasserstein distance]
Let \((M,{\rm d}_M,\mu)\) and \((N,{\rm d}_N,\nu)\) be metric probability
spaces, and let \(p\ge 1\). Their \(p\)-Gromov--Wasserstein distance is
\[
\mathrm{GW}_p\bigl((M,{\rm d}_M,\mu),(N,{\rm d}_N,\nu)\bigr)
:=
\frac12
\inf_{\pi\in\Pi(\mu,\nu)}
\left(
\int_{(M\times N)^2}
\bigl|{\rm d}_M(x,x')-{\rm d}_N(y,y')\bigr|^p\,
d\pi(x,y)\,d\pi(x',y')
\right)^{1/p}.
\]
\end{defi}

\begin{rem}
If the underlying metric spaces are the same, then the Gromov--Wasserstein
distance is controlled by the Wasserstein distance. More precisely, for any
metric space \((M,{\rm d})\) and any probability measures \(\mu,\nu\) on \(M\),
\[
\mathrm{GW}_p\bigl((M,{\rm d},\mu),(M,{\rm d},\nu)\bigr)
\le W_p(\mu,\nu).
\]
\end{rem}

To generalize our results to the Gromov--Wasserstein setting, we rely on the following standard estimate for the convergence rate of the empirical measure in Wasserstein distance:
\begin{theor}[Empirical Wasserstein rate, cf.~\cite{WB19}]
\label{thm:empirical-wasserstein-rate}
Fix \(p\ge 1\). Assume that \((M,{\rm d},\mu)\) satisfies the manifold
assumption (Assumption~\ref{_assump: manifold}) and the lower Ahlfors
regularity assumption (Assumption~\ref{_def: mu}). Let
\[
X_1,\dots,X_n \stackrel{\mathrm{i.i.d.}}{\sim} \mu
\qquad \text{ and } \qquad
\hat\mu_n:=\frac1n\sum_{i=1}^n \delta_{X_i},
\]
where $\delta_{X_i}$ denotes the Dirac measure at $X_i$. Then
\[
W_p(\mu,\hat\mu_n)=n^{-1/d+o(1)},
\]
with probability tending to \(1\), as \(n\to\infty\).
\end{theor}

\begin{rem}
This is a standard consequence of the sharp rates for empirical measures in
Wasserstein distance proved by Weed--Bach~\cite{WB19}. The compact manifold
assumption gives the required \(d\)-dimensional covering-number upper bounds,
while Assumption~\ref{_def: mu} supplies the corresponding lower-mass bound.
\end{rem}

Now we can state the Gromov--Wasserstein consequence of Theorem~\ref{thm:bootstrap}: 
\begin{theor}[Gromov--Wasserstein consequence of Theorem~\ref{thm:bootstrap}]
\label{thm:bootstrap-gw}
Fix \(p\) with \(1\le p<d/2\), and equip \(\bfV\) with the uniform probability measure
\[
\nu_{\bfV}:=\frac1n\sum_{v\in\bfV}\delta_v.
\]
Assume moreover that there exists \(\eta>0\) such that
\[
\left(\frac{1}{\sp n}\right)^{\frac{2}{d+5}}
\le n^{-1/d-\eta}.
\]
Then,  with probability at least \(1-100p_0-o(1)\), there
exists the same dilation factor \(R\) as in Theorem~\ref{thm:bootstrap} such
that
\[
\mathrm{GW}_p\bigl((\bfV,\hat d,\nu_{\bfV}),(M,R{\rm d},\mu)\bigr)
\le n^{-1/d+o(1)}.
\]
\end{theor}

\begin{proof}
Let
\[
\hat\mu_n:=\frac1n\sum_{v\in\bfV}\delta_{X_v}.
\]
Use the coupling
\[
\pi:=\frac1n\sum_{v\in\bfV}\delta_{(v,X_v)}
\]
between \(\nu_{\bfV}\) and \(\hat\mu_n\). Then the distortion term in the
definition of \(\mathrm{GW}_p\) is bounded by the corresponding uniform
pointwise error bound from Theorem~\ref{thm:bootstrap}, and the factor \(1/2\)
comes from our normalization of the Gromov--Wasserstein distance. This gives the
bound with \(\hat\mu_n\) in place of \(\mu\). The displayed estimates then
follow from the triangle inequality, Theorem \ref{thm:empirical-wasserstein-rate}, 
and the fact that on a fixed metric space one has \(\mathrm{GW}_p\le W_p\).
The assumption on \((\sp n)^{-2/(d+5)}\) ensures that the bootstrap error is
absorbed into the \(n^{-1/d+o(1)}\) term. 
\end{proof}

\subsection{Related work}
Our work follows and builds on a few recent papers in the area of manifold learning.
Huang, Jiradilok, and Mossel~\cite{HJM24} propose recovering a manifold from a random geometric graph generated by first sampling points from the manifold space and then connecting each pair of sampled points independently with a probability that depends on their distance in a given embedding. 

Huang, Jiradilok, and Mossel~\cite{HJM24} provide an efficient algorithm for recovering all pairwise distances from such random geometric graphs with error $n^{-\Omega(1/d)}$. 
In~\cite{HJM25} the same authors study the same problem where the edge probabilities are now a function of the intrinsic distance of the points in the manifold. They were able to derive similar results with error bound on the pairwise distances $O(n^{1/(d+2)})$ up to polylog $n$ factor. They also provided a lower-bounds showing no algorithm can recover pairwise distances to accuracy better than $n^{-6/(d+2)}$ up logarithmic factors. 
 
Fefferman, Marty, and Ren~\cite{FMR25} generalize the problem by considering a much more general model of noisy distances for each pair and derived using similar ideas a pairwise distance estimator with error of $n^{-1/(2d)}$ up to polylog $n$ factor. A notable feature of the results of~\cite{FMR25} is that they do not require the knowledge of the link function in advance, and also relax the condition on the underlying link function and metric space.

More broadly, random geometric graphs provide a geometric refinement of the
classical Erd\H{o}s--R\'enyi model; see, for example,
\cite{bollobas2011random,janson2011random,frieze2015introduction,Pen03,dettmann2016random}.
Recent inferential work on such models includes distinguishability results
between geometric and Erd\H{o}s--R\'enyi graphs
\cite{BDEM16,LMSY22,BBN20} and nonparametric recovery of connection kernels for
translation-invariant models on spheres \cite{ADC19,EMP22,DDC23}. 

From the viewpoint of manifold learning, our work is related to the problem of
recovering pairwise distances between points on a manifold from generalized
noisy observations, ranging from direct distance measurements corrupted by
noise to binary observations whose expectations are given by the link function
evaluated at the underlying distance.
Combined with the reconstruction framework of \cite{FIKLN20,FILLN21}, our
distance estimates can be used to build a Riemannian manifold close to the
latent one. Related work on manifold inference from noisy ambient samples
includes \cite{JMLR:v26:25-0183,aizenbud2021non}, which studies projection and
tangent-space estimation with quantitative guarantees.

\subsection{Proof overview}
We start with defining a few notations we use throughout the current paper.

\begin{defi}[Asymptotic comparison notation]
For nonnegative sequences $(a_n)$ and $(b_n)$, we write $a_n\lesssim b_n$ if there exists a constant $C>0$ that might depend on $L_\rmp$ and $\ell_\rmp$, such that
\[
a_n\le C b_n
\]
for all sufficiently large $n$. We write $a_n\gtrsim b_n$ if $b_n\lesssim a_n$, and $a_n\asymp b_n$ if both $a_n\lesssim b_n$ and $a_n\gtrsim b_n$ hold. We also use the notation $a_n = O(b_n)$ to denote $a_n \lesssim b_n$.

We write
\[
a_n\ll b_n
\]
as an equivalent notation for $a_n = o(b_n)$, that is, if $a_n/b_n\to 0$ as $n\to\infty$.

Finally, we write $a_n\simeq b_n$ if $a_n/b_n\to 1$ as $n\to\infty$.
\end{defi}

In this work, if $k$ is a positive integer, then the function $\log^k(n)$ denotes $(\log(n))^k$. Furthermore, the notation $[k]$ denotes the set of all positive integers less than or equal to $k$.

\medskip

\step{Limitation of the cluster approach}

For simplicity, let us consider the dense regime \(\sp=1\).
Write \(\{X_v\}_{v\in\bfV}\) for the latent points corresponding to the graph vertices.
In previous approaches, one associates to each vertex \(v\) a cluster
\(U_v\subseteq\bfV\) of vertices whose latent points lie near \(X_v\).
Let us imagine for the moment that we are given ideal clusters of the form
\[
U_v:=\{v'\in\bfV:\ \D{X_v}{X_{v'}}\le \xi\},
\qquad
|U_v|\gtrsim n\xi^d,
\]
for some scale \(\xi>0\).
This is essentially the best one can hope for: by the lower regularity assumption,
the number of sampled points within distance \(\xi\) of \(X_v\) is of order
\(n\xi^d\), and this is sharp, for instance, under the uniform distribution on
a manifold.

To continue our discussion, let us define the following average. For each set $U \subseteq \bfV$ and $w \in \bfV$, 
define
\[
\cn{U}{w}
=
\frac{1}{\sp|U|}\sum_{u\in U} Z_{u,w},
\qquad
\acn{U}{w}
=
\frac{1}{|U|}\sum_{u\in U}\rmp(\D{X_u}{X_w}).
\]
Conditionally on the latent points, the difference
\(\cn{U_v}{w}-\acn{U_v}{w}\) is an average of independent centered subgaussian
random variables. Standard concentration inequalities, together with the
Lipschitz continuity of \(\rmp\), give that with high probability, for every
\(v,w\in\bfV\),
\[
\cn{U_v}{w}
=
\acn{U_v}{w}
+
O\!\left(\frac{1}{\sqrt{|U_v|}}\right)
=
\rmp(\D{X_v}{X_w})
+
O(\xi)
+
O\!\left(\frac{1}{\sqrt{n\xi^d}}\right).
\]
The first error term comes from the fact that, for \(u\in U_v\),
\(\D{X_u}{X_w}\) may differ from \(\D{X_v}{X_w}\) by at most \(O(\xi)\).

If \(\rmp\) is known, we may invert \(\rmp\) and obtain an estimator for
\(\D{X_v}{X_w}\) with error of the same order,
\[
O(\xi)+O\!\left(\frac{1}{\sqrt{n\xi^d}}\right).
\]
Optimizing in \(\xi\) gives
\[
\xi \asymp n^{-1/(d+2)},
\qquad
\xi+\frac{1}{\sqrt{n\xi^d}}
=
O\!\left(n^{-1/(d+2)}\right).
\]
This is already close to the volumetric barrier \(n^{-1/d}\), but still worse
than it. In other words, regardless of how well one can extract the clusters
\(U_v\) from the observed graph, the best accuracy obtainable from this
approach is of order \(n^{-1/(d+2)}\). This is precisely the rate achieved
in~\cite{HJM25}, up to polylogarithmic factors.

A simple refinement is to average over the bipartite graph between \(U_v\) and
\(U_w\), rather than between \(U_v\) and the singleton \(\{w\}\). This improves
the rate to \(n^{-1/(d+1)}\), but it still does not reach the volumetric
barrier.

\step{Orthogonal Ring clusters for Distance Estimation Routine (ORDER)}

The main improvement in the present work comes from adapting the cluster to the
specific pair \((v,w)\) whose distance we wish to estimate. Instead of using the
same cluster \(U_v\) for all \(w\), we construct a pair-dependent set
\(U_{v,w}\), designed specifically for approximating \(\D{X_v}{X_w}\).

To explain the idea, let us again work in a local Euclidean picture. Assume
that \(X_{u'}\) is close to \(X_v\), with
\[
\D{X_v}{X_{u'}}\asymp \xi,
\qquad
\xi\ll \D{X_v}{X_w}.
\]
Consider the triangle with vertices \(X_v,X_w,X_{u'}\).
A crude estimate gives
\[
|\D{X_v}{X_w}-\D{X_{u'}}{X_w}|
\le
\D{X_v}{X_{u'}}
\lesssim
\xi.
\]
This is exactly the source of the \(O(\xi)\) bias in the cluster approach.

Now suppose instead that the triangle is right-angled at \(X_v\). Then the
Pythagorean theorem gives
\[
\D{X_{u'}}{X_w}^2
=
\D{X_v}{X_w}^2+\D{X_v}{X_{u'}}^2,
\]
and therefore
\[
\D{X_v}{X_w}
=
\sqrt{\D{X_{u'}}{X_w}^2-\D{X_v}{X_{u'}}^2}
=
\D{X_{u'}}{X_w}
+
O\!\left(\frac{\xi^2}{\D{X_v}{X_w}}\right).
\]
Thus, if \(\D{X_v}{X_w}\) is bounded below by a constant, the bias improves from
order \(\xi\) to order \(\xi^2\).

In a general triangle, the law of cosines yields
\[
\D{X_{u'}}{X_w}^2
=
\D{X_v}{X_w}^2
+
\D{X_v}{X_{u'}}^2
-
2\D{X_v}{X_w}\D{X_v}{X_{u'}}\cos\theta,
\]
where \(\theta\) is the angle at \(X_v\). Expanding as before, one obtains
\begin{equation}
\label{eq:proof-overview-law-of-cosines}
|\D{X_v}{X_w}-\D{X_{u'}}{X_w}|
\lesssim
\frac{\xi^2}{\D{X_v}{X_w}}
+
\xi\,|\cos\theta|.
\end{equation}
If \(\D{X_v}{X_w}\asymp 1\), this becomes
\[
|\D{X_v}{X_w}-\D{X_{u'}}{X_w}|
\lesssim
\xi^2+\xi\,|\cos\theta|.
\]
Hence, in order to retain the quadratic improvement, it is enough to impose
\[
|\cos\theta|=O(\xi).
\]

This leads to the following geometric picture. Work in \(\mathbb R^d\), place
\(X_v\) at the origin, and let \(e_1\) be the unit vector pointing from
\(X_v\) to \(X_w\). Write
\[
x=(x_1,x_\perp)\in \mathbb R\times\mathbb R^{d-1},
\]
where \(x_1\) is the coordinate in the \(e_1\)-direction and \(x_\perp\) is the
orthogonal component. Consider the set
\[
\mathcal S:=
\left\{
x\in\mathbb R^d:\ 
|x_1|\le \xi^2,
\quad
\tfrac12\xi\le \|x_\perp\|\le 2\xi
\right\}.
\]
This is a thin ring slab centered at \(X_v\), transverse to the direction
toward \(X_w\). If \(x\in\mathcal S\), then
\[
\|x\|\asymp \|x_\perp\|\asymp \xi,
\qquad
|\cos\theta_x|
=
\left|\left\langle \frac{x}{\|x\|},e_1\right\rangle\right|
=
\frac{|x_1|}{\|x\|}
\lesssim
\xi,
\]
where \(\theta_x\) is the angle between \(x\) and \(e_1\).
Thus the points in \(\mathcal S\) satisfy exactly the angular condition needed
to make the error in \eqref{eq:proof-overview-law-of-cosines} quadratic.

The volume of \(\mathcal S\) is of order
\[
\xi^{d-1}\cdot \xi^2=\xi^{d+1}.
\]
By lower regularity, a sampled point lands in \(\mathcal S\) with probability of
order \(\xi^{d+1}\). If we could extract the corresponding cluster
\[
U_{v,w}:=\{u'\in\bfV:\ X_{u'}\in \mathcal S\},
\]
then the same concentration argument as before would give
\[
\cn{U_{v,w}}{w}
=
\acn{U_{v,w}}{w}
+
O\!\left(\frac{1}{\sqrt{|U_{v,w}|}}\right)
=
\rmp(\D{X_v}{X_w})
+
O(\xi^2)
+
O\!\left(\frac{1}{\sqrt{n\xi^{d+1}}}\right).
\]
Optimizing the right-hand side now yields
\[
\xi\asymp n^{-1/(d+5)},
\]
and therefore an error of order
\[
n^{-2/(d+5)}.
\]
This is exactly the gain stated in Theorem~\ref{thm:bootstrap}, up to the
small loss in the exponent needed to absorb additional polylogarithmic factors.
The main remaining question is how to extract such ring clusters from the
observed graph.

In the actual manifold setting, one must replace Euclidean distances by
geodesic distances and control the distortion caused by curvature. Two caveats
arise. First, the argument only applies when all the points lie within the
local geometric scale \(\rG\). Second, the law of cosines does not hold
exactly, and one must control the curvature error. This is done
using triangle comparison estimates; see Section~\ref{sec:geometric-comparison-tools}
and in particular Lemma~\ref{_lem: M-opposite-side} for the precise version of
\eqref{eq:proof-overview-law-of-cosines}. The Euclidean ring slab
\(\mathcal S\) is replaced by its manifold analogue defined through Fermi
coordinates along a geodesic; see Section~\ref{sec: latent-location-events}.

\step{Bootstrapping: from coarse clusters to multiplicative distance recovery}

Before discussing the main idea of the proof, we explain a first
bootstrapping step: how one can pass from coarse local clusters to a
\emph{multiplicative} recovery of latent distances. This, in turn, allows one
to construct clusters at essentially any prescribed scale above the volumetric
barrier and later further refine them to find ring clusters. 
From this point onward, we return to the general sparse notation, so the
dependence on \(\sp\) will be kept explicit.

A key ingredient, adapted from the cluster-based strategy in~\cite{HJM25}, is
the following. Suppose we are given a family of coarse clusters \(B_v\), indexed
by vertices \(v\in\bfV_0\), where \(\bfV_0\subseteq\bfV\) has proportional size
and the latent sample \(X_{\bfV_0}\) forms a sufficiently dense net of the
manifold. Assume each \(B_v\) collects vertices whose latent points lie within
distance \(r_0\) of \(X_v\), so that \(|B_v|\asymp nr_0^d\).
In the actual construction, the clusters \(B_v\) are built using only the
induced subgraph on \(\bfV_0\), while the vertices \(u,w\) to be compared lie in
disjoint blocks outside \(\bfV_0\). This separation is what makes the relevant
concentration genuinely conditional on fresh randomness.

For any two vertices \(u,w\), define
\[
\Delta_{\rm cn}(u,w):=\max_{v\in\bfV_0}\big|\cn{B_v}{u}-\cn{B_v}{w}\big|.
\]

Here \(\rG\) denotes the fixed local geometric scale below which the manifold
admits uniformly bi-Lipschitz local coordinates, the sampling measure is
regular, and \(\rmp\) behaves
bi-Lipschitzly. Lemma~\ref{lem:delta-cn-proxy} shows that, on the corresponding
high-probability event,
\[
\frac{\log^2 n}{\sqrt{\sp n}}\le \Delta_{\rm cn}(u,w)\lesssim \rG
\qquad\Longrightarrow\qquad
\Delta_{\rm cn}(u,w)\asymp \D{X_u}{X_w}.
\]
This is the analogue of Proposition~6.1 in~\cite{HJM25}; both arguments rely on
the same geometric mechanism, namely the regularity lemma
(Lemma~\ref{lem: regularity-geometry}), which quantifies the following property of $M$: 
Suppose points $p,p',q,q'$ in a metric space with $p$ and $p'$ close to each other, and $q$ and $q'$ close to each other, then the difference $\D{p}{q}-\D{p}{q'}$ is proportional to the difference $\D{p'}{q}-\D{p'}{q'}$, with matching sign.

Thus, by observing the graph alone, one can recover the \emph{order of
magnitude} of the latent distance between any two vertices. In particular, one
can construct local clusters at essentially any chosen scale, even in a
dimension-free regime such as \(n^{-1/2}\) up to polylogarithmic factors.
However, if one wants the cluster to contain many sample points, the scale must
still remain above the volumetric barrier.

From this point onward, we may regard \(\Delta_{\rm cn}\) as a bi-Lipschitz
proxy for latent distance above the threshold
\[
\frac{\log^2 n}{\sqrt{\sp n}},
\]
which in the dense case \(\sp=1\) reduces to \(\log^2 n/\sqrt n\).

\step{From multiplicative distance recovery to ring-cluster extraction}

Suppose now that \(\D{X_u}{X_v}\asymp 1\). Our goal is to construct a
pair-dependent cluster \(\ring{u,v}\) such that, for every \(u'\in \ring{u,v}\), the
triangle \(X_v,X_u,X_{u'}\) is approximately right-angled at \(X_u\).

As a first step, use the multiplicative proxy \(\Delta_{\rm cn}\) to collect all
vertices \(u'\) satisfying
\[
\Delta_{\rm cn}(u,u')\asymp \xi.
\]
This produces a set of points lying at latent distance of order \(\xi\) from
\(X_u\); in the proof this set is denoted by \(\ring{u}\).
From \(\ring{u}\), we then wish to retain only those points that are
approximately orthogonal to the direction toward \(X_v\).

To see why this helps, let \(\theta\) denote the angle at \(X_u\) in the triangle
\(X_v,X_u,X_{u'}\). The triangle comparison estimate
\[
\big|\D{X_u}{X_v}-\D{X_{u'}}{X_v}-\D{X_u}{X_{u'}}\cos\theta\big|
\le
4\frac{\D{X_u}{X_{u'}}^2}{\D{X_u}{X_v}}
\asymp
\xi^2
\]
(see Lemma~\ref{_lem: M-opposite-side}) is the manifold analogue of the law of
cosines after absorbing curvature errors. In particular, it shows that 
\[
\big|\D{X_u}{X_v}-\D{X_{u'}}{X_v}\big|\gg \xi^2 
\quad \Leftrightarrow
\quad
|\cos\theta|\gg \xi.
\]
Thus, if we could estimate the difference
\(\big|\D{X_u}{X_v}-\D{X_{u'}}{X_v}\big|\) to accuracy \(\xi^2\), then we could
filter out precisely those points \(u'\) that are too far from being orthogonal.

At first glance, this seems to require estimating
\[
\big|\D{X_u}{X_v}-\D{X_{u'}}{X_v}\big|
\]
to additive precision \(\xi^2\), which is much finer than the volumetric
scale. Since \(\D{X_u}{X_v}\asymp 1\), the multiplicative proxy
\(\Delta_{\rm cn}\) is far too coarse for this task. The key idea is therefore
not to compare \(u\) and \(u'\) directly at the single point \(v\), but instead
to average over a small cluster \(\mathcal B_v\) around \(v\), and to study the
observable difference
\[
\cn{\mathcal B_v}{u}-\cn{\mathcal B_v}{u'}.
\]
Why does this help? 
Suppose \(\mathcal B_v\) collects points \(v'\) at distance
of order \(\xi\) from \(X_v\). If the triangle \(X_v,X_u,X_{u'}\) is nearly
right-angled at \(X_u\), then the triangle \(X_{v'},X_u,X_{u'}\) remains nearly
right-angled at \(X_u\) for every \(v'\in\mathcal B_v\), because moving \(X_v\)
to a nearby point \(X_{v'}\) changes the angle at \(X_u\) by only
\(O(\xi/\D{X_u}{X_v})\asymp O(\xi)\). This is again a consequence of the
regularity lemma (Lemma~\ref{lem: regularity-geometry}).

Therefore, if \(|\cos\theta|\lesssim \xi\), then for every \(v'\in\mathcal B_v\),
the corresponding angle \(\theta'\) at \(X_u\) in the triangle
\(X_{v'},X_u,X_{u'}\) still satisfies \(|\cos\theta'|\lesssim \xi\). Applying
the law of cosines in the form~\eqref{eq:proof-overview-law-of-cosines}, we get
\[
|\D{X_u}{X_{v'}}-\D{X_{u'}}{X_{v'}}|
\lesssim
\xi^2+\xi|\cos\theta'|
\lesssim
\xi^2.
\]
Averaging over \(v'\in\mathcal B_v\) then yields
\[
|\cn{\mathcal B_v}{u}-\cn{\mathcal B_v}{u'}|
\lesssim
\xi^2.
\]
Conversely, the same reasoning allows one to deduce that  
\[
\big|\D{X_u}{X_v}-\D{X_{u'}}{X_v}\big|\gg \xi^2 
\quad \Rightarrow \quad
|\cn{\mathcal B_v}{u}-\cn{\mathcal B_v}{u'}|\gg \xi^2.
\]
Thus the quantity
\[
|\cn{\mathcal B_v}{u}-\cn{\mathcal B_v}{u'}|
\]
is indeed an observable proxy for whether \(u'\) is nearly orthogonal to the
direction of \(v\).

In conclusion, by first extracting \(\ring{u}\) and then filtering it using the
observable criterion above, one obtains the desired ring cluster \(\ring{u,v}\).
For this cluster one has
\[
|\cn{\ring{u,v}}{v}-\rmp(\D{X_u}{X_v})|
\lesssim
\xi^2,
\]
which is the key gain behind the improved error rate.

\step{Ring approach when the latent distance is small}

The discussion above is most effective when the reference distance is bounded
below by a constant, since the quadratic gain
\[
|\D{X_u}{X_v}-\D{X_{u'}}{X_v}|\lesssim \xi^2
\]
comes with a factor \(1/\D{X_u}{X_v}\) hidden in the law-of-cosines estimate.
When \(\D{X_v}{X_w}\) itself is small, we therefore introduce an auxiliary
vertex \(u\) lying approximately on the geodesic through \(X_v\) and \(X_w\),
chosen so that
\[
\D{X_u}{X_v}
\]
is of constant order and, in particular, much larger than \(\D{X_v}{X_w}\).

The purpose of this choice is that the ring cluster \(\ring{u,v}\), which is
constructed to be approximately orthogonal to the direction from \(X_u\) to
\(X_v\), is then automatically adapted to the pair \((u,w)\) as well, provided
that \(X_u,X_v,X_w\) are nearly colinear. Indeed, if the triangle
\((X_u,X_v,X_{u'})\) is approximately right-angled at \(X_u\), then the same is
true for the triangle \((X_u,X_w,X_{u'})\), because \(X_v\) and \(X_w\) lie
almost on the same geodesic ray viewed from \(X_u\). Consequently, one can show
not only that
\[
|\cn{\ring{u,v}}{v}-\rmp(\D{X_u}{X_v})|
\lesssim
\xi^2,
\]
but also that
\begin{equation}
\label{eq:proof-overview-ring-uvw}
|\cn{\ring{u,v}}{w}-\rmp(\D{X_u}{X_w})|
\lesssim
\xi^2.
\end{equation}
Thus the same ring cluster provides simultaneous approximations to both
\(\D{X_u}{X_v}\) and \(\D{X_u}{X_w}\) at quadratic accuracy, and subtracting
these two estimates allows us to recover the short distance \(\D{X_v}{X_w}\)
with error of order \(\xi^2\).

The remaining issue is how to choose such a vertex \(u\) from the observed
graph. This is done through an optimization procedure: we fix an anchor value
\(\rmp_1\), corresponding to some constant reference distance, and search over
vertices \(u'\) satisfying
\[
\cn{R_{u',v}}{v}\simeq \rmp_1.
\]
Among these candidates, we then select
\[
u
:=
\operatorname*{arg\,max}_{u'}
\cn{R_{u',v}}{w}.
\]
Intuitively, the constraint
\(\cn{R_{u',v}}{v}\simeq \rmp_1\) forces \(X_{u'}\) to lie at roughly the
correct distance from \(X_v\), while the maximization of
\(\cn{R_{u',v}}{w}\) pushes \(X_{u'}\) into the direction of \(X_w\). The
resulting maximizer is therefore close to the desired point \(u\) on the
geodesic ray through \(X_v\) and \(X_w\).

\step{Recovery in the known-\(\rmp\) case}

Suppose first that \(\rmp\) is known. Then the previous step already gives a
direct estimator for any latent distance \(\D{X_v}{X_w}\) in the local regime
\[
\D{X_v}{X_w}\lesssim \rG.
\]
Indeed, after choosing the auxiliary point \(u\) and the corresponding ring
cluster \(\ring{u,v}\), we obtain
\[
\cn{\ring{u,v}}{v}
=
\rmp(\D{X_u}{X_v})+O(\xi^2),
\qquad
\cn{\ring{u,v}}{w}
=
\rmp(\D{X_u}{X_w})+O(\xi^2).
\]
Since \(\rmp\) is bi-Lipschitz on the relevant interval, we may invert \(\rmp\)
and conclude that
\[
\rmp^{-1}\!\big(\cn{\ring{u,v}}{w}\big)
-
\rmp^{-1}\!\big(\cn{\ring{u,v}}{v}\big)
=
\D{X_v}{X_w}+O(\xi^2),
\]
so in this regime the latent distance is recovered with additive error of order
\(\xi^2\).

To pass from local to global distances, we use only that \(M\) is a geodesic
space. Namely, for any two points \(p,q\in M\), one may join them by a
minimizing geodesic and partition this geodesic into segments of length at most
\(\rG\). Each segment can then be estimated by the local procedure above, with
error \(O(\xi^2)\). This produces a weighted graph on the sampled vertices,
where each edge of latent length at most \(\rG\) is assigned its local distance
estimate. Running Dijkstra's algorithm on this graph yields an estimate of the
global latent distance. Since a geodesic of total length at most \({\rm diam}(M)\)
contains at most \(O({\rm diam}(M)/\rG)\) such local segments, the total error
accumulates by at most this factor. Therefore, in the known-\(\rmp\) case, one
recovers every latent distance with additive error
\[
O\!\left(\frac{{\rm diam}(M)}{\rG}\,\xi^2\right).
\]

\step{Unknown \(\rmp\): the argument of \cite{FMR25}}
When $\rmp$ is unknown, we follow the strategy of \cite{FMR25} in order to reconstruct the underlying $\rmp$. Here we give a high-level overview of their argument in our context. 

Suppose for the moment that, whenever \(\D{X_u}{X_v}\lesssim \rG\), we can
estimate \(\rmp(\D{X_u}{X_v})\) up to additive error \(O(\xi^2)\).
Then, for two pairs \((u,v)\) and \((u',v')\), if
\[
\D{X_u}{X_v} \ge \D{X_{u'}}{X_{v'}} + O(\xi^2),
\]
the monotonicity of \(\rmp\) implies
\[
\rmp(\D{X_u}{X_v}) \le \rmp(\D{X_{u'}}{X_{v'}}),
\]
so we can compare latent distances up to an additive uncertainty of order
\(\xi^2\). In other words, even without knowing \(\rmp\), we still have an
\emph{ordering oracle} for distances.

Now choose a reference pair \((v_0,v_1)\) at some mesoscopic scale, and write
\[
{\rm d}_1:=\D{X_{v_0}}{X_{v_1}}.
\]
Although \({\rm d}_1\) is unknown, the ring construction allows us to estimate
\(\rmp({\rm d}_1)\) up to error \(O(\xi^2)\).
The next step is to locate an approximate midpoint between \(X_{v_0}\) and
\(X_{v_1}\). To do this, one searches among all vertices \(v'\) satisfying
\[
\rmp(\D{X_{v_0}}{X_{v'}})\simeq \rmp(\D{X_{v_1}}{X_{v'}}),
\]
and chooses one maximizing \(\rmp(\D{X_{v_0}}{X_{v'}})\). The first condition
forces \(X_{v'}\) to be approximately equidistant from \(X_{v_0}\) and
\(X_{v_1}\), while the second pushes \(X_{v'}\) as close as possible to the
geodesic segment between them. Thus \(X_{v'}\) should lie near the midpoint of
that segment, and therefore
\[
\D{X_{v_0}}{X_{v'}}\simeq \D{X_{v_1}}{X_{v'}}\simeq \frac12 {\rm d}_1.
\]

The quality of this midpoint approximation depends on a localization error,
which we denote by \(\eta\): namely, \(\eta\) measures how close one can force a
sample point to lie to the true midpoint. If such a vertex \(v'\) can be found
with
\[
\D{X_{v_0}}{X_{v'}} = \frac12{\rm d}_1 + O(\eta),
\qquad
\D{X_{v_1}}{X_{v'}} = \frac12{\rm d}_1 + O(\eta),
\]
then the ring argument gives access to \(\rmp({\rm d}_1/2)\) up to error
\(O(\xi^2+\eta)\).

Iterating this midpoint construction, one obtains vertices
\[
v(s),\qquad s\in\left\{\frac{k}{2^N}:0\le k\le 2^N\right\},
\]
such that
\[
\D{X_{v_0}}{X_{v(s)}} = s\,{\rm d}_1 + O\bigl(N(\xi^2+\eta)\bigr).
\]
Thus, after \(N\) dyadic subdivision steps, one recovers a discrete net of
sample points along the geodesic from \(X_{v_0}\) to \(X_{v_1}\). 
The largest useful choice of \(N\) is therefore of order
\(\log((\xi^2+\eta)^{-1})\), since beyond that scale the accumulated error
and the lack of sufficiently well-placed sample points overwhelm the gain from
further subdivision.

Interpolating
the corresponding values of \(\rmp(\D{X_{v_0}}{X_{v(s)}})\) produces an
observable approximation of \(\rmp\) on the interval \([0,{\rm d}_1]\), up to
additive error of order
\[
\log\!\big((\xi^2+\eta)^{-1}\big)\,(\xi^2+\eta).
\]
Inverting this reconstructed ruler then yields estimates for all distances
\(\D{X_u}{X_v}\lesssim {\rm d}_1\), again up to the same error scale.
In particular, the true value of \({\rm d}_1\) is not observable; one only
recovers distances up to an overall scaling factor depending on the
bi-Lipschitz constants of \(\rmp\).

\step{Reappearance of the volumetric barrier}
The key difficulty is that the volumetric barrier reappears in the midpoint
localization error \(\eta\). A trivial choice is
\[
\eta \asymp n^{-1/d},
\]
since lower regularity guarantees only that every point of the manifold lies
within distance \(O(n^{-1/d})\) of some sample point. This was fine if one only wanted to reconstruct \(\rmp\) above the volumetric barrier, but it becomes a bottleneck when one tries to push the reconstruction down to finer scales.
Therefore, even if the ring construction itself gives an accuracy \(\xi^2\ll n^{-1/d}\), the midpoint step still contributes an error of order \(n^{-1/d}\), which dominates
the final reconstruction.

One can improve \(\eta\) when the segment being bisected is still relatively
long: for example, if the distance to be resolved is \(L\), then ring-slab
arguments can localize a midpoint to precision of order
\[
\eta \asymp \frac{\xi^2}{L}.
\]
However, if we want to push the precision below the volumetric barrier, then inevitably we have to deal with \(L\lesssim n^{-1/d}\), and the improvement breaks down. Indeed, the nearest sampled vertex to that
midpoint is typically still at distance of order \(n^{-1/d}\gtrsim L\). Thus there is a
genuine obstruction to recovering \(\rmp\) itself with precision better than the
volumetric barrier.

This is precisely the point where our present approach departs from the
midpoint-based reconstruction of \cite{FMR25}: rather than repeatedly
reconstructing \(\rmp\) through midpoint localization, we exploit the additional
structure of ring clusters to build more stable observables for points
projecting to a reference geodesic, thereby avoiding this barrier in the target
local regime.

\step{Unknown \(\rmp\): adapting the midpoint approach}

An intuitive way to understand our approach in the unknown-\(\rmp\) case is the
following. Fix two vertices \(w^0,w^1\) such that
\[
\D{X_{w^0}}{X_{w^1}}\asymp 1,
\]
and let \(\gamma\) be the minimizing geodesic joining \(X_{w^0}\) and
\(X_{w^1}\). We first seek a collection of vertices whose latent points lie in
a thin tubular neighborhood of \(\gamma\). Concretely, for suitable choices of
a level \(p\), one may look among vertices \(v\) satisfying
\[
\rmp(\D{X_v}{X_{w^0}})\simeq p
\]
and then maximize \(\rmp(\D{X_v}{X_{w^1}})\). Since \(\rmp\) is decreasing, this
favors points that are simultaneously close to the appropriate level set around
\(X_{w^0}\) and as close as possible to the geodesic direction toward \(X_{w^1}\).
In this way one obtains a family \(A\) of vertices whose latent points lie in a
thin tube around \(\gamma\).

Let \(\widetilde X_v\) denote the nearest-point projection of \(X_v\) onto
\(\gamma\). Because every thin ring slab around \(\gamma\) contains a sampled
point, the projected points \(\{\widetilde X_v\}_{v\in A}\) form a net along
\(\gamma\), and the spacing of this net is of order \(\xi^2\), rather than
order \(\xi\). This is precisely the gain coming from the ring-slab
construction: the transverse width is \(\asymp \xi\), while the longitudinal
thickness is only \(\asymp \xi^2\).

Now suppose that we had an \emph{ordering oracle} for the projected points
\(\{\widetilde X_v\}_{v\in A}\), accurate up to error \(\xi^2\). Then we
could run the midpoint-based reconstruction scheme of~\cite{FMR25} directly on
the one-dimensional set \(\{\widetilde X_v\}_{v\in A}\subset \gamma\). Since the
underlying space is now the geodesic \(\gamma\), which is one-dimensional, the
relevant volumetric barrier is reduced to the longitudinal scale \(\xi^2\).
Consequently, one can push the midpoint recursion down to scale \(\xi^2\),
which is exactly the precision of the ordering oracle.

The remaining task is therefore to build such an oracle from the graph.
To do this, we first restrict attention to points \(v\in A\) satisfying
\[
\D{X_v}{X_{w^0}},\ \D{X_v}{X_{w^1}}
\ge
c\,\D{X_{w^0}}{X_{w^1}}
\]
for some small constant \(c>0\), so that the points of interest stay away from
the endpoints \(X_{w^0}\) and \(X_{w^1}\). For each such \(v\), and for
\(i\in\{0,1\}\), we then choose a vertex \(u_v^i\) by maximizing the relevant
ring statistic relative to the pair \((v,w^i)\). By the same geometric argument
as in the ring-cluster construction, the point \(X_{u_v^i}\) lies close to the
geodesic from \(X_v\) to \(X_{w^i}\), on the side of \(X_{w^i}\), and the
associated ring cluster \(R_{u_v^i,v}\) is approximately orthogonal to that
geodesic, which in turn is close to \(\gamma\).

At this stage, one might hope that for two points \(v,v'\in A\), the quantity
\(\cn{R_{u_v^i,v}}{v'}\) directly approximates
\(\rmp(\D{X_{u_v^i}}{X_{v'}})\) up to error \(\xi^2\). This is not true in
general, because the ring \(R_{u_v^i,v}\) is constructed to be approximately
orthogonal to the geodesic from \(X_v\) to \(X_{w^i}\), not to the geodesic
from \(X_v\) to \(X_{v'}\), and the two geodesics may differ by a large angle if $X_v$ and $X_{v'}$ are close. 
However, this is exactly what we need: the value  
\(\cn{R_{u_v^i,v}}{v'}\) \emph{suppresses} the component of \(X_{v'}\) orthogonal to the
reference geodesic \(\gamma\). As a consequence,
\[
\cn{R_{u_v^i,v}}{v'}
\simeq
\rmp(\D{X_{u_v^i}}{\widetilde X_{v'}})
\simeq
\rmp(\D{\wX_{u_v^i}}{\widetilde X_{v'}})
\]
with error of order \(\xi^2\). Indeed, for the same reason, 
\(
\cn{R_{u_v^i,v}}{v}
\simeq
\rmp(\D{\wX_{u_v^i}}{\widetilde X_{v}})
\)
with error of order $\xi^2$. Therefore, by comparing
\(\cn{R_{u_v^i,v}}{v'}\) with \(\cn{R_{u_v^i,v}}{v}\), one obtains an ordering oracle
for the projected positions \(\widetilde X_v\) and \(\widetilde X_{v'}\), again
at precision \(\xi^2\).

This is the key new ingredient in the unknown-\(\rmp\) case. Thus the
unknown-\(\rmp\) problem is reduced to constructing an observable
one-dimensional ruler along a reference geodesic.

\subsection*{Organization of the paper.}
In Section~\ref{sec:model-assumptions-events} we introduce the graph model, assumptions, and basic
concentration events. Section~\ref{sec:geometric-comparison-tools}
collects the geometric comparison tools used throughout the argument.
Sections~\ref{sec:distanceOrder}--\ref{sec:distance-calibration} contain the
core ORDER pipeline: we first obtain coarse order information, then extract
ring-based sets, establish the required good latent-location events, and
finally calibrate these observables into local distance estimates.
In Section~\ref{sec:all-pairs-extension} we extend these local estimates to all
pairs, in particular we prove the theorem when \(\rmp\) is known. 
In Sections~\ref{sec:ring-maximizer-geometry}
and~\ref{sec:recovering-distance-unknown-rmp} we analyze the maximizing
candidate and resolve the case of unknown link function. Section~\ref{sec:probability-estimates}
collects the probability bounds, and the final two sections contain deferred
auxiliary and geometric proofs.

\section*{Acknowledgments}
E.M. is partially supported by Vannevar Bush Faculty Fellowship ONR-N00014-20-1-2826, by Simons Investigator award (622132), by MURI N000142412742 award and 
NSF DMS-2031883 award.

\section{Model, Assumptions, and basic events}
\label{sec:model-assumptions-events}
\subsection{Metric Probability Space and Random Graph Model}
\label{subsec:metric-probability-space}

We first collect the assumptions and sampling facts that concern only the
ambient metric probability space and do not involve the observed graph.

\begin{defi}[Metric measure space / metric probability space]
A \emph{metric measure space} is a triple \((M,{\rm d},\mu)\), where
\((M,{\rm d})\) is a metric space and \(\mu\) is a Borel measure on \(M\).

If in addition
\[
\mu(M)=1,
\]
then \((M,{\rm d},\mu)\) is called a \emph{metric probability space}.
\end{defi}

Informally, we will construct a random graph \(G\) on the vertex set \(\bfV\), 
we first sample i.i.d.\ points \(\{X_v\}_{v\in\bfV}\) from \(\mu\). Then we create edges between vertices in \(\bfV\) with weight
$$
    Z_{u,v} := B_{u,v}\widetilde Z_{u,v},\,
$$
where  \(\widetilde Z_{u,v}\) has mean \(\rmp({\rm d}(X_u,X_v))\) with subgaussian tails and \(B_{u,v}\) is a Bernoulli variable served as a sparsity mask. These random variables are independent once conditioned on the sampled points \(\{X_v\}_{v\in\bfV}\). 
The precise definition is as follows.

\begin{defi}[Random Graph model]
\label{def:graph_model}
Let $(M,{\rm d},\mu)$ be a metric probability space, and consider a non-increasing function  
\[
\rmp:[0,\infty)\to\mathbb R.
\]
Let \(\bfV\) be a finite vertex set of $n$ vertices, and let
\(\sp\in(0,1]\) be a sparsity parameter satisfying 
$$
 \sp n \ge (\log n)^{\log\log n}. 
$$

\step{Sample points}
Let \(\{X_v\}_{v\in\bfV}\) be i.i.d.\ samples from \(\mu\). 

\step{Weighted edges}
For each unordered pair \(\{u,v\}\subseteq\bfV\) with \(u\neq v\), let \(\mathcal U_{u,v}\sim
\mathrm{Unif}[0,1]\), and assume that the family
\(\{\mathcal U_{u,v}:u<v\}\) is independent, with
\(\mathcal U_{v,u}:=\mathcal U_{u,v}\).

In addition,  we assume there is a measurable function
\[
\mathbf F:[0,\infty)\times[0,1]\to\mathbb R
\]
such that
 for every \(t\ge 0\),
\[
\int_0^1 \mathbf F(t,s)\,ds = \rmp(t),
\]
and that
\[
\mathbf F(t,\mathcal U)-\rmp(t)
\]
is \(\Ksg\)-subgaussian uniformly in \(t\), where \(\mathcal U\sim\mathrm{Unif}[0,1]\).

In other words, the random variable
\[
\widetilde Z_{u,v}
:=
\mathbf F\!\bigl({\rm d}(X_u,X_v),\mathcal U_{u,v}\bigr),
\qquad u\neq v
\]
condition on $X_u$ and $X_v$, is a subgaussian random variable with mean $\rmp({\rm d}(X_u,X_v))$ and subgaussian parameter $\Ksg$.

For each unordered pair \(\{u,v\}\subseteq\bfV\) with \(u\neq v\), let
\(B_{u,v}\sim \mathrm{Bernoulli}(\sp)\), independently over \(u<v\), and
independent of \(\{X_v\}_{v\in\bfV}\) and \(\{\mathcal U_{u,v}\}_{u<v}\). Set
\[
B_{v,u}:=B_{u,v},
\qquad
Z_{u,v}:=B_{u,v}\widetilde Z_{u,v}.
\]
Further, we set \(Z_{u,u}=0\).
\end{defi}

For later use, if \(U,V\subseteq\bfV\), we write
\[
X_U:=(X_u)_{u\in U},
\qquad
\mathcal U_{U,V}:=(\mathcal U_{u,v})_{u\in U,\,v\in V,\,u\neq v},
\]
\[
B_{U,V}:=(B_{u,v})_{u\in U,\,v\in V,\,u\neq v},
\qquad
Z_{U,V}:=(Z_{u,v})_{u\in U,\,v\in V,\,u\neq v}.
\]
In particular, when \(U=V\), the collection \(Z_{U,U}\) is exactly the induced
subgraph on \(U\).
\begin{rem}[Soft random geometric graph]
The classical soft random geometric graph model corresponds to the special case where
\(\rmp:[0,\infty)\to[0,1]\) is monotone decreasing, and define
\[
\mathbf F(t,s):=\mathbf 1\{s\le \rmp(t)\},
\qquad (t,s)\in[0,\infty)\times[0,1].
\]
Then \(Z_{u,v}\) is Bernoulli with parameter \(\sp\,\rmp({\rm d}(X_u,X_v))\). In this case the graph is unweighted, and the edge probability depends on the distance between the latent points, where the larger the distance, the smaller the edge probability.
\end{rem}

\begin{rem}[Noisy distances]
Alternatively, $-Z_{u,v}$ can be viewed as a noisy distance estimator between $X_u$ and $X_v$, especially when $\rmp(\D{X_u}{X_v})= - \D{X_u}{X_v}$.  
\end{rem}

Next, we layout the assumptions of the metric space, the sampling measure, and the link function $\rmp$ that we will work with. 

\begin{assump}[Manifold Assumption]
\label{_assump: manifold}
We assume that $(M,g)$ is a compact, connected, smooth Riemannian manifold of dimension $d \ge 2$.
Let $K(\sigma)$ denote the sectional curvature at a $2$-plane $\sigma \subset TM$, and define
\[
    \kappa := \sup_{\sigma} |K(\sigma)|\,,
\]
which is finite by compactness of $M$.
Let $\rinj(M)$ denote the injectivity radius of $M$, which is also positive due to compactness.
Further, we define
\begin{align}
    \label{def:r_M}
    r_M := \min\{ \rinj(M), \pi/\sqrt{\kappa} \}\,.
\end{align}
\end{assump}

\begin{defi}[Metric space notation]
Let \((M,{\rm d})\) be a metric space. For any two points \(x,y\in M\), we
write
\[
    {\rm d}(x,y)
\]
for the underlying distance between \(x\) and \(y\). We also write
\[
    {\rm diam}(M)
\]
for the diameter of \(M\). For any point \(p\in M\)
and radius \(r>0\), we write
\[
    B_M(p,r):=\{x\in M:{\rm d}(x,p)<r\}
\]
for the open ball centered at \(p\) with radius \(r\). When the ambient space
is clear from context, we drop the subscript and simply write \(B(p,r)\).
If \(M\) is a Riemannian manifold, then we also use
\[
    \D{x}{y}
\]
to denote the geodesic distance between \(x\) and \(y\), so in that case
\(\D{x}{y}={\rm d}(x,y)\).
\end{defi}

\begin{assump}
\label{_def: mu}
Let $\mu$ be a probability measure on a $d$-dimensional Riemannian manifold $M$, and define $\mu_{\min}(r) := \min_{p \in M} \mu(B(p,r))$.
We assume that $\mu$ satisfies a \emph{lower Ahlfors regularity} condition: there exist constants ${\rm r}_\mu > 0$ and $c_\mu > 0$ such that
\begin{equation}\label{eq: mu-ball}
\mu_{\min}(r) \ge c_\mu \, r^d, \qquad \forall r \in [0, {\rm r}_\mu].
\end{equation}
\end{assump}

\begin{assump}[Link function]
\label{_def: distance-probability} 
Let $\rmp:[0,\infty) \to \mathbb{R}$ be a non-increasing function.  We assume there exist positive constants $L_{\rmp}, \ell_\rmp, \rmrp > 0$ such that for all $a,b \in [0,\rmrp]$,  
$$ \ell_\rmp |a-b| \le |\rmp(a) - \rmp(b)| \le L_\rmp |a-b|.$$
That is, $\rmp$ is bi-Lipschitz on $[0,\rmrp]$. 
Given a metric space $(M,{\rm d})$, we also assume that 
$$
    M_{\rmp} =\max_{t \in [0, {\rm diam}(M)]}|\rmp(t)| < +\infty. 
$$

\end{assump}

\begin{defi}[Accessible radius]
\label{def:rG}
Define
\[
\rG:=\tfrac{1}{16}\min\!\left\{\rM,\,\rmrp,\,r_\mu\right\}.
\]
\end{defi}

\begin{defi}[$(r,\lambda,p)$ cluster generating algorithm]
\label{def:cluster_generating_algorithm}
Fix $r>0$, $\lambda, p\in[0,1]$.
An algorithm $\mathcal{A}$ is called an \emph{$(r,\lambda, p)$ cluster generating algorithm} (at sample size $n$) if, under the graph model of Definition~\ref{def:graph_model} with $|\bfV|=n$, where $M$ is a $d$-dimensional Riemannian manifold, $\rmp$ satisfies Assumption~\ref{_def: distance-probability}, and $\mu$ satisfies Assumption~\ref{_def: mu}, given as input the sampled graph, the parameters $\ell_\rmp$, $L_\rmp$, $\rG$, $c_\mu$, and possibly the function $\rmp$ itself, it outputs with probability at least $1-p$ the following:
for each $v$ in $\bfV$, a set $B_v \subseteq \{v' \in \bfV : \D{X_v}{X_{v'}} \le r\}$ satisfying
\[
    |B_v| \ge \lambda \, n \, \mu_{\min}(r).
\]
The probability is over the randomness of the sampled graph (and any internal randomness of $\mathcal{A}$).
\end{defi}

Without loss of generality, we assume any cluster generating algorithm always outputs $B_v$ for each $v$ by adopting the convention that, upon failure, it returns $B_v=\{v\}$ for all $v$.

\subsection{Setup}
\label{subsec:setup}

\begin{assump}
    \label{assump: Riemannian model}
Throughout the paper, we work under the following setup. Let \((M,{\rm d},\mu)\) be a metric probability space satisfying the Manifold assumption \ref{_assump: manifold} and the lower Ahlfors regularity assumption \ref{_def: mu}. Let $\rmp$ be a link function satisfying \ref{_def: distance-probability}. Then consider the graph model defined in \ref{def:graph_model}. 

We assume the number of vertices $|\bfV|$ is sufficiently large, and for convenience, let $|\bfV|=N=5n$, and partition $\bfV$ into five disjoint subsets $\bfV_0,\bfV_1,\bfV_2,\bfU_1,\bfU_2$, each of size $n$.
\end{assump}

We also assume the existence of a cluster generating algorithm with certain parameters, which will be specified in the respective sections.

\begin{defi}[${\frak q}$-scale parameter]
\label{def:xi-q}
For ${\frak q}=5$, define
\[
\xiq := \left( \frac{1}{\sp n} \right)^{\frac{1}{d+{\frak q}}} (\log n)^{C_{\tref{def:xi-q}}},
\]
where $C_{\tref{def:xi-q}}$ is a sufficiently large universal constant.  
\end{defi}
\begin{align}
    \delta_n := \frac{1}{\log n}\,.
\end{align}
The specific choice of $\delta_n$ is not essential; we only require a sequence that tends to zero sufficiently slowly to absorb constant terms in the proof, while not significantly affecting the final result. In particular, under the standing sparsity assumption \(\sp n \ge (\log n)^{\log\log n}\), we have \(\xiq=o(\delta_n^A)\) for every fixed \(A>0\).
Conceptually, $\delta_n$ should be replaceable by a sufficiently small constant.

\begin{defi}[Prefixed shortest geodesics]
\label{def:prefixed-geodesics}
For each ordered pair \(p,q\in M\), fix a unit-speed geodesic \(\gamma_{p,q}\)
as follows. If \(p\neq q\), let
\[
\gamma_{p,q}:[-2\rM,\D{p}{q}+2\rM]\to M
\]
be a distance-minimizing geodesic from \(p\) to \(q\), parameterized by arc
length so that
\[
\gamma_{p,q}(0)=p,
\qquad
\gamma_{p,q}(\D{p}{q})=q,
\]
and extended by length \(2\rM\) beyond both endpoints. The existence of such a geodesic is guaranteed by the Hopf-Rinow theorem. 
If \(p=q\), choose
arbitrarily a unit vector \(e_{p,p}\in T_pM\), and define
\[
\gamma_{p,p}(t):=\exp_p(t e_{p,p}),
\qquad t\in[-2\rM,2\rM].
\]
When \(p\neq q\), we choose these geodesics compatibly so that
\(\gamma_{q,p}\) is the same unoriented curve as \(\gamma_{p,q}\), with the
opposite orientation.

For for each ordered pair \((w,v)\in\bfV\times\bfV\), we define 
\[
\gamma_{w,v}:=\gamma_{X_w,X_v}\,.
\]
\end{defi}

\subsection{Basic concentration events}
Here we introduce two basic high-probability events that will be used repeatedly in the proof. The first one is a uniform lower occupancy event, which guarantees that every ball of radius above the threshold scale contains a sufficiently large number of sample points. The second one is an edge fluctuation event, which controls the deviation of the normalized average \(\cn{U}{v}\) from its expectation \(\acn{U}{v}\). The proofs of the probability bounds are given in the Appendix, as they are based on standard concentration inequalities with union bounds and/or discretization arguments.

\begin{lemma}[Uniform lower occupancy above the threshold scale]
    \label{lem:Ept}
Let $(M,{\rm d},\mu)$ be a metric probability space, and assume that there exist
$r_\mu>0$ and a nonnegative nondecreasing function
\[
\phi:(0,r_\mu]\to[0,1]
\]
such that
\[
\mu(B(x,r)) \ge \phi(r)
\qquad \text{for all } x\in M,\; 0<r\le r_\mu.
\]
Let $\{X_v\}_{v\in \bfV}$ be i.i.d.\ samples from $\mu$ with
$|\bfV|=n$. Define
\[
\rho_n
:=
\max\left\{
\inf\left\{\delta\in(0,r_\mu]: \phi(\delta/2)\ge \frac{\log^2 n}{n}\right\},
\,
r_\mu e^{-n}
\right\},
\]
which is well-defined when $n$ is sufficiently large.
Define the event
\[
\Ept{\bfV}
:=
\left\{
|X_{\bfV}\cap B(x,r)|
\ge \frac{n\phi(r/3)}{2}
\quad
\text{for all } x\in M \text{ and all } r\in[4\rho_n,r_\mu]
\right\}.
\]
Then
\[
\mathbb P\bigl(\Ept{\bfV}\bigr)\ge 1-n^{-\omega(1)}.
\]
\end{lemma}
\begin{rem}
The quantity \(\rho_n\) depends on the particular lower-mass function
\(\phi\) used in Lemma~\ref{lem:Ept}. The extra cutoff \(r_\mu e^{-n}\) is
included only to ensure that the proof of Lemma~\ref{lem:Ept} involves at most
\(O(n)\) dyadic scales. In applications under Assumption~\ref{_def: mu}, a
convenient choice is
\[
\phi(r):=c_\mu r^d,
\qquad 0<r\le r_\mu.
\]
For this choice, and for all sufficiently large \(n\),
\[
\rho_n
=
2\left(\frac{\log^2 n}{c_\mu n}\right)^{1/d},
\]
because this polynomial scale dominates \(r_\mu e^{-n}\).
\end{rem}

\begin{defi}[Normalized averages]\label{_def: normalized_number_of_common_neighbors}
For $v \in \bfV$ and $U \subseteq \bfV$, 
we define the normalized averages of $v$ in $U$ by
\begin{equation}\label{_eq:def_normalized_number_of_common_neighbors}
\cn{U}{v}
:= \frac{\sum_{u\in U} Z_{u,v}}
{\sp \,|U|}\,.
\end{equation}
Further, we also define the (conditional) expectation
\[
\acn{U}{v}:=\frac{1}{|U|}\sum_{u\in U}\rmp\!\big(\D{X_v}{X_u}\big),
\]
so that $\bbE\big[\cn{U}{v}\,\big|\,X_v,X_U\big]=\acn{U}{v}$.

Accordingly, we also define a notion of fluctuation error $\fe{U}$ that depends on $|U|$:
\begin{align}
    \label{_eq: def_fe}
    \fe{U} := \frac{\log n}{\sqrt{\sp \,|U|}}\,.
\end{align}
\end{defi}

\begin{lemma}[Edge fluctuation event]
\label{_lem: navi}
Assume the graph model of Definition~\ref{def:graph_model}, and suppose that
\[
\|\rmp\|_\infty:=\sup_{t\ge 0} |\rmp(t)|<\infty.
\]
Let \(U,V\subseteq\bfV\) be disjoint, and assume that
\[
\sp |U| \ge \log^2 n.
\]
Define
\[
\Enavi{U}{V}
:=
\left\{
\forall v\in V:\ 
\big|\cn{U}{v}-\acn{U}{v}\big|\le \fe{U}
\right\},
\qquad
\fe{U}:=\frac{\log n}{\sqrt{\sp\,|U|}}.
\]
 Then for all fixed
realizations \(X_V=x_V\) and \(X_U=x_U\),
\[
\Pr\!\big(\Enavi{U}{V}^{\,c}\ \big|\ X_V=x_V,\ X_U=x_U\big)\le n^{-\omega(1)}.
\]
\end{lemma}

\section{Geometric Comparison Tools}
\label{sec:geometric-comparison-tools}
In this section, we present the geometric comparison tools that are essential for our analysis. These tools allow us to relate the geometry of the manifold $M$ to that of the model spaces $M_{\pm\kappa}$, which have constant curvature. The key idea is that within a sufficiently small neighborhood (specifically, within the ball $B_M(p,r_M)$), the geometry of $M$ can be compared to that of the model spaces, and we can derive inequalities that relate distances and angles in $M$ to those in $M_{\pm\kappa}$. We refer the readers to the books \cite{CE08} and \cite{AKP24} for the related backgrounds on comparison theory.

\begin{defi}[Model Space]
\label{def: model_space}
For any dimension $d \ge 2$ and curvature $\kappa \in \mathbb{R}$, we denote by $M_\kappa^d$ (or simply $M_\kappa$) the unique simply connected, complete, $d$-dimensional Riemannian manifold with constant sectional curvature $\kappa$. Specifically:
\begin{itemize}
    \item If $\kappa > 0$, $M_\kappa$ is the sphere $\mathbb{S}^d$ with radius $1/\sqrt{\kappa}$.
    \item If $\kappa = 0$, $M_\kappa$ is the Euclidean space $\mathbb{R}^d$.
    \item If $\kappa < 0$, $M_\kappa$ is the hyperbolic space $\mathbb{H}^d$ with curvature $\kappa$.
\end{itemize}
We denote the diameter of the model space by $\dm{\kappa}$. Note that $\dm{\kappa} = \pi/\sqrt{\kappa}$ if $\kappa > 0$, and $\dm{\kappa} = \infty$ otherwise.
\end{defi}

\begin{defi}[Exponential Map and Logarithm]
\label{def: exp_log_map}
Let $(M,g)$ be a complete Riemannian manifold. For any point $p \in M$, the \emph{exponential map} at $p$, denoted $\exp_p : T_p M \to M$, maps a tangent vector $v \in T_p M$ to the point $\gamma(1)$, where $\gamma: [0,1] \to M$ is the unique geodesic starting at $p$ with initial velocity $\gamma'(0) = v$.

The exponential map is a diffeomorphism from a neighborhood of the origin in $T_p M$ onto a neighborhood of $p$ in $M$. Within the injectivity radius $\rinj(p)$, the inverse is well-defined. We denote this inverse map by the \emph{Riemannian logarithm}:
\[
    \log_p : B_M(p, \rinj(p)) \to B_{T_p M}(0, \rinj(p)).
\]
For any $x \in B_M(p, \rinj(p))$, the vector $\log_p(x)$ is the unique tangent vector in $T_p M$ such that $\exp_p(\tv{x}) = x$ and $\|\tv{x}\| = \D{p}{x}$.
\end{defi}

\begin{defi}[Model Space Side Length]
\label{def: os_kappa}
Let $\kappa \in \mathbb{R}$, $\theta\in[0,\pi]$, and $a,b>0$.
If $\kappa > 0$, we further assume $a+b < \pi/\sqrt{\kappa}$.
We write
\begin{align*}
    \os^\kappa(\theta;a,b)
\end{align*}
for the length of the side opposite the angle~$\theta$ in a geodesic triangle whose two adjacent sides forming the angle $\theta$ have lengths~$a$ and~$b$ in the model space $M^d_\kappa$ (the simply connected space of constant curvature $\kappa$). It is well-known that such a geodesic triangle is unique up to isometry in $M_\kappa^d$, and thus $\os^\kappa(\theta;a,b)$ is well-defined.
\end{defi}

\begin{defi}[Riemannian and Comparison Angles]
\label{def: comparison_angles}
  Let $p,x,y\in M$ with $x,y\in B(p,\rinj(M)) \setminus \{p\}$.  
  We define
  \[
     \ang{}{p}{x}{y}\;:=\;\angle \bigl(\tv{x},\tv{y}\bigr) = \arccos\left(\frac{\langle \tv{x},\tv{y}\rangle}{\|\tv{x}\| \, \|\tv{y}\|}\right),
  \]
  which is the angle at~$p$ of the geodesic triangle $\triangle pxy$ in~$M$.

  Given $\kappa>0$ and assuming $\D{p}{x}+\D{p}{y} + \D{x}{y} < 2\pi/\sqrt{\kappa}$, 
  we further set
  \[
     \ang{\kappa}{p}{x}{y}
     \quad\text{and}\quad
     \ang{-\kappa}{p}{x}{y}
  \]
  to be, respectively, the angles at~$\tilde p$ of the comparison triangle
  $\triangle\tilde p\tilde x\tilde y \subseteq M_\kappa$ and of the corresponding triangle in $M_{-\kappa}$—that is, the unique triangle in the model space whose side lengths equal
  $\D{p}{x}$, $\D{p}{y}$, and $\D{x}{y}$.
\end{defi}

\begin{rem}
 Every geodesic of length less than $r_M$ is indeed a unique minimizing geodesic between the two end points. Thus, for any three points in $M$ with pairwise distances less than $r_M$, we can form a geodesic triangle by connecting the points with minimizing geodesics. 
\end{rem}

\begin{lemma}[Triangle Comparison Inequalities]
\label{_lem: tri_lem}
Let $\kappa > 0$ be a positive constant. Let \((M,g)\) be a compact $d$-dimensional Riemannian manifold whose sectional curvatures satisfy
\[
\ang{\pm\kappa}{a,b,c}{}{}
\]
denote the angle opposite to the side of length $c$ in a geodesic triangle in the model space $M_{\pm\kappa}$ with side lengths $a,b,c$.
Let
\[
\os^{\pm\kappa}(\theta;a,b)
\]
denotes the length of the side opposite the angle $\theta$ in a model-space
triangle in $M_{\pm\kappa}$ whose two adjacent side lengths are $a$ and $b$.
Then
\begin{align}
\label{eq: tri_lem_os}
  \os^{\kappa}(\theta;a,b)
  \;\le\;
  c
  \;\le\;
  \os^{-\kappa}(\theta;a,b)
\end{align}
and 
\begin{align}
\label{eq: tri_lem_ang}
    \ang{-\kappa}{a,b,c}{}{}
    \;\le\;
    \theta
    \;\le\;
    \ang{\kappa}{a,b,c}{}{}.
\end{align}
\end{lemma}
The above triangle comparison result is a standard consequence of the Rauch-Comparison theorem. For reader who are interested in its proof, we refer to the proof of Lemma 3.1 in \cite{HJM25}.

\begin{lemma}[Law of cosines in model spaces, see e.g. \cite{CE08}]
\label{lem:model-law-of-cosines}
Let $a,b>0$ and $\theta\in[0,\pi]$.
\begin{itemize}
    \item If $\kappa>0$ and $c=\os^\kappa(\theta;a,b)$, then
    \[
    \cos(\sqrt{\kappa}\,c)
    =
    \cos(\sqrt{\kappa}\,a)\cos(\sqrt{\kappa}\,b)
    +
    \sin(\sqrt{\kappa}\,a)\sin(\sqrt{\kappa}\,b)\cos\theta,
    \]
    provided that $a,b< \pi/\sqrt{\kappa}$.
    \item If $\kappa=0$ and $c=\os^0(\theta;a,b)$, then
    \[
    c^2=a^2+b^2-2ab\cos\theta.
    \]
    \item If $\kappa< 0$ and $c=\os^{-\kappa}(\theta;a,b)$, then
    \[
    \cosh(\sqrt{\kappa}\,c)
    =
    \cosh(\sqrt{\kappa}\,a)\cosh(\sqrt{\kappa}\,b)
    -
    \sinh(\sqrt{\kappa}\,a)\sinh(\sqrt{\kappa}\,b)\cos\theta.
    \]
\end{itemize}
\end{lemma}

By combining Lemma~\ref{_lem: tri_lem} with
Lemma~\ref{lem:model-law-of-cosines}, we can transfer geometric
estimates from the model spaces $M_{\pm\kappa}$ back to $M$: given two
sides and an included angle, we can estimate the opposite side length in
$M$; conversely, given side-length information, we can estimate the
corresponding angle in $M$. The lemmas below are concrete instances of
this principle that we will invoke repeatedly.

\begin{lemma}[Geometric Distortion, Lemma 3.3 in \cite{HJM25}]
\label{_lem: M-opposite-side}
Consider three points in $M$ with pairwise distances less than $r_M/4$, and the geodesic triangle formed by these three points. Let $a,b,c$ be the lengths of the three sides, and $\theta$ be the angle opposite to the side of length $c$. 
Suppose $b \le \tfrac14 a$. Then, 
\[
 -\frac{7}{6}\frac{b^2}{a} - \frac{\kappa b^2}{3} \cdot b|\cos(\theta)|
\le
    a - c -  b \cos(\theta)  
\le
    \pi \frac{b}{a} \cdot b |\cos(\theta)|\,.
\]
A weaker but more concise version of the above is
\[
    |a - c - b \cos(\theta)| \le 4 \frac{b^2}{a}\,. 
\]
\end{lemma}

A consequence of the above lemma which was exploited in the paper \cite{HJM24,HJM25} is the following consequence (in \cite{HJM24} it was the ambient space analogue of the lemma below):
\begin{lemma}[Regularity, Lemma 3.6 in \cite{HJM25}]
\label{lem: regularity-geometry}    
\MAssump
Let $p, x, q \in M$ be three points in $M$ with
$$
    \D{p}{x}, \D{p}{q} < \frac{1}{8} \rM
    \qquad \mbox{and} \qquad 
    16 \frac{\D{p}{q}}{\D{p}{x}} \le  
     |\cos(\ang{}{p}{q}{x})| \,.
$$

Suppose $x' \in M$ is a point satisfying 
\begin{align*}
16 \frac{\D{x}{x'}}{\D{p}{x}}  \le |\cos(\ang{}{p}{q}{x})|. 
\end{align*}
Then, the following holds:
\begin{align}
    \frac{1}{2} 
\le \frac{\D{p}{x'} - \D{q}{x'}}{\cos(\ang{}{p}{q}{x}) \D{p}{q}} 
\le 2\,.
\end{align}
In particular, the sign of $\D{p}{x'} - \D{q}{x'}$ is the same as the sign of $\cos(\ang{}{p}{q}{x})$.
Furthermore, if $\rmp$ satisfies the bi-Lipschitz condition in Assumption~\ref{_def: distance-probability} and  
\begin{align*}
    \D{p}{x} \le \frac{1}{2} \rmrp\,,
\end{align*}
we have 
\begin{align*}
    \frac{1}{2} \ell_\rmp 
\le 
    \frac{\rmp(\D{q}{x'}) - \rmp(\D{p}{x'})}{\cos(\ang{}{p}{q}{x}) \D{p}{q}}
\le  
    2L_\rmp\,.
\end{align*} 

\end{lemma}

\begin{lemma}[Small Angle Bound, Lemma 3.5 in \cite{HJM25}]
\label{_lem:spherical-angle}
Fix $\kappa>0$. Consider a geodesic triangle in the model space $M_\kappa$ with
adjacent sides $a,b$ and angle $\theta$ between them.  
Assume
\[
      0<a,b<\tfrac14\,\dm{\kappa}
      \qquad\text{and}\qquad
      c:=\os^{\kappa}(\theta;a,b)\;\le\;\tfrac12\,a .
\]
Then
$$
    \theta \le 2 \frac{c}{a}\,.
$$ 
\end{lemma}

Here we have a better estimate when we have a right angle: 
\begin{lemma} [Right Triangles]
\label{lem: M-right-angle}  
There exists a universal constant $C_{\tref{lem: M-right-angle}}>1$ such that the following holds. 
Let $a,b,c$ denote the side lengths of a geodesic triangle in $M$ with
$a,b,c \le \rG$, and let $\theta$ be the angle opposite to side $c$. In addition, assume that $\alpha$, the angle opposite to side $a$, is $\pi/2$. Then the following holds:
First, 
\begin{align}
    \label{eq: M-sin-cos}
    \frac{1}{C_{\tref{lem: M-right-angle}}} a\sin(\theta) \le  c \le C_{\tref{lem: M-right-angle}} a\sin(\theta) \qquad \text{and} \qquad 
    \frac{1}{C_{\tref{lem: M-right-angle}}} a\cos(\theta) \le  b \le C_{\tref{lem: M-right-angle}} a\cos(\theta)\,.
\end{align}
Second, 
\begin{align}
    \label{eq:M-right-angle-difference}
a-b \ge \frac18 \min\!\left\{\frac{c^2}{b},\,c\right\},
\end{align}
with the convention that \(\frac{c^2}{b}=+\infty\) when \(b=0\). In addition, if $c \le b/4$, then we have the reverse bound
\begin{align*}
    a-b \le 4 \frac{c^2}{b}. 
\end{align*}
\end{lemma}
In the Euclidean case, we have $c=a\sin(\theta)$ and $b=a\cos(\theta)$ when $\alpha=\pi/2$, which is indeed how we define $\sin(\theta)$ and $\cos(\theta)$.

\subsection{Geometric Properties of Geodesics}
Beside the comparison geometry tools, we also need some geometric properties of geodesics in $M$. The following lemmas are standard in Riemannian geometry, and we provide the proofs in Appendix~\ref{app:deferred-geom-proofs} for completeness. 

\begin{lemma}[Distance to a nearby geodesic]
\label{lem:distance-to-nearby-geodesic}
Assume $M$ is a Riemannian manifold satisfying Assumption~\ref{_assump: manifold}. Let \(p\in M\), let \(q\in B_M(p,\tfrac12 \rM)\), and let
\[
\gamma:\mathbb R\to M
\]
be a unit-speed geodesic with \(\gamma(0)=q\). Let \((a,b)\subset\mathbb R\) be
the maximal open interval containing \(0\) such that
\[
\gamma((a,b))\subset B_M(p,\tfrac12 \rM).
\]
Assume that
\(
p\notin \gamma((a,b)).
\) and define
\[
f(t):=\D{p}{\gamma(t)},
\qquad t\in(a,b).
\]
Then \(f\) is smooth and strictly convex on \((a,b)\), and
\[
f'(t)
=
-\Big\langle
\frac{\exp_{\gamma(t)}^{-1}(p)}{\|\exp_{\gamma(t)}^{-1}(p)\|},
\,\dot\gamma(t)
\Big\rangle
=
-\cos\theta_t,
\]
where \(\theta_t\) is the angle between \(\dot\gamma(t)\) and the minimizing
geodesic from \(\gamma(t)\) to \(p\). In particular, \(f\) has a unique critical
point \(t_\star\), characterized by the condition that \(\dot\gamma(t_\star)\)
is orthogonal to the minimizing geodesic from \(\gamma(t_\star)\) to \(p\).
\end{lemma}

As a consequence of the above lemma, it allows us to properly projecting a nearby point onto a geodesic segment: 
\begin{cor}[Local unique projection onto a geodesic]
\label{cor:local-proj-geodesic}
Assume $M$ is a Riemannian manifold satisfying Assumption~\ref{_assump: manifold}. Fix \(q\in M\), and let
\[
\gamma:[-\rM/2,\rM/2]\to M
\]
be a unit-speed geodesic with \(\gamma(0)=q\). Then for every
\[
x\in B_M(q,\rM/4),
\]
there exists a unique point
\[
\pi_\gamma(x)\in \gamma([-\rM/2,\rM/2])
\]
such that
\[
\D{x}{\pi_\gamma(x)}
=
\min_{|s|\le \rM/2}\D{x}{\gamma(s)}.
\]
Moreover, \(\pi_\gamma(x)=\gamma(s(x))\) for some \(|s(x)|<\rM/2\), and if
\(x\notin \gamma([-\rM/2,\rM/2])\), then the minimizing geodesic from
\(\pi_\gamma(x)\) to \(x\) meets \(\gamma\) orthogonally at \(\pi_\gamma(x)\).
\end{cor}

\begin{lemma}[Two orthogonal perturbations of a geodesic segment]
\label{lem:two-orthogonal-perturbations}
Assume $M$ is a Riemannian manifold satisfying Assumption~\ref{_assump: manifold}.
Let $\gamma$ be a minimizing geodesic, and let $p,q\in \gamma$ with
\[
\D{p}{q}\le \tfrac{1}{16} \rM\,.
\]
Let $p',q'\in M$ satisfy that $p$ (resp.\ $q$) is a closest point of $p'$ (resp.\ $q'$) on $\gamma$.
Assume
\[
\D{p}{p'}\le \D{p}{q}/8,\qquad \D{q}{q'}\le \D{p}{q}/8.
\]
Then
\[
\big|\D{p'}{q'}-\D{p}{q}\big|
\le
C_{\tref{lem:two-orthogonal-perturbations}}\,\frac{(\D{p}{p'}+\D{q}{q'})^2}{\D{p}{q}}\,, 
\]
where $C_{\tref{lem:two-orthogonal-perturbations}}>0$ is a universal constant.
\end{lemma}

\section{Order of distance estimation}
\label{sec:distanceOrder}

We apply an $(r_0,\lambda_0,p_0)$ cluster generating algorithm to the induced subgraph on $\bfV_0$, whose vertex set has size $|\bfV_0|=n$, where $r_0$ is a sufficiently small constant multiple of $\rG$. Let $\Evt{cluster}$ denote the event that the algorithm succeeds. Thus, on $\Evt{cluster}$, for every $v\in\bfV_0$, the output set $B_v$ satisfies
\[
    B_v \subseteq \{v' \in \bfV_0 : \D{X_v}{X_{v'}} \le r_0\}
    \qquad\text{and}\qquad
    |B_v| \ge \lambda_0 n \mu_{\min}(r_0) \ge \lambda_0 c_\mu r_0^d n .
\]
In particular,
\[
\Pr(\Evt{cluster})\ge 1-p_0.
\]

We also work on the event $\Ept{\bfV_0}$. By Lemma~\ref{lem:Ept} and Assumption~\ref{_def: mu}, on $\Ept{\bfV_0}$ the sample $X_{\bfV_0}$ forms a $C(\log^2 n/n)^{1/d}$-net of $M$ for some absolute constant $C$.

Moreover, since each $B_v$ is determined by the induced subgraph on $\bfV_0$, it is independent of the edge variables used in the normalized averages against vertices in $\bfV_1\sqcup\bfV_2\sqcup\bfU_1\sqcup\bfU_2$. 

Accordingly, define
\[
    \Evt{base} := \Evt{cluster} \cap \Ept{\bfV_0}
    \cap \bigcap_{v\in\bfV_0} \Enavi{B_v}{\bfV_1\sqcup \bfV_2\sqcup \bfU_1 \sqcup \bfU_2}.
\]

\begin{lemma}
\label{lem:delta-cn-proxy}
There exists a sufficiently small constant $c_{\rm cn}>0$, depending only on $\ell_\rmp$ and $L_\rmp$, such that the following holds.
Assume $\Evt{base}$ holds, and suppose $r_0 \le c_{\rm cn}\rG$.
For $u,w\in \bfV_1 \sqcup \bfV_2 \sqcup \bfU_1 \sqcup \bfU_2$, define
\[
\Delta_{\rm cn}(u,w):=\max_{v\in\bfV_0} |\cn{B_v}{u} -\cn{B_v}{w}|.
\]
Then, for all $u,w\in\bfV_1 \sqcup \bfV_2 \sqcup \bfU_1 \sqcup \bfU_2$, if
\[
 \frac{\log^2 n}{\sqrt{\sp n}}
\le \Delta_{\rm cn}(u,w) \le c_{\rm cn}\rG,
\]
then
\[
\frac{1}{2L_{\rmp}}\,\Delta_{\rm cn}(u,w)
\le
\D{X_u}{X_w}
\le
\frac{6}{\ell_\rmp}\,\Delta_{\rm cn}(u,w).
\]
For simplicity, let $C_{\tref{lem:delta-cn-proxy}} := \max\{2L_{\rmp}, 6/\ell_\rmp\}$. Then the above inequalities can be rewritten as
\[
\frac{1}{C_{\tref{lem:delta-cn-proxy}}}\,\Delta_{\rm cn}(u,w)
\le
\D{X_u}{X_w}
\le
C_{\tref{lem:delta-cn-proxy}}\,\Delta_{\rm cn}(u,w).
\]
\end{lemma}

In other words, $\Delta_{\rm cn}(u,w)$ serves as a proxy for the latent distance $\D{X_u}{X_w}$, with multiplicative distortion, however, it works in a range that is dimension-free.

\begin{proof}

\step{Global upper bound on $\Delta_{\rm cn}$}
On $\Enavi{B_v}{\bfV_1\sqcup \bfV_2\sqcup \bfU_1 \sqcup \bfU_2}$, for any $v \in \bfV_0$ and $u,w \in \bfV_1\sqcup \bfV_2\sqcup \bfU_1 \sqcup \bfU_2$, we have
\begin{align*}
    |\cn{B_v}{u} -\cn{B_v}{w}|
    &\le    
        |\cn{B_v}{u} - \acn{B_v}{u}| + |\cn{B_v}{w} - \acn{B_v}{w}| + |\acn{B_v}{u} - \acn{B_v}{w}|\\
    &
        \le 2\fe{B_v} + L_{\rmp}\D{X_u}{X_w}.
\end{align*}
where we used the Lipschitz property of $\rmp$ for the last inequality and triangle inequality: 
\[
\begin{aligned}
|\acn{B_v}{u}-\acn{B_v}{w}|
&=
\left|
\frac1{|B_v|}\sum_{v'\in B_v}
\Big(\rmp(\D{X_u}{X_{v'}})-\rmp(\D{X_w}{X_{v'}})\Big)
\right| 
\le
\frac{L_{\rmp}}{|B_v|}\sum_{v'\in B_v}\D{X_u}{X_w}
=
L_{\rmp}\D{X_u}{X_w}.
\end{aligned}
\]
Hence, on $\Evt{base}$,  
\begin{align}
\label{eq:delta-cn-upper-exact}
\Delta_{\rm cn}(u,w)
\le
L_{\rmp}\D{X_u}{X_w} + 2\max_{v\in\bfV_0}\fe{B_v}.
\end{align}

\step{Good candidate for lower bound $\Delta_{\rm cn}$}
The reverse inequality requires a directional choice of the cluster center.
Fix $u,w\in\bfV_1\sqcup \bfV_2\sqcup \bfU_1 \sqcup \bfU_2$, and for now assume
\[
\delta := \D{X_u}{X_w} \le 0.01\rG.
\]
Let
\[
\varepsilon_n := C\Big(\frac{\log^2 n}{n}\Big)^{1/d},
\]
where $C>0$ is chosen so that, on $\Ept{\bfV_0}$, the sample $X_{\bfV_0}$ is an $\varepsilon_n$-net of $M$.
Recall that $\gamma_{u,w}$ is the fixed unit-speed geodesic with
\[
\gamma_{u,w}(0)=X_u,
\qquad
\gamma_{u,w}(\delta)=X_w.
\]
By $\Ept{\bfV_0}$, there exists $v\in\bfV_0$ such that
\[
\D{X_v}{\gamma_{u,w}(-0.4\rG)} \le \varepsilon_n.
\]
Let $\tilde X_v$ be a nearest point of $X_v$ on the image of $\gamma_{u,w}$. Then
\[
\tilde X_v=\gamma_{u,w}(t_v)
\quad\text{for some } t_v\in[-0.4\rG-\varepsilon_n,\,-0.4\rG+\varepsilon_n],
\]
and
\[
\D{X_v}{\tilde X_v}\le \varepsilon_n.
\]
Moreover, the minimizing geodesic from $\tilde X_v$ to $X_v$ is orthogonal to $\gamma_{u,w}$ at $\tilde X_v$, so the geodesic triangle with vertices
\[
X_w,\ \tilde X_v,\ X_v
\]
is right-angled at $\tilde X_v$.
Also, $X_u$ and $\tilde X_v$ lie on the same geodesic ray issuing from $X_w$, hence
\[
\ang{}{X_w}{X_u}{X_v}
=
\ang{}{X_w}{\tilde X_v}{X_v}.
\]
Applying Lemma~\ref{lem: M-right-angle} to the right triangle
$(X_w,\tilde X_v,X_v)$, we obtain
\[
\sin\!\big(\ang{}{X_w}{X_u}{X_v}\big)
=
\sin\!\big(\ang{}{X_w}{\tilde X_v}{X_v}\big)
\lesssim
\frac{\D{\tilde X_v}{X_v}}{\D{X_w}{\tilde X_v}}
\lesssim
\frac{\varepsilon_n}{\rG}.
\]
Therefore, for all sufficiently large $n$,
\[
\cos\!\big(\ang{}{X_w}{X_u}{X_v}\big)\ge \frac12.
\]

\step{Lower bound on $\Delta_{\rm cn}$}
Now apply Lemma~\ref{lem: regularity-geometry} with
\[
p:=X_w,\qquad q:=X_u,\qquad x:=X_v,\qquad x':=X_{v'}\quad (v'\in B_v).
\]
On $\Evt{cluster}$, for each $v'\in B_v$ we have
\[
\D{x}{x'}=\D{X_v}{X_{v'}}\le r_0.
\]
Also,
\[
\D{p}{q}=\delta\le 0.01\rG,
\]
and from the construction of $v$,
\[
\D{p}{x}=\D{X_w}{X_v}\ge 0.4\rG-\varepsilon_n.
\]
Moreover, since $\tilde X_v=\gamma_{u,w}(t_v)$ with
$t_v\in[-0.4\rG-\varepsilon_n,\,-0.4\rG+\varepsilon_n]$ and
$\D{X_v}{\tilde X_v}\le \varepsilon_n$, we also have
\[
\D{p}{x}
\le \D{X_w}{\tilde X_v}+\D{\tilde X_v}{X_v}
\le 0.42\rG+\varepsilon_n.
\]
Therefore, for large $n$ and $r_0$ a sufficiently small constant multiple of $\rG$,
\[
16\frac{\D{p}{q}}{\D{p}{x}}
\le
16\frac{0.01\rG}{0.4\rG-\varepsilon_n}
\le \frac12
\le
\cos\!\big(\ang{}{p}{q}{x}\big),
\]
and
\[
16\frac{\D{x}{x'}}{\D{p}{x}}
\le
16\frac{r_0}{0.4\rG-\varepsilon_n}
\le \frac12
\le
\cos\!\big(\ang{}{p}{q}{x}\big).
\]
Furthermore,
\[
\D{p}{q}<\frac{\rM}{8},
\qquad
\D{p}{x}<\frac{\rM}{8},
\qquad
\D{p}{x}\le \frac{\rmrp}{2},
\]
for all sufficiently large $n$, since $\rG\le \rM/16$ and $\rG\le \rmrp/16$.
Hence all hypotheses of Lemma~\ref{lem: regularity-geometry} hold, and for every $v'\in B_v$,
\[
\frac{\ell_\rmp}{2}
\le
\frac{\rmp(\D{X_u}{X_{v'}})-\rmp(\D{X_w}{X_{v'}})}
{\cos(\ang{}{X_w}{X_u}{X_v})\,\delta}
\le
2L_\rmp.
\]
Using $\cos(\ang{}{X_w}{X_u}{X_v})\in[1/2,1]$, we obtain
\[
\frac{\ell_\rmp}{4}\,\delta
\le
\rmp(\D{X_u}{X_{v'}})-\rmp(\D{X_w}{X_{v'}})
\le
2L_\rmp\,\delta,
\qquad \forall v'\in B_v.
\]
Therefore,
\[
\acn{B_v}{u} - \acn{B_v}{w}
=
\frac1{|B_v|}\sum_{v'\in B_v}
\Big(
\rmp(\D{X_u}{X_{v'}})-\rmp(\D{X_w}{X_{v'}})
\Big)
\ge \frac{\ell_\rmp}{4}\,\delta.
\]
Using $\Enavi{B_v}{\bfV_1\sqcup \bfV_2\sqcup \bfU_1 \sqcup \bfU_2}$,
\[
|\cn{B_v}{u}-\cn{B_v}{w}|
\ge
\frac{\ell_\rmp}{4}\,\D{X_u}{X_w}-2\fe{B_v}.
\]
Equivalently,
\[
\D{X_u}{X_w}
\le
\frac{4}{\ell_\rmp}\Big(|\cn{B_v}{u}-\cn{B_v}{w}|+2\fe{B_v}\Big).
\]
Taking the maximum over $v\in\bfV_0$ yields, for $\D{X_u}{X_w}\le 0.01\rG$,
\begin{align}
\label{eq:delta-cn-lower-exact}
\Delta_{\rm cn}(u,w)
\ge
\frac{\ell_\rmp}{4}\,\D{X_u}{X_w}
-2\max_{v\in\bfV_0}\fe{B_v}.
\end{align}

\step{Localization}
We claim that there exists a constant $c_{\rm loc}>0$, depending only on
$\ell_\rmp$, such that on $\Evt{base}$,
\begin{equation}
\label{eq:delta-cn-localization}
\D{X_u}{X_w}\ge 0.01\rG
\qquad\Longrightarrow\qquad
\Delta_{\rm cn}(u,w)\ge c_{\rm loc}\,\rG
\end{equation}
for all $u,w\in \bfV_1\sqcup\bfV_2\sqcup\bfU_1\sqcup\bfU_2$.
Indeed, fix such $u,w$ and write
\[
\delta:=\D{X_u}{X_w}\ge 0.01\rG.
\]
By $\Ept{\bfV_0}$, there exists $v\in\bfV_0$ such that
\[
\D{X_v}{X_u}\le \varepsilon_n.
\]
On $\Evt{cluster}$, for every $v'\in B_v$,
\[
\D{X_v}{X_{v'}}\le r_0,
\]
hence
\[
\D{X_u}{X_{v'}}\le r_0 + \varepsilon_n
\qquad\text{and}\qquad
\D{X_w}{X_{v'}}\ge \delta-(r_0 + \varepsilon_n).
\]
Since $\rmp$ is non-increasing, we obtain
\[
\rmp(\D{X_u}{X_{v'}})-\rmp(\D{X_w}{X_{v'}})
\ge
\rmp(r_0+\varepsilon_n)-\rmp\!\bigl(\min\{\delta-(r_0+\varepsilon_n),\rmrp\}\bigr)
\ge
 0.5 \ell_\rmp \delta  \,.
\]
Similar to the previous step, using $\Enavi{B_v}{\bfV_1\sqcup \bfV_2\sqcup \bfU_1 \sqcup \bfU_2}$, summing over $v'\in B_v$ and using $\Enavi{B_v}{\bfV_1\sqcup\bfV_2\sqcup\bfU_1\sqcup\bfU_2}$,
\[
|\cn{B_v}{u}-\cn{B_v}{w}|
\ge
0.5\ell_\rmp \delta-2\fe{B_v} \ge c_{\rm loc} \rG\,,
\]
for a suitable choice of $c_{\rm loc}>0$. 

\step{Conclusion}
On $\Evt{cluster}$, for every $v\in\bfV_0$,
\[
|B_v|\ge \lambda_0 c_\mu r_0^d n,
\]
and therefore
\[
\fe{B_v}
=
\frac{\log n}{\sqrt{\sp |B_v|}}
\le
\frac{1}{\sqrt{\lambda_0 c_\mu r_0^d}}
\frac{\log n}{\sqrt{\sp n}}
=
o\!\left(\frac{\log^2 n}{\sqrt{\sp n}}\right),
\]
since $\lambda_0$ and $r_0$ are $n$-independent.

Choose $c_{\rm cn}>0$ sufficiently small so that
\[
c_{\rm cn}\le c_{\rm loc}
\]
and also so that all previous smallness conditions on $r_0\le c_{\rm cn}\rG$ are satisfied.
Now assume
\[
\frac{\log^2 n}{\sqrt{\sp n}}
\le \Delta_{\rm cn}(u,w)\le c_{\rm cn}\rG.
\]
By \eqref{eq:delta-cn-localization}, this forces
\[
\D{X_u}{X_w}<0.01\rG.
\]
Hence the local lower bound \eqref{eq:delta-cn-lower-exact} applies, while the global upper bound \eqref{eq:delta-cn-upper-exact} always holds. Since the fluctuation term is negligible compared to
\[
\Delta_{\rm cn}(u,w)\ge \frac{\log^2 n}{\sqrt{\sp n}},
\]
we conclude that, for all sufficiently large $n$,
\[
\frac{1}{2L_{\rmp}}\,\Delta_{\rm cn}(u,w)
\le
\D{X_u}{X_w}
\le
\frac{6}{\ell_\rmp}\,\Delta_{\rm cn}(u,w).
\]
This proves the lemma.
\end{proof}

\section{Ring Extraction}
\label{sec:ringExtraction}
\begin{defi}[Proxy sets and mesoscopic pairs from $\Delta_{\rm cn}$ thresholds]
\label{def:mesoscopic-pairs-v1}
For each $v\in\bfV_1$, define
\[
\mathcal B_v^{(2)}
:=
\{v'\in\bfV_2:\ \Delta_{\rm cn}(v,v')\le \delta_n^3\xiq\}.
\]
For each $u\in\bfU_1$, define
\[
\ring{u}
:=
\{u'\in\bfU_2:\ \delta_n^2\xiq \le \Delta_{\rm cn}(u,u')\le \delta_n\xiq\}.
\]
Define also the set of mesoscopic pairs
\begin{equation}
\label{def:meso}
\meso
:=
\left\{
(u,v)\in \bfU_1\times\bfV_1:\ 
\delta_n^2\rG\le \Delta_{\rm cn}(u,v)\le \delta_n\rG
\right\}.
\end{equation}
\end{defi}

As a corollary of Lemma~\ref{lem:delta-cn-proxy}, these proxy sets are
comparable to the corresponding latent-distance-defined sets.

\begin{cor}[Inclusion relations for proxy sets and mesoscopic pairs]
\label{cor:mesoscopic-pairs-v1}
The following holds for all sufficiently large $n$. Assume $\Evt{base}\cap \Ept{\bfV_2}$.  
For every $v\in\bfV_1$,
\[
\Big\{
v'\in\bfV_2:\ \D{X_v}{X_{v'}}\le C_{\tref{lem:delta-cn-proxy}}^{-1}\delta_n^3\xiq
\Big\}
\subseteq
\mathcal B_v^{(2)}
\subseteq
\Big\{
v'\in\bfV_2:\ \D{X_v}{X_{v'}}\le C_{\tref{lem:delta-cn-proxy}}\delta_n^3\xiq
\Big\}.
\]

For every $u\in\bfU_1$,
\[
\Big\{
u'\in\bfU_2:\ 
C_{\tref{lem:delta-cn-proxy}}\delta_n^2\xiq
\le
\D{X_u}{X_{u'}}
\le
C_{\tref{lem:delta-cn-proxy}}^{-1}\delta_n\xiq
\Big\}
\subseteq
\ring{u}
\subseteq
\Big\{
u'\in\bfU_2:\ 
C_{\tref{lem:delta-cn-proxy}}^{-1}\delta_n^2\xiq
\le
\D{X_u}{X_{u'}}
\le
C_{\tref{lem:delta-cn-proxy}}\delta_n\xiq
\Big\}.
\]

Moreover,
\begin{align*}
&\Big\{
(u,v)\in \bfU_1\times\bfV_1:\ 
C_{\tref{lem:delta-cn-proxy}}\delta_n^2\rG
\le
\D{X_u}{X_v}
\le
C_{\tref{lem:delta-cn-proxy}}^{-1}\delta_n\rG
\Big\}
\\
&\hspace{5em}\subseteq
\meso
\\
&\hspace{5em}\subseteq
\Big\{
(u,v)\in \bfU_1\times\bfV_1:\ 
C_{\tref{lem:delta-cn-proxy}}^{-1}\delta_n^2\rG
\le
\D{X_u}{X_v}
\le
C_{\tref{lem:delta-cn-proxy}}\delta_n\rG
\Big\}.
\end{align*}

Finally, for every $v\in\bfV_1$,
\[
\big|\mathcal B_v^{(2)}\big|
\gtrsim
(\delta_n^3\xiq)^d\,n,
\]
where the implicit constant depends only on the fixed model parameters.
\end{cor}

\begin{defi}[Screened ring sets]
\label{def:angle-screened-subset}
Fix $(u,v)\in \bfU_1\times \bfV_1$.
Set
\[
\ring{u,v}
:=
\left\{
u'\in\ring{u}:\ 
\Big|\cn{{\cal B}_v^{(2)}}{u'}-\cn{{\cal B}_v^{(2)}}{u}\Big|
<
\frac{\ell_\rmp^2}{96L_\rmp^2}\,\delta_n^{-1}\,
\frac{\Deta_{\rm cn}(u,u')^2}{\Deta_{\rm cn}(u,v)}
\right\}.
\]
\end{defi}

\begin{rem}
\label{rem:edge-revelation}
No edges between $\bfV_1$ and $\bfU_2$ are revealed in the definition of the sets above.
In particular, the variables ${\cal U}_{\bfV_1,\bfU_2}$ remain unrevealed.
\end{rem}

We also introduce the event
\begin{equation}
\label{eq:angle-screened-events}
\Evt{R}
:=
\Ept{\bfV_2}\cap \bigcap_{v\in\bfV_1}\Enavi{{\cal B}_v^{(2)}}{\bfU_1\sqcup\bfU_2}.
\end{equation}

The next lemma shows that membership in $\ring{u,v}$ certifies approximate orthogonality, and conversely that sufficiently orthogonal points are admitted into $\ring{u,v}$.

\begin{lemma}[Certification for screened ring sets]
\label{lem:angle-screened-subset}
Assume $\Evt{base}\cap \Evt{R}$ holds.
Then, for all sufficiently large $n$, the following holds for every $(u,v)\in\meso$:
\begin{enumerate}
    \item[\textup{(i)}]
    If $u'\in \ring{u,v}$, then
    \[
    \Big|\cos\!\big(\theta_{u,v}(u')\big)\Big|
    <
    \delta_n^{-1}\,
    \frac{\D{X_u}{X_{u'}}}{\D{X_u}{X_v}}.
    \]

    \item[\textup{(ii)}]
    If $u'\in \ring{u}$ satisfies
    \[
    \Big|\cos\!\big(\theta_{u,v}(u')\big)\Big|
    \le
    20\,\frac{\D{X_u}{X_{u'}}}{\D{X_u}{X_v}},
    \]
    then $u'\in \ring{u,v}$.
\end{enumerate}
\end{lemma}

\begin{rem}[Coverage deferred]
\label{rem:coverage-deferred}
A third statement, namely a lower bound on $|\ring{u,v}|$, requires an additional latent-location event $\Evt{locring}$ depending only on the sampled positions
\[
X_{\bfV_1},\ X_{\bfV_2},\ X_{\bfU_1},\ X_{\bfU_2}.
\]
We postpone the definition of $\Evt{locring}$, as well as the lower bound on $|\ring{u,v}|$, Lemma \ref{lem:angle-screened-subset-coverage}, to Section~\ref{sec: latent-location-events}.
\end{rem}

\begin{cor}[Distance stability for screened ring points]
\label{cor:screened-ring-distance-stability}
Assume $\Evt{base}\cap \Evt{R}$ holds. Then for every $(u,v)\in\meso$, if
$u'\in\ring{u,v}$, then
\[
\Big|\D{X_u}{X_v}-\D{X_{u'}}{X_v}\Big|
\le
2C_{\tref{lem:delta-cn-proxy}}\,
\delta_n^{-1}\frac{\xiq^2}{\rG},
\]
where $C_{\tref{lem:delta-cn-proxy}}$ is the constant from Corollary~\ref{cor:mesoscopic-pairs-v1}.
\end{cor}
\begin{proof}
Fix $(u,v)\in\meso$ and $u'\in\ring{u,v}$. Set
\[
a:=\D{X_u}{X_v},\qquad
b:=\D{X_u}{X_{u'}},\qquad
c:=\D{X_{u'}}{X_v},\qquad
\theta:=\theta_{u,v}(u').
\]
By Corollary~\ref{cor:mesoscopic-pairs-v1}, we have
\[
C_{\tref{lem:delta-cn-proxy}}\delta_n^2\rG
\le
a
\le
C_{\tref{lem:delta-cn-proxy}}^{-1}\delta_n\rG,
\qquad
C_{\tref{lem:delta-cn-proxy}}^{-1}\delta_n^2\xiq
\le
b
\le
C_{\tref{lem:delta-cn-proxy}}\delta_n\xiq.
\]
Hence
\(
\frac{b}{a}
\lesssim
\delta_n^{-1}\frac{\xiq}{\rG}
=o(1)\,.
\)
We may therefore apply Lemma~\ref{_lem: M-opposite-side} to $\triangle X_uX_{u'}X_v$, obtaining
\[
|a-c|
\le
b|\cos\theta|+4\frac{b^2}{a}.
\]
Since $u'\in\ring{u,v}$, Lemma~\ref{lem:angle-screened-subset}\textup{(i)} gives
\(
|\cos\theta|
<
\delta_n^{-1}\frac{b}{a},
\)
and therefore
\[
|a-c|
\le
(\delta_n^{-1}+4)\frac{b^2}{a}.
\]
Using again the bounds on $a$ and $b$ from Corollary~\ref{cor:mesoscopic-pairs-v1},
\[
\frac{b^2}{a}
\le
\frac{(C_{\tref{lem:delta-cn-proxy}}\delta_n\xiq)^2}{C_{\tref{lem:delta-cn-proxy}}\delta_n^2\rG}
=
C_{\tref{lem:delta-cn-proxy}}\frac{\xiq^2}{\rG}.
\]
Thus
\[
|a-c|
\le
(\delta_n^{-1}+4)\,C_{\tref{lem:delta-cn-proxy}}\frac{\xiq^2}{\rG}
\le
2C_{\tref{lem:delta-cn-proxy}}\,
\delta_n^{-1}\frac{\xiq^2}{\rG}.
\]
\end{proof}

\subsection{Proof of Lemma~\ref{lem:angle-screened-subset}}
\label{subsec:proof-angle-screened-subset}

\begin{lemma}[Pointwise angle stability]
\label{lem:meso-angle-estimate-pointwise}
Assume $\Evt{base}$ holds, and let $(u,v)\in\meso$.
Fix $u'\in \ring{u}$ and define
\[
\theta_{u,v}(u'):=\ang{}{X_u}{X_v}{X_{u'}}.
\]
Then, for all sufficiently large $n$, for every $v'\in \mathcal B_v^{(2)}$ and every $K\ge 20$, the following statements hold:
\begin{enumerate}
    \item[\textup{(i)}]
    If
    \(
    \big|\cos\!\big(\theta_{u,v}(u')\big)\big|
    \ge
    K\,\frac{\D{X_u}{X_{u'}}}{\D{X_u}{X_v}},
    \)
    then
    \[
    {\rm sign} \left( \D{X_u}{X_{v'}}-\D{X_{u'}}{X_{v'}}\right) = {\rm sign}\!\big(\cos\!\big(\theta_{u,v}(u')\big)
    \quad \mbox{and}\quad 
    \big|\D{X_u}{X_{v'}}-\D{X_{u'}}{X_{v'}}\big|
    \ge
    \frac{K}{2}\,\frac{\D{X_u}{X_{u'}}^2}{\D{X_u}{X_v}}.
    \]

    \item[\textup{(ii)}]
    If
    \(
    \big|\cos\!\big(\theta_{u,v}(u')\big)\big|
    <
    K\,\frac{\D{X_u}{X_{u'}}}{\D{X_u}{X_v}},
    \)
    then
    \[
    \big|\D{X_u}{X_{v'}}-\D{X_{u'}}{X_{v'}}\big|
    \le
    2K\,\frac{\D{X_u}{X_{u'}}^2}{\D{X_u}{X_v}}.
    \]
\end{enumerate}
\end{lemma}

\begin{proof}
\step{1. Geometric separation}
Fix $u'\in \ring{u}$ and $v'\in \mathcal B_v^{(2)}$. 
By Corollary~\ref{cor:mesoscopic-pairs-v1}, since $(u,v)\in\meso$, we have
\[
C_{\tref{lem:delta-cn-proxy}}\delta_n^2\rG
\le
\D{X_u}{X_v}
\le
C_{\tref{lem:delta-cn-proxy}}^{-1}\delta_n\rG.
\]
Since $u'\in\ring{u}$ and $v'\in\mathcal B_v^{(2)}$, the same corollary gives
\[
C_{\tref{lem:delta-cn-proxy}}^{-1}\delta_n^2\xiq
\le
\D{X_u}{X_{u'}}
\le
C_{\tref{lem:delta-cn-proxy}}\delta_n\xiq,
\qquad
\D{X_v}{X_{v'}}
\le
C_{\tref{lem:delta-cn-proxy}}\delta_n^3\xiq.
\]
Together with triangle inequalities, these bounds imply the following separation of scales (for large $n$): 
$$
    \rM \gg \D{X_u}{X_v} \asymp \D{X_{u'}}{X_v} \asymp \D{X_u}{X_{v'}} \asymp \D{X_{u'}}{X_{v'}} \gg \D{X_u}{X_{u'}} \gg \D{X_v}{X_{v'}}.  
$$
\step{2. Angular stability}
Set
\[
\theta:=\theta_{u,v}(u')=\ang{}{X_u}{X_v}{X_{u'}},
\qquad
\phi_{v'}:=\ang{}{X_u}{X_{v'}}{X_{u'}}.
\]
Since both angles are based at $X_u$,
\[
|\phi_{v'}-\theta|
\le
\ang{}{X_u}{X_v}{X_{v'}}.
\]
Applying Lemmas~\ref{_lem: tri_lem} and~\ref{_lem:spherical-angle}
to $\triangle X_uX_vX_{v'}$ (valid due to the separation of scales established in Step 1), we obtain
\[
|\phi_{v'}-\theta|
\le
\ang{}{X_u}{X_v}{X_{v'}}
\le
2\frac{\D{X_v}{X_{v'}}}{\D{X_u}{X_v}}
\ll
\frac{\D{X_u}{X_{u'}}}{\D{X_u}{X_v}},
\]
for all sufficiently large $n$. 

\step{3. Linearization}
Apply Lemma~\ref{_lem: M-opposite-side} to the triangle
$\triangle X_uX_{u'}X_{v'}$, with angle $\phi_{v'}$ at $X_u$.
By the separation of scales from Step~1, all hypotheses of the lemma are satisfied. Hence
\[
\left|
\D{X_u}{X_{v'}}
-
\D{X_{u'}}{X_{v'}}
-
\D{X_u}{X_{u'}}\cos(\phi_{v'})
\right|
\le
4\frac{\D{X_u}{X_{u'}}^2}{\D{X_u}{X_{v'}}}
\le 
5\frac{\D{X_u}{X_{u'}}^2}{\D{X_u}{X_v}},
\]
where the last inequality uses $\D{X_u}{X_{v'}}\ge \D{X_u}{X_v}-\D{X_v}{X_{v'}}\ge (1-o(1))\D{X_u}{X_v}$ for large $n$. 
\step{4. Proof of \textup{(i)}}
Assume
\[
\big|\cos(\theta)\big|
\ge
K\,\frac{\D{X_u}{X_{u'}}}{\D{X_u}{X_v}},
\qquad K\ge 20.
\]
By Step~2 and the $1$-Lipschitz property of $\cos$,
\[
\big|\cos(\phi_{v'})\big|
\ge 
\big|\cos(\theta)\big| - |\phi_{v'}-\theta|
\ge 
\bigl(K-o(1)\bigr)\frac{\D{X_u}{X_{u'}}}{\D{X_u}{X_v}},
\]
and $\cos(\phi_{v'})$ has the same sign as $\cos(\theta)$ for all sufficiently large $n$.

Combining this with Step~3, we obtain
\[
\left|
\D{X_u}{X_{v'}}-\D{X_{u'}}{X_{v'}}
\right|
\ge
\bigl(K-o(1)-5\bigr)\frac{\D{X_u}{X_{u'}}^2}{\D{X_u}{X_v}}
\ge
\frac{K}{2}\,\frac{\D{X_u}{X_{u'}}^2}{\D{X_u}{X_v}},
\]
for all sufficiently large $n$, since $K\ge 20$.
As $\D{X_u}{X_{u'}}\cos(\phi_{v'})$ dominates the difference $5\frac{\D{X_u}{X_{u'}}^2}{\D{X_u}{X_v}}$, we have  
$$
{\rm sign} \left( \D{X_u}{X_{v'}}-\D{X_{u'}}{X_{v'}}\right) = {\rm sign}(\cos(\phi_{v'}))= {\rm sign}\!\big(\cos\!\big(\theta_{u,v}(u')\big)\big).
$$

\step{5. Proof of \textup{(ii)}}
Assume
\[
\big|\cos(\theta)\big|
<
K\,\frac{\D{X_u}{X_{u'}}}{\D{X_u}{X_v}}.
\]
Again by Step~2 and the $1$-Lipschitz property of $\cos$,
\[
\big|\cos(\phi_{v'})\big|
\le
\big|\cos(\theta)\big|
+
\big|\phi_{v'}-\theta\big|
\le
\bigl(K+o(1)\bigr)\frac{\D{X_u}{X_{u'}}}{\D{X_u}{X_v}}.
\]
Therefore, using Step~3,
\[
\left|
\D{X_u}{X_{v'}}-\D{X_{u'}}{X_{v'}}
\right|
\le
\left|
\D{X_u}{X_{u'}}\cos(\phi_{v'})
\right|
+
5\frac{\D{X_u}{X_{u'}}^2}{\D{X_u}{X_v}}
\le
\bigl(K+o(1)+5\bigr)\frac{\D{X_u}{X_{u'}}^2}{\D{X_u}{X_v}}
\le
2K\,\frac{\D{X_u}{X_{u'}}^2}{\D{X_u}{X_v}},
\]
for all sufficiently large $n$, since $K\ge 20$. This proves \textup{(ii)}.
\end{proof}
\begin{proof}[Proof of Lemma~\ref{lem:angle-screened-subset}]
Fix $u'\in\ring{u}$. Write
\[
\cn{{\cal B}_v^{(2)}}{u'}-\cn{{\cal B}_v^{(2)}}{u}
=
\Big(\cn{{\cal B}_v^{(2)}}{u'}-\acn{{\cal B}_v^{(2)}}{u'}\Big)
+
\Big(\acn{{\cal B}_v^{(2)}}{u'}-\acn{{\cal B}_v^{(2)}}{u}\Big)
+
\Big(\acn{{\cal B}_v^{(2)}}{u}-\cn{{\cal B}_v^{(2)}}{u}\Big).
\]

\step{1. The empirical error is negligible}
Since $u\in\bfU_1$ and $u'\in\ring{u}\subseteq\bfU_2$, the event $\Evt{R}$ gives
\[
\max\left\{
\Big|\cn{{\cal B}_v^{(2)}}{u'}-\acn{{\cal B}_v^{(2)}}{u'}\Big|,
\,
\Big|\cn{{\cal B}_v^{(2)}}{u}-\acn{{\cal B}_v^{(2)}}{u}\Big|
\right\}
\le
\fe{{\cal B}_v^{(2)}}
=
\frac{\log n}{\sqrt{\sp\,|{\cal B}_v^{(2)}|}}.
\]
By Corollary~\ref{cor:mesoscopic-pairs-v1},
\[
|{\cal B}_v^{(2)}|
\gtrsim
(\delta_n^3\xiq)^d\,n.
\]
Together with
\[
\xiq=\left(\frac{1}{\sp n}\right)^{1/(d+{\frak q})}(\log n)^{C_{\tref{def:xi-q}}}
\quad \Rightarrow\quad
\sp n\,\xiq^{d+{\frak q}}=(\log n)^{C_{\tref{def:xi-q}}(d+{\frak q})}
\]
this gives
\[
\fe{{\cal B}_v^{(2)}}
\lesssim
\delta_n^{-3d/2}\frac{\log n}{\sqrt{\sp\,n\,\xiq^d}}
=
\delta_n^{-1-3d/2+\frac{C_{\tref{def:xi-q}}(d+{\frak q})}{2}}\xiq^{{\frak q}/2}.
\]

On the other hand, since $u'\in\ring{u}$ and $(u,v)\in\meso$, Corollary~\ref{cor:mesoscopic-pairs-v1} also gives
\[
\frac{\D{X_u}{X_{u'}}^2}{\D{X_u}{X_v}}
\gtrsim
\delta_n^3\frac{\xiq^2}{\rG}.
\]
Since ${\frak q}=5$, we have
\[
\delta_n^{-1-3d/2+\frac{C_{\tref{def:xi-q}}(d+{\frak q})}{2}}\xiq^{{\frak q}/2}
=
o\!\left(\delta_n^3\frac{\xiq^2}{\rG}\right),
\]
provided that
\[
C_{\tref{def:xi-q}}>\frac{3d+8}{d+5}.
\]
Indeed, under ${\frak q}=5$ we have
\[
\frac{
\delta_n^{-1-3d/2+\frac{C_{\tref{def:xi-q}}(d+{\frak q})}{2}}\xiq^{{\frak q}/2}
}{
\delta_n^3\frac{\xiq^2}{\rG}
}
\lesssim
\rG\,\delta_n^{-4-\frac{3d}{2}+\frac{C_{\tref{def:xi-q}}(d+5)}{2}}\xiq^{1/2}
=o(1),
\]
since \(\xiq\to0\) and the exponent of \(\delta_n\) is positive. In particular,
any choice \(C_{\tref{def:xi-q}}\ge 4\) works. Therefore
\[
\bigg|
\Big(\cn{{\cal B}_v^{(2)}}{u'}-\cn{{\cal B}_v^{(2)}}{u}\Big)
-
\Big(\acn{{\cal B}_v^{(2)}}{u'}-\acn{{\cal B}_v^{(2)}}{u}\Big)
\bigg|
=
o\!\left(\frac{\D{X_u}{X_{u'}}^2}{\D{X_u}{X_v}}\right).
\]

\step{2. Lower bound under the large-angle condition}
Assume
\[
|\cos\theta_u|
\ge
K\frac{\D{X_u}{X_{u'}}}{\D{X_u}{X_v}},
\qquad K\ge 20.
\]
By Lemma~\ref{lem:meso-angle-estimate-pointwise}, for every $v'\in{\cal B}_v^{(2)}$,
\[
\D{X_u}{X_{v'}}-\D{X_{u'}}{X_{v'}}
\]
has the same sign for all $v'$, and
\[
\Big|\D{X_u}{X_{v'}}-\D{X_{u'}}{X_{v'}}\Big|
\ge
\frac{K}{2}\,\frac{\D{X_u}{X_{u'}}^2}{\D{X_u}{X_v}}.
\]
Moreover, by the separation of scales, both
\[
\D{X_u}{X_{v'}},\qquad \D{X_{u'}}{X_{v'}}
\]
lie in $[0,\rmrp]$ for all sufficiently large $n$. Since $\rmp$ is non-increasing and has lower Lipschitz constant $\ell_\rmp$, the quantities
\[
\rmp(\D{X_{u'}}{X_{v'}})-\rmp(\D{X_u}{X_{v'}})
\]
also have the same sign for all $v'$, and satisfy
\[
\Big|
\rmp(\D{X_{u'}}{X_{v'}})-\rmp(\D{X_u}{X_{v'}})
\Big|
\ge
\frac{\ell_\rmp K}{2}\,\frac{\D{X_u}{X_{u'}}^2}{\D{X_u}{X_v}}.
\]
Averaging over $v'\in{\cal B}_v^{(2)}$ and using Step~1, we obtain
\[
\Big|\cn{{\cal B}_v^{(2)}}{u'}-\cn{{\cal B}_v^{(2)}}{u}\Big|
\ge
\frac{\ell_\rmp K}{4}\,\frac{\D{X_u}{X_{u'}}^2}{\D{X_u}{X_v}}
\]
for all sufficiently large $n$.

\step{3. Upper bound under the small-angle condition}
Assume
\[
|\cos\theta_u|
<
K\frac{\D{X_u}{X_{u'}}}{\D{X_u}{X_v}}.
\]
Using Lemma~\ref{lem:meso-angle-estimate-pointwise} and the upper Lipschitz constant $L_\rmp$, we obtain
\[
\Big|
\acn{{\cal B}_v^{(2)}}{u'}-\acn{{\cal B}_v^{(2)}}{u}
\Big|
\le
2L_\rmp K\,\frac{\D{X_u}{X_{u'}}^2}{\D{X_u}{X_v}}.
\]
Combining this with Step~1 gives
\[
\Big|\cn{{\cal B}_v^{(2)}}{u'}-\cn{{\cal B}_v^{(2)}}{u}\Big|
\le
4L_\rmp K\,\frac{\D{X_u}{X_{u'}}^2}{\D{X_u}{X_v}},
\]
for all sufficiently large $n$.

\step{4. Conversion to $\Deta_{\rm cn}$}
Since $u'\in\ring{u}$ and $(u,v)\in\meso$, Lemma~\ref{lem:delta-cn-proxy} applies to both pairs $(u,u')$ and $(u,v)$. Hence
\[
\frac{1}{2L_\rmp}\,\Deta_{\rm cn}(u,w)
\le
\D{X_u}{X_w}
\le
\frac{6}{\ell_\rmp}\,\Deta_{\rm cn}(u,w),
\qquad w\in\{u',v\}.
\]
Therefore
\[
\frac{\D{X_u}{X_{u'}}^2}{\D{X_u}{X_v}}
\ge
\frac{\ell_\rmp}{24L_\rmp^2}\,
\frac{\Deta_{\rm cn}(u,u')^2}{\Deta_{\rm cn}(u,v)},
\qquad
\frac{\D{X_u}{X_{u'}}^2}{\D{X_u}{X_v}}
\le
\frac{72L_\rmp}{\ell_\rmp^2}\,
\frac{\Deta_{\rm cn}(u,u')^2}{\Deta_{\rm cn}(u,v)}.
\]
Combining these bounds with Steps~2 and~3 gives the two statements in the lemma.
\end{proof}

\section{Good events on the latent location distribution}
\label{sec: latent-location-events}

In this subsection, we introduce several high-probability events concerning the
latent sample locations. These objects are used only in the analysis of the
algorithm and are not assumed to be observable from the graph.

Locally near each geodesic \(\gamma_{p,q}\), we can describe a point in $M$ in terms of its position along the geodesic and its displacement in the orthogonal directions. This is achieved by the Fermi coordinate construction, which we now describe. 

\paragraph{Euclidean picture.}
It is helpful to begin with the flat model in \(\mathbb R^d\). Any geodesic is a
straight line, and after a rigid motion we may assume that a reference geodesic
is the first coordinate axis,
\[
\gamma(t)=(t,0,\dots,0).
\]
A point near this line is then described by two pieces of data: its position
\(t\) along the line, and its displacement \(y\in\mathbb R^{d-1}\) in the
orthogonal directions. Thus a tubular neighborhood of \(\gamma\) is
parameterized by
\[
(t,y)\in\mathbb R\times \mathbb R^{d-1}.
\]
In these coordinates, the sets used below are simply short intervals in the
\(t\)-variable together with annuli in the transverse variable \(y\).

On a Riemannian manifold, we use the same idea along a geodesic. Fix a
unit-speed geodesic \(\gamma:J\to M\). Choose an orthonormal basis of
\(\dot\gamma(t_0)^\perp\) at one point \(t_0\in J\), and parallel transport it
along \(\gamma\). This yields an orthonormal frame
\[
V_1(t),\dots,V_{d-1}(t),\qquad t\in J,
\]
such that \(V_1(t),\dots,V_{d-1}(t)\) are orthogonal to \(\dot\gamma(t)\) for
every \(t\in J\), and together with \(\dot\gamma(t)\) form an orthonormal basis
of \(T_{\gamma(t)}M\).

We then define the associated Fermi map by
\[
\Phi_\gamma(t,y)
:=
\exp_{\gamma(t)}\!\left(\sum_{i=1}^{d-1} y^iV_i(t)\right),
\qquad
(t,y)\in J\times\mathbb R^{d-1}.
\]
Geometrically, \(\Phi_\gamma(t,y)\) is obtained by starting at \(\gamma(t)\)
and moving in the orthogonal direction determined by \(y\). This map is always
well-defined. On sufficiently short windows and sufficiently small transverse
scales, particularly on the scale \(\rM\),
it provides a genuine tubular coordinate chart around the geodesic. We
state the quantitative version below. 
\begin{lemma}[Tubular coordinates on a short geodesic window]
\label{lem:fermi-short-window}
There exists a universal constant \(c_{\rm fm}\in(0,1)\) such that the following holds.

Let \(\gamma:J\to M\) be a unit-speed geodesic. If \(I\subseteq J\) has length
at most \(c_{\rm fm}\rM\), then the restriction
\[
\Phi_\gamma\big|_{\,I\times B_{\mathbb R^{d-1}}(0,c_{\rm fm}\rM)}
\]
is injective. Moreover, it is \(2\)-bi-Lipschitz onto its image: for all
\(t,t'\in I\) and all \(y,y'\in B_{\mathbb R^{d-1}}(0,c_{\rm fm}\rM)\),
\[
\frac{1}{2}\sqrt{|t-t'|^2+\|y-y'\|_2^2}
\le
\D{\Phi_\gamma(t,y)}{\Phi_\gamma(t',y')}
\le
2\sqrt{|t-t'|^2+\|y-y'\|_2^2}.
\]
\end{lemma}
This is rather a standard result in Riemannian geometry. Essentially, it follows again from comparison results, specifically via Jacobi-field comparisons: We first show the pullback metric \(\Phi_\gamma^*g\) is close to the Euclidean metric on the relevant domain (see Lemma~\ref{lem:fermi-metric-comparison}). This relies on the Jacobi-field comparison, specfically we use the integral form of the Jacobi equation and represent the solution as a Volterra integral equation, which allows us to control the deviation of the solution from the Euclidean case in terms of the curvature bound and the size of the domain. Once we have the pullback metric close to Euclidean, we can deduce that \(\Phi_\gamma\) is bi-Lipschitz by comparing lengths of curves in the domain and their images under \(\Phi_\gamma\), with some proper scaling to a smaller domain than the one where the pullback metric is close to Euclidean. The concrete proof is given in Appendix~\ref{subsec: geodesics-fermi-coords} for completeness.

Fix a unit-speed geodesic \(\gamma:J\to M\), and let \(\Phi_\gamma\) be a Fermi
map along \(\gamma\). For \(t\in J\), radii \(0\le \rho_-\le \rho_+\), and
\(h>0\), we define the associated annular slab by
\begin{align}
    \label{eq:fermi-slab-def}
\mathcal S_\gamma(t;\rho_-,\rho_+,h)
:=
\left\{
\Phi_\gamma(s,y):
|s-t|\le h,\ \rho_-\le \|y\|_2\le \rho_+
\right\}\,.
\end{align}
\begin{rem}
The definition of \(\mathcal S_\gamma(t;\rho_-,\rho_+,h)\) does not depend on
the choice of transported orthonormal frame, since changing the frame only
rotates the transverse variable \(y\), while \(\|y\|_2\) is rotation-invariant.
\end{rem}

If $h,\rho_-,\rho_+$ are also smaller than ${\rm r}_\mu$, then we have a good control on the measure of $\mathcal S_\gamma(t;\rho_-,\rho_+,h)$, as stated in the next lemma.

\begin{lemma}[Uniform lower mass of moving annular slabs]
\label{lem:moving-ring-mass}
Let \(\gamma:J\to M\) be a unit-speed geodesic, let \(t\in J\), and let
\(0\le \rho_-<\rho_+\) and \(h>0\). Assume
\[
[t-h,t+h]\subseteq J,
\qquad
2h\le c_{\rm fm}\rG,
\qquad
\rho_+\le c_{\rm fm}\rG.
\]
Then
\[
\mu\!\left(\mathcal S_\gamma(t;\rho_-,\rho_+,h)\right)
\gtrsim
h\Big((\rho_+)^{d-1}-(\rho_-)^{d-1}\Big),
\]
where the implicit constant depends only on the fixed model parameters.
\end{lemma}
We postpone the proof of this lemma to the end of this section. 
We now specialize the preceding geometric construction to our random sample.
Recall that for each ordered pair \((w,v)\in\bfV\times\bfV\), we have 
\[
\gamma_{w,v}:=\gamma_{X_w,X_v}\,,
\]
see the Definition \ref{def:prefixed-geodesics} of the prefixed geodesics.
Define
\[
h_n:=\delta_n^4\xiq^2,
\qquad
\rho_n^-:=\frac78\,\delta_n^{1.5}\xiq,
\qquad
\rho_n^+:=\frac98\,\delta_n^{1.5}\xiq.
\]
For \(t\in[-\rM,\D{X_w}{X_v}+\rM]\), let
\begin{align}
\label{eq:Svwt-def}
\mathcal S(w,v,t)
:=
\mathcal S_{\gamma_{w,v}}(t;\rho_n^-,\rho_n^+,h_n).
\end{align}
Thus \(\mathcal S(w,v,t)\) is an annular slab of longitudinal thickness
\(h_n=\delta_n^4\xiq^2\), centered at time \(t\) along \(\gamma_{w,v}\), and
lying at transverse distance of order \(\delta_n^{1.5}\xiq\) from the geodesic. 
\begin{rem}
    \label{rem:Svwt-geometry}
    For every $y \in \mathcal S(w,v,t)$, let $\tilde y$ be the projection of $y$ onto the geodesic $\gamma_{w,v}$. From the definition of $\mathcal S(w,v,t)$, we have    
    $$
        \frac78\,\delta_n^{1.5}\xiq \le \D{y}{\tilde y} \le  \frac98\,\delta_n^{1.5}\xiq
          \quad \text{and}\quad |t - \D{\tilde y}{X_w}| \le h_n = \delta_n^4\xiq^2.
    $$
\end{rem}

In particular, Lemma \ref{lem:moving-ring-mass} suggests that \(\mu(\mathcal S(w,v,t))\) is of order \(\xiq^{d+1}\) up to some powers of \(\delta_n\). Our main lemma in this section is the following.   
\begin{lemma}[Uniform lower occupancy of moving rings]
\label{lem:U1-moving-ring-occupancy}
Define the event
\begin{align*}
\Evt{locRing}
:=
\bigg\{
\forall (w,v)\in\bfV^2 \text{ with } w\neq v,\ 
\forall t\in[-\rM,\D{X_w}{X_v}+\rM],\ 
\forall \mathbf C&\in\{\mathbf U_1,\mathbf U_2,\mathbf V_0, \mathbf V_1,\mathbf V_2\},\\
&
|X_{\mathbf C}\cap \mathcal S(w,v,t)|
\ge n\,\delta_n^{5d}\xiq^{d+1}
\bigg\}.
\end{align*}
Then
\[
\Pr(\Evt{locRing})\ge 1-n^{-\omega(1)}.
\]
\end{lemma}

Before proving Lemma~\ref{lem:U1-moving-ring-occupancy}, we show how it implies the desired lower bound on the size of the screened ring sets \(\ring{u,v}\) for \((u,v)\in\meso\), and in particular, why we include the logarithmic factor in the definition of \(\xiq\). 

\begin{lemma}[Coverage of screened ring sets]
\label{lem:angle-screened-subset-coverage}
Assume $\Evt{base}\cap \Evt{R}\cap 
\Evt{locRing}$ holds.
Then, for all sufficiently large $n$, for every $(u,v)\in\meso$,
\[
|\ring{u,v}|
\ge
n\,\delta_n^{5d}\xiq^{d+1}.
\]
Consequently, for the choice of \(C_{\tref{def:xi-q}}\) sufficiently large,
\[
\fe{\ring{u,v}}
\le
\frac{\xiq^2}{\rG}\,.
\]
\end{lemma}

\begin{proof}
Fix $(u,v)\in\meso$, and set
\[
t_{u,v}:=\D{X_v}{X_u},
\qquad
\mathcal S_{u,v}:=\mathcal S(v,u,t_{u,v}).
\]
Since $\Evt{locRing}$ holds, taking $(w,v)=(v,u)$ and
$\mathbf C=\mathbf U_2$ in the definition of $\Evt{locRing}$ gives
\[
|X_{\mathbf U_2}\cap \mathcal S_{u,v}|
\ge
n\,\delta_n^{5d}\xiq^{d+1}.
\]
We claim that
\[
X_{\mathbf U_2}\cap \mathcal S_{u,v}\subseteq \ring{u,v}.
\]
This immediately yields the desired lower bound on $|\ring{u,v}|$.

Let \(u''\in \mathbf U_2\) satisfy \(X_{u''}\in \mathcal S_{u,v}\).
By definition of \(\mathcal S_{u,v}\), there exist
\[
s\in[t_{u,v}-h_n,t_{u,v}+h_n],
\qquad
y\in\mathbb R^{d-1},
\qquad
\rho_n^-\le \|y\|_2\le \rho_n^+,
\]
such that
\[
X_{u''}=\Phi_{\gamma_{v,u}}(s,y).
\]
Since \(X_u=\Phi_{\gamma_{v,u}}(t_{u,v},0)\), Lemma~\ref{lem:fermi-short-window}
gives
\[
\frac12\sqrt{|s-t_{u,v}|^2+\|y\|_2^2}
\le
\D{X_u}{X_{u''}}
\le
2\sqrt{|s-t_{u,v}|^2+\|y\|_2^2}.
\]
Because
\[
|s-t_{u,v}|\le h_n=\delta_n^4\xiq^2
\qquad\text{and}\qquad
\rho_n^-\le \|y\|_2\le \rho_n^+
\asymp \delta_n^{1.5}\xiq,
\]
we obtain
\[
\D{X_u}{X_{u''}}
\asymp
\delta_n^{1.5}\xiq.
\]
Since
\[
\delta_n^2\xiq  \ll \delta_n\xiq,
\]
Corollary~\ref{cor:mesoscopic-pairs-v1} implies that \(u''\in\ring{u}\) for all
sufficiently large \(n\).

It remains to show that \(u''\in\ring{u,v}\). Let
\[
z:=\gamma_{v,u}(s).
\]
Then
\[
\D{X_u}{z}=|s-t_{u,v}|\le h_n.
\]
Moreover, by construction of the Fermi map, the geodesic from \(z\) to \(X_{u''}\)
is orthogonal to \(\gamma_{v,u}\) at \(z\). Since \(z\) lies on the same geodesic
line as \(X_v\) and \(X_u\), we have
\[
\big|\cos(\theta_{u,v}(u''))\big|
=
\big|\cos\!\big(\ang{}{X_u}{z}{X_{u''}}\big)\big|.
\]
Applying Lemma~\ref{lem: M-right-angle} to the right triangle with vertices
\[
X_u,\ z,\ X_{u''},
\]
we obtain
\[
\big|\cos(\theta_{u,v}(u''))\big|
\lesssim
\frac{\D{X_u}{z}}{\D{X_u}{X_{u''}}}
\lesssim
\frac{h_n}{\delta_n^{1.5}\xiq}
=
\delta_n^{2.5}\xiq.
\]
On the other hand, since \((u,v)\in\meso\), Corollary~\ref{cor:mesoscopic-pairs-v1}
gives
\[
\D{X_u}{X_v}\le C_{\tref{lem:delta-cn-proxy}}^{-1}\delta_n\rG.
\]
Combining this with \(\D{X_u}{X_{u''}}\asymp \delta_n^{1.5}\xiq\), we get
\[
\frac{\D{X_u}{X_{u''}}}{\D{X_u}{X_v}}
\gtrsim
\delta_n^{0.5}\frac{\xiq}{\rG}.
\]
Since \(\delta_n^{2.5}\xiq=o\!\bigl(\delta_n^{0.5}\xiq/\rG\bigr)\), it follows that
\[
\big|\cos(\theta_{u,v}(u''))\big|
\le
20\,\frac{\D{X_u}{X_{u''}}}{\D{X_u}{X_v}}
\]
for all sufficiently large \(n\). Hence Lemma~\ref{lem:angle-screened-subset}\textup{(ii)}
yields \(u''\in\ring{u,v}\), proving the claim
\[
X_{\mathbf U_2}\cap \mathcal S_{u,v}\subseteq \ring{u,v}.
\]

Therefore
\[
|\ring{u,v}|
\ge
|X_{\mathbf U_2}\cap \mathcal S_{u,v}|
\ge
n\,\delta_n^{5d}\xiq^{d+1}.
\]
Finally,
\[
\fe{\ring{u,v}}
=
\frac{\log n}{\sqrt{\sp\,|\ring{u,v}|}}
\le
\delta_n^{-5d/2}\frac{\log n}{\sqrt{\sp\,n\,\xiq^{d+1}}}
=
\delta_n^{-1-5d/2+\frac{C_{\tref{def:xi-q}}(d+{\frak q})}{2}}\xiq^{({\frak q}-1)/2},
\]
using \(\sp n\,\xiq^{d+{\frak q}}=(\log n)^{C_{\tref{def:xi-q}}(d+{\frak q})}\). Since
\({\frak q}=5\), the exponent \(({\frak q}-1)/2\) equals \(2\). Therefore, for
\(C_{\tref{def:xi-q}}\ge 6\), we have  
\[
\fe{\ring{u,v}}
\le
\frac{\xiq^2}{\rG}.
\]
This proves the lemma.
\end{proof}

\subsection{Proof of Lemma~\ref{lem:U1-moving-ring-occupancy}}
\label{subsec:proof-moving-ring-occupancy}
\begin{proof}[Proof of Lemma~\ref{lem:U1-moving-ring-occupancy}]
Set
\[
h_n:=\delta_n^4\xiq^2,
\qquad
\rho_n^-:=\frac78\,\delta_n^{1.5}\xiq,
\qquad
\rho_n^+:=\frac98\,\delta_n^{1.5}\xiq,
\qquad
\eta_n:=\delta_n^{5d}\xiq^{d+1},
\]
and let
\[
\mathcal S^{\rm aux}(w,v,t)
:=
\mathcal S_{\gamma_{w,v}}\!\left(t;\rho_n^-,\rho_n^+,\frac12 h_n\right).
\]

Fix an ordered pair \((w,v)\in\bfV^2\) with \(w\neq v\), and fix
\[
\mathbf C\in\{\mathbf U_1,\mathbf U_2,\bfV_0, \mathbf V_1,\mathbf V_2\}.
\]
We show that, with probability at least \(1-n^{-\omega(1)}\),
\[
|X_{\mathbf C}\cap \mathcal S(w,v,t)|\ge n\eta_n
\qquad
\forall t\in[-\rM,\D{X_w}{X_v}+\rM].
\]

By Lemma~\ref{lem:moving-ring-mass},
\[
\mu\!\left(\mathcal S^{\rm aux}(w,v,t)\right)
\gtrsim
h_n\Big((\rho_n^+)^{d-1}-(\rho_n^-)^{d-1}\Big)
\asymp
\delta_n^{1.5d+2.5}\xiq^{d+1},
\]
uniformly in \((w,v,t)\). Since \(1.5d+2.5<5d\), it follows that, for all
sufficiently large \(n\),
\[
\mu\!\left(\mathcal S^{\rm aux}(w,v,t)\right)\ge 4\eta_n
\qquad
\forall (w,v),\ \forall t.
\]

Next discretize
\[
[-\rM,\D{X_w}{X_v}+\rM]
\]
by a grid \(t_0,\dots,t_J\) of mesh at most \(h_n/4\). Then \(J\lesssim h_n^{-1}\).
For each \(j\), conditional on \(X_w,X_v\),
\[
N_j:=|X_{\mathbf C\setminus\{w,v\}}\cap \mathcal S^{\rm aux}(w,v,t_j)|
\sim {\rm Binomial}(n',p_j),
\]
where \(n':=|\mathbf C\setminus\{w,v\}|\ge n-2\) and
\[
p_j=\mu\!\left(\mathcal S^{\rm aux}(w,v,t_j)\right)\ge 4\eta_n.
\]
Hence \(\mathbb E[N_j\mid X_w,X_v]\ge 2n\eta_n\) for all large \(n\), and by
Chernoff's inequality,
\[
\Pr\!\left(N_j\le n\eta_n \,\middle|\, X_w,X_v\right)
\le
\exp(-c\,n\eta_n).
\]
Since
\[
\begin{aligned}
n\eta_n
&= n\,\delta_n^{5d}\xiq^{d+1} 
= n(\log n)^{-5d+C_{\tref{def:xi-q}}(d+1)}
\left(\frac{1}{\sp n}\right)^{\frac{d+1}{d+{\frak q}}} 
\ge
n^{\frac{{\frak q}-1}{d+{\frak q}}}
(\log n)^{-5d+C_{\tref{def:xi-q}}(d+1)}
\gg \log n,
\end{aligned}
\]
where the last inequality holds whenever ${\frak q}>1$. 
Hence, we have \(\exp(-c\,n\eta_n)=n^{-\omega(1)}\). A union bound over all grid points
therefore yields
\[
\Pr\!\left(
\forall j,\ 
|X_{\mathbf C\setminus\{w,v\}}\cap \mathcal S^{\rm aux}(w,v,t_j)|
\ge n\eta_n
\right)
\ge
1-n^{-\omega(1)}.
\]

Now let \(t\in[-\rM,\D{X_w}{X_v}+\rM]\) be arbitrary, and choose \(j\) such that
\[
|t-t_j|\le h_n/4.
\]
Because the slabs are defined in Fermi coordinates, shrinking only the
longitudinal half-width gives the exact inclusion
\[
\mathcal S^{\rm aux}(w,v,t_j)\subseteq \mathcal S(w,v,t).
\]
Therefore
\[
|X_{\mathbf C}\cap \mathcal S(w,v,t)|
\ge
|X_{\mathbf C\setminus\{w,v\}}\cap \mathcal S^{\rm aux}(w,v,t_j)|
\ge
n\eta_n.
\]

Thus the desired bound holds uniformly in \(t\) for this fixed pair \((w,v)\)
and block \(\mathbf C\) with probability at least \(1-n^{-\omega(1)}\). A final
union bound over all ordered pairs \((w,v)\) with \(w\neq v\) and all four
choices of \(\mathbf C\) gives
\[
\Pr(\Evt{locRing})\ge 1-n^{-\omega(1)}.
\]
\end{proof}

\begin{proof}[Proof of Lemma~\ref{lem:moving-ring-mass}]
Let
\[
A:=\Bigl\{(s,y)\in \mathbb R\times\mathbb R^{d-1}:
|s-t|\le h,\ \rho_-\le \|y\|_2\le \rho_+\Bigr\}.
\]
By definition,
\[
\mathcal S_\gamma(t;\rho_-,\rho_+,h)=\Phi_\gamma(A).
\]
Since
\[
2h\le c_{\rm fm}\rG
\qquad\text{and}\qquad
\rho_+\le c_{\rm fm}\rG,
\]
Lemma~\ref{lem:fermi-short-window} shows that \(\Phi_\gamma\) is injective and
\(2\)-bi-Lipschitz on the relevant domain.

The Euclidean volume of \(A\) is
\[
|A|
\asymp
h\Bigl((\rho_+)^{d-1}-(\rho_-)^{d-1}\Bigr).
\]
Choose
\[
\eta:=\frac1{20}\min\{h,\rho_+-\rho_-,\rho_+\}.
\]
Then one can pack \(A\) with \(N\) pairwise disjoint Euclidean balls of radius
\(\eta\), where
\[
N\eta^d\gtrsim |A|.
\]
Indeed, this follows from a standard maximal-packing argument applied to a
slightly shrunken annular box inside \(A\).

Let \(a_1,\dots,a_N\in A\) be the centers of such a packing, and write
\[
x_i:=\Phi_\gamma(a_i)\in \mathcal S_\gamma(t;\rho_-,\rho_+,h).
\]
Since \(\Phi_\gamma\) is \(2\)-bi-Lipschitz, the manifold balls
\[
B_M(x_i,\eta/4)
\]
are pairwise disjoint and contained in \(\mathcal S_\gamma(t;\rho_-,\rho_+,h)\).
Moreover,
\[
\eta\le \rho_+\le c_{\rm fm}\rG\le \rG\le r_\mu,
\]
so lower Ahlfors regularity yields
\[
\mu\bigl(B_M(x_i,\eta/4)\bigr)\gtrsim \eta^d
\qquad\text{uniformly in }i.
\]
Summing over \(i\), we obtain
\[
\mu\!\left(\mathcal S_\gamma(t;\rho_-,\rho_+,h)\right)
\gtrsim
N\eta^d
\gtrsim
|A|
\asymp
h\Bigl((\rho_+)^{d-1}-(\rho_-)^{d-1}\Bigr),
\]
which proves the claim.
\end{proof}

\section{Distance Calibration}
\label{sec:distance-calibration}
With the Ring sets in place, we are ready to estimate the latent distances between pairs of points in $\bfV_1$ that are relatively close in the latent space, i.e., pairs in the target set $\tar$ defined in \eqref{def:target}. For this purpose, we need the event $\Enavi{\ring{u,v}}{w}$ from Lemma~\ref{_lem: navi} to hold for all $(u,v)\in\meso$ and $w\in\bfV_1$. This is the last key event that we will need to assume for the distance calibration argument. We now define the intersection of all these events, which is the good event that gurantees the success of the distance calibration estimate. Let 
\begin{align}
\Evt{ray}
:=
\bigcap_{(u,v)\in\meso}\ \bigcap_{w\in\bfV_1}
\Enavi{\ring{u,v}}{w},
\end{align}
and
\begin{align}
\Evt{all}
:=
\Evt{base}\cap \Evt{locRing}\cap \Evt{R}\cap \Evt{ray}.
\end{align}

\begin{lemma}
\label{lem:Evt-all-prob}
Assume the standing hypotheses and suppose there exists an $(r_0,\lambda_0,p_0)$ cluster generating algorithm. Then
\[
\Pr(\Evt{all})
\ge
1-p_0-n^{-\omega(1)}.
\]
\end{lemma}
This is a standard concentration argument which relies on union bound and conditional independence. 
The proof is deferred to Appendix~\ref{sec:probability-estimates}. It proceeds
by revealing the randomness in stages: first the induced subgraph on
\(\bfV_0\), then the latent locations, and finally the edge variables needed
for the three successive families of navigation events. The order matters,
because the random sets \(B_v\), \({\cal B}_v^{(2)}\), and \(\ring{u,v}\) must
be determined before the corresponding events
\(\Enavi{B_v}{\cdot}\), \(\Enavi{{\cal B}_v^{(2)}}{\cdot}\), and
\(\Enavi{\ring{u,v}}{\cdot}\) are well-defined.

\medskip
\step{Standing assumption for Sections~\ref{sec:distance-calibration}--\ref{sec:recovering-distance-unknown-rmp}}
{\it Throughout this section and Sections~\ref{sec:ring-maximizer-geometry}
and~\ref{sec:recovering-distance-unknown-rmp}, we work on the event
\(\Evt{all}\), unless stated otherwise. In particular, all arguments in these
sections concern the recovery of latent distances for pairs of points in
\(\bfV_1\).}
\medskip

As the first consequence, let us show that we can use $\cn{\ring{u,v}}{v}$ to approximate $\rmp(\D{X_u}{X_v})$ for every $(u,v)\in\meso$ with error of order $\xiq^2/\rG$ up to polylogarithmic factors. This is the first step towards estimating the latent distance $\D{X_v}{X_w}$ for each $(v,w)\in\tar$.

\begin{lemma}[Single-point ring average approximation]
\label{lem:ring-single-avg}
There exists a constant
\(
C_{\tref{lem:ring-single-avg}} := 2L_\rmp C_{\tref{cor:screened-ring-distance-stability}}
\)
depending only on $L_\rmp$ and $\ell_\rmp$ so that the following holds
when $n$ is sufficiently large. 
Assume the event $\Evt{all}$.  For every $(u,v)\in\meso$,
\[
\Big|\cn{\ring{u,v}}{v}-\rmp(\D{X_u}{X_v})\Big|
\le
C_{\tref{lem:ring-single-avg}}\,
\delta_n^{-1}\frac{\xiq^2}{\rG}.
\]
\end{lemma}

\begin{proof}
Fix $(u,v)\in\meso$.
Since $\Evt{all}\subseteq \Evt{ray}$, together with $\fe{\ring{u,v}}
\le
\tfrac{\xiq^2}{\rG}$ from Lemma~\ref{lem:angle-screened-subset-coverage},
we have
\[
\Big|\cn{\ring{u,v}}{v}-\acn{\ring{u,v}}{v}\Big|
\le
\fe{\ring{u,v}}
\le
\frac{\xiq^2}{\rG}\,.
\]
Next, for every $u'\in \ring{u,v}$, Corollary~\ref{cor:screened-ring-distance-stability} gives
\[
\big|\D{X_v}{X_{u'}}-\D{X_u}{X_v}\big|
\le
C_{\tref{cor:screened-ring-distance-stability}}\,
\delta_n^{-1}\frac{\xiq^2}{\rG}.
\]
Since $(u,v)\in\meso$, Corollary~\ref{cor:mesoscopic-pairs-v1} and triangle inequality imply  
\[
 \delta_n^2\rG
 \lesssim \D{X_u}{X_v} \asymp  \D{X_{u'}}{X_v} 
 \lesssim \delta_n\rG = o(\rG)\,.
\]
The $L_\rmp$-Lipschitz property of $\rmp$ on $[0,\rmrp]$ yields
\[
\Big|\acn{\ring{u,v}}{v}-\rmp(\D{X_u}{X_v})\Big|
\le
L_\rmp\,C_{\tref{cor:screened-ring-distance-stability}}\,
\delta_n^{-1}\frac{\xiq^2}{\rG}\,.
\]
Together with the previous bound, the lemma follows with a suitable choice of the constant $C_{\tref{lem:ring-single-avg}}$.
\end{proof}
 
Let
\begin{equation}
\label{def:target}
\tar
:=
\left\{
(v,w)\in \bfV_1\times\bfV_1:\ 
\Delta_{\rm cn}(v,w)\le \delta_n^3\rG
\right\}.
\end{equation}
By Lemma~\ref{lem:delta-cn-proxy}, since $\Evt{all}\subseteq \Evt{base}$, on $\Evt{all}$ and for all sufficiently large $n$,
\[
\Big\{
(v,w)\in\bfV_1\times\bfV_1:\ 
\D{X_v}{X_w}\le C_{\tref{lem:delta-cn-proxy}}^{-1}\delta_n^3\rG
\Big\}
\subseteq
\tar
\subseteq
\Big\{
(v,w)\in\bfV_1\times\bfV_1:\ 
\D{X_v}{X_w}\le C_{\tref{lem:delta-cn-proxy}}\delta_n^3\rG
\Big\}\,.
\]

Our goal in this section is to estimate the latent distance $\D{X_v}{X_w}$ for each
$(v,w)\in\tar$, with error of order $\xiq^2/\rG$ up to polylogarithmic factors.
The strategy is to choose a suitable point $u\in\bfU_1$ with $(u,v)\in\meso$ such
that $X_u,X_v,X_w$ lie close to a common geodesic, with $X_w$ located between
$X_u$ and $X_v$. For such a choice of $u$, we will show that both
\[
\cn{\ring{u,v}}{v}
\qquad\text{and}\qquad
\cn{\ring{u,v}}{w}
\]
approximate
\[
\rmp(\D{X_u}{X_v})
\qquad\text{and}\qquad
\rmp(\D{X_u}{X_w}),
\]
respectively, again with error of order $\xiq^2/\rG$ up to polylogarithmic
factors.  

When $\rmp$ is known and the three points are nearly aligned along a geodesic,
this leads to the approximation
\[
\D{X_v}{X_w}
\approx
\rmp^{-1}\!\big(\cn{\ring{u,v}}{w}\big)
-
\rmp^{-1}\!\big(\cn{\ring{u,v}}{v}\big),
\]
with error of order $\xiq^2/\rG$ up to polylogarithmic factors. The case where $\rmp$ is unknown is more delicate, as we follow on a modified strategy of \cite{FMR25}, which will be discussed in the next section. 

Let us shift the focus to a vertex $v \in \bfV_1$. We first will collect all $u$ such that $(u,v)\in\meso$ and $\cn{\ring{u,v}}{v}$ is close to a fixed anchor value $\rmp_1$. This means that $\D{X_u}{X_v}$ is close to ${\rm s}_1:=\rmp^{-1}(\rmp_1)$, with error of order $\xiq^2/\rG$ up to polylogarithmic factors, as suggested by Lemma~\ref{lem:ring-single-avg}. 
\begin{defi}[Good anchor and candidate set]
\label{def:anchor}
We say that \(\rmp_1 \in \mathbb{R}\) is a \emph{good anchor} if
\begin{align}
    \label{eq:good-anchor-condition}
\rmp_1\in[\rmp(\delta_n^{1.4}\rG),\,\rmp(\delta_n^{1.6}\rG)]\,.
\end{align}
For any good anchor $\rmp_1$, we pair it with  
$$
    {\rm s}_1 := \rmp^{-1}(\rmp_1) \in [\delta_n^{1.6}\rG,\,\delta_n^{1.4}\rG],
$$
and for each \(v\in\bfV_1\), define the \emph{candidate set}
\begin{align}
\label{def:anchor-candidate-set}
\Cs{v} =
\Cs{v}(\rmp_1)
:=
\left\{
u\in\bfU_1:\ 
(u,v)\in\meso,
\quad
|\cn{\ring{u,v}}{v} - \rmp_1| \le  2C_{\tref{lem:ring-single-avg}}\,
\delta_n^{-1}\frac{\xiq^2}{\rG}
\right\}.
\end{align}
\end{defi}

The next lemma shows that such points $u$ are abundant, even under an additional prescribed directional condition that $u$ lies in a ring slab along the geodesic $\gamma_{w,v}$ for any $w$ with $\D{X_w}{X_v}\lesssim \delta_n^{3}\rG$ (satisfied for every $w$ with $(v,w)\in\tar$ on $\Evt{all}$). This is the key technical lemma that allows us to show existence of a good candidate $u$ for each $(v,w)\in\tar$ such that $X_u,X_v,X_w$ are nearly aligned along a geodesic, which implicitly gives the desired distance calibration estimate.

\begin{lemma}[Distance calibration from moving ring slabs]
\label{lem:moving-ring-distance-calibration}
Assume $\Evt{all}$ holds. 
Fix $v\in\bfV_1$ and let $s$ satisfy
\[
\delta_n^{1.6}\rG \le s \le \delta_n^{1.4}\rG.
\]
Fix any $w\in\bfV_1$, for each $u \in \bfU_1$ such that $X_u$ is contained in the Ring slab $\mathcal S(v,w,s)$, we have 
$$
\big|\D{X_u}{X_v} - s \big| \le \delta_n \frac{\xiq^2}{\rG}\,,
\quad\text{and}\quad
\big|\cn{\ring{u,v}}{v} - \rmp(s) \big| \le 2 C_{\tref{lem:ring-single-avg}}\delta_n^{-1} \frac{\xiq^2}{\rG}\,.
$$
This implies that $(u,v) \in\meso$ and $u \in \Cs{v}(\rmp(s))$. Further, from $\Evt{all}\subseteq \Evt{locRing}$, the number of such $u$ is at least $n\delta_n^{5d}\xiq^{d+1} \gg 1$. 
\end{lemma}
\begin{proof}
    Consider the point $y:=\gamma_{v,w}(s)$.  
    Fix any $u\in\bfU_1$ such that $X_u \in \mathcal S(v,w,s)$ and let $\wX_{u}$ be the projection of $X_u$ onto the geodesic $\gamma_{v,w}$. From the geometric properties of $\mathcal S(v,w,s)$ (see Remark~\ref{rem:Svwt-geometry}), we have
    $$
        \D{X_u}{\wX_{u}} \le \tfrac98\delta_n^{1.5}\xiq
        \quad \text{and}\quad  \D{y}{\wX_{u}} \le \delta_n^4\xiq\,.
    $$
    First, the bounds above together with the triangle inequality gives
    $$
        \D{\wX_{u}}{X_v} \asymp \D{y}{X_v} = s \gg \D{X_u}{\wX_{u}}. 
    $$
    Apply Lemma \ref{lem: M-right-angle} to the triangle $(X_u, \wX_{u}, X_v)$, we conclude that  
    $$
        \big|\D{X_u}{X_v} - \D{\wX_{u}}{X_v}\big| \le 4 \frac{\D{X_u}{\wX_{u}}^2}{\D{\wX_{u}}{X_v}}
    \lesssim \frac{\delta_n^3\xiq^2}{s} \lesssim \delta_n^{1.4} \frac{\xiq^2}{\rG}. 
    $$
    Together with $\D{y}{X_v} = s$ and $\D{y}{\wX_{u}} \le \delta_n^4\xiq$, we get 
    $$
        \big|\D{X_u}{X_v} - s\big| \lesssim \delta_n^{1.4} \frac{\xiq^2}{\rG}.
    $$
    The bound on $\big|\cn{\ring{u,v}}{v} - \rmp(s)\big|$ follows from the above bound and the $L_\rmp$-Lipschitz continuity of $\rmp$ on $[0,\rmrp]$, together with Lemma~\ref{lem:ring-single-avg}.

\end{proof}

The main result of this section is 
\begin{lemma}
\label{lem:anchor-ring-dist}
There exists a constant $C_{\tref{lem:anchor-ring-dist}}$ depending only on $L_\rmp$ and $\ell_\rmp$ so that the following holds.
Assume $\Evt{all}$ holds, and let $\rmp_1$ be a good anchor. Then for every $v,w \in\bfV_1$ such that $\D{X_v}{X_w}\le \delta_n^{1.7}\rG$, 
\[
     |\max_{u \in \Cs{v}} \cn{\ring{u,v}}{w} - \rmp({\rm s}_1 - \D{X_v}{X_w})| \le C_{\tref{lem:anchor-ring-dist}} \delta_n^{-1}\frac{\xiq^2}{\rG}.
\]
In particular, this condition includes all pairs \((v,w)\in\tar\).
\end{lemma}

In other words, knowing $\rmp$ then we are able to estimate $\D{X_v}{X_w}$ for every $v,w \in \bfV_1$ with $\D{X_v}{X_w}\lesssim \delta_n^{3} \rG$ up to error $\xiq^2/\rG$ (up to additional polylog factors). Therefore, as a consequence of Lemma \ref{lem:anchor-ring-dist} and the probability estimate in Lemma~\ref{lem:Evt-all-prob}, we have the following local distance recovery result when $\rmp$ is known: 

\begin{theor}[Local distance recovery when $\rmp$ is known]
\label{thm:local-distance-recovery-known-rmp} 
Consider the Riemannian graph model satisfying Assumption~\ref{assump: Riemannian model}. Suppose moreover that there exists an $(r_0,\lambda_0,p_0)$ cluster generating algorithm with $r_0 \le c_{\rm cn}\rG$, where $c_{\rm cn}>0$ is the constant from Lemma~\ref{lem:delta-cn-proxy}, which depends only on $L_\rmp$ and $\ell_\rmp$.

Then there is an algorithm which, given as input the sampled graph $G$ on $5n$ vertices, the function $\rmp$, the parameters $\rG$, $c_\mu$, and $d$, and a prescribed subset $\bfV_1$ of size $n$, outputs with probability at least $1-p_0-n^{-\omega(1)}$ a symmetric set ${\cal P}\subseteq \bfV_1\times \bfV_1$ and a function $\hat{d}:{\cal P}\to \mathbb{R}$ such that
$$
\text{for every } (v,w)\in {\cal P}, \qquad
|\hat{d}(v,w) - \D{X_v}{X_w}| \le C_{\tref{lem:anchor-ring-dist}} \delta_n^{-1}\frac{\xiq^2}{\rG},
$$
and ${\cal P}$ contains every pair $(v,w)\in \bfV_1^2$ such that 
$$
    \D{X_v}{X_w} \le  c \delta_n^{3}\rG,
$$
where $c = C_{\tref{lem:delta-cn-proxy}}^{-1}$ is the constant from Lemma~\ref{lem:delta-cn-proxy}, which depends only on $L_\rmp$ and $\ell_\rmp$.

Moreover, the running time of this algorithm is the running time of the $(r_0,\lambda_0,p_0)$ cluster generating algorithm plus $O(n^3)$.
\end{theor}

\subsection{Proof of Lemma~\ref{lem:anchor-ring-dist}}
\label{subsec:proof-anchor-ring-dist}

\begin{lemma}[Anchor-to-target upper bound through screened rings]
\label{lem:anchor-ring-upper-bound}
There exists a constant
\[
    C_{\tref{lem:anchor-ring-upper-bound}}
    :=
    \max\!\left\{
    \frac{3 C_{\tref{lem:ring-single-avg}}}{\ell_\rmp}
    + 2C_{\tref{lem:delta-cn-proxy}},
    \,
    L_\rmp\!\left(
    \frac{3 C_{\tref{lem:ring-single-avg}}}{\ell_\rmp}
    + 2C_{\tref{lem:delta-cn-proxy}}
    \right)+1
    \right\},
\]
depending only on the fixed model parameters, such that the following holds
when $n$ is sufficiently large. 
Assume $\Evt{all}$ holds and fix a good anchor $\rmp_1$ and ${\rm s}_1=\rmp^{-1}(\rmp_1)$.
Fix $v,w\in\bfV_1$, let $u\in\Cs{v}$, and let $u'\in\ring{u,v}$. Then, 
\[
{\rm s}_1-\D{X_v}{X_w}
\,-\,
C_{\tref{lem:anchor-ring-upper-bound}}\,
\delta_n^{-1}\frac{\xiq^2}{\rG}
\le \D{X_{u'}}{X_w}
\le
{\rm s}_1+\D{X_v}{X_w}
\,+\,
C_{\tref{lem:anchor-ring-upper-bound}}\,
\delta_n^{-1}\frac{\xiq^2}{\rG}.
\]
Moreover, if
\(
\D{X_v}{X_w}\le \delta_n^{1.7}\rG,
\)
then
\[
\rmp\!\left({\rm s}_1-\D{X_v}{X_w}\right)
\,+\,
C_{\tref{lem:anchor-ring-upper-bound}}\,
\delta_n^{-1}\frac{\xiq^2}{\rG}
\ge 
\cn{\ring{u,v}}{w}
\ge
\rmp\!\left({\rm s}_1+\D{X_v}{X_w}\right)
\,-\,
C_{\tref{lem:anchor-ring-upper-bound}}\,
\delta_n^{-1}\frac{\xiq^2}{\rG}.
\]
In particular, this condition includes all pairs \((v,w)\in\tar\), since on
\(\Evt{all}\), Lemma~\ref{lem:delta-cn-proxy} gives
\(\D{X_v}{X_w}\lesssim \delta_n^3\rG \ll \delta_n^{1.7}\rG\).
\end{lemma}
\begin{proof}
\step{1. Distance bounds}
From $u\in\Cs{v}$ and Lemma~\ref{lem:ring-single-avg},
\[
    \Big|\rmp(\D{X_u}{X_v})-\rmp_1\Big|
\le 
    \Big|\rmp(\D{X_u}{X_v})-\cn{\ring{u,v}}{v}\Big|
+   \Big|\cn{\ring{u,v}}{v}-\rmp_1\Big|
\le 
    3 C_{\tref{lem:ring-single-avg}}\,\delta_n^{-1}\frac{\xiq^2}{\rG}.
\]
Since $u\in\Cs{v}\subseteq\meso$, Corollary~\ref{cor:mesoscopic-pairs-v1} gives
\(
\D{X_u}{X_v}\lesssim \delta_n\rG=o(\rG)\le \rmrp
\)
for sufficiently large $n$, so $\rmp(\D{X_u}{X_v})\in\rmp([0,\rmrp])$.
Applying the lower Lipschitz bound of $\rmp$ on $[0,\rmrp]$ gives 
\[
\Big|\D{X_u}{X_v}-{\rm s}_1\Big|
\le
    \frac{3 C_{\tref{lem:ring-single-avg}}}{\ell_\rmp}\,\delta_n^{-1}\frac{\xiq^2}{\rG}.
\]
Also, $u'\in\ring{u,v}$ and Corollary~\ref{cor:screened-ring-distance-stability} give
\[
|\D{X_{u'}}{X_v}
- \D{X_u}{X_v}|
\le
2C_{\tref{lem:delta-cn-proxy}}\,
\delta_n^{-1}\frac{\xiq^2}{\rG}.
\]
Therefore,
\[
\Big|\D{X_{u'}}{X_v}-{\rm s}_1\Big|
\le
\bigg(
\underbrace{\frac{3 C_{\tref{lem:ring-single-avg}}}{\ell_\rmp}
    + 2C_{\tref{lem:delta-cn-proxy}}}_{=: \, C}
\bigg)\delta_n^{-1}\frac{\xiq^2}{\rG},
\]
and the first statement of the lemma follows by the triangle inequality. 

\step{2. Bounding $\cn{\ring{u,v}}{w}$ under a direct distance assumption}
Since $u\in\Cs{v}\subseteq\meso$, we have $\D{X_u}{X_v}=o(\rG)$, and hence the
first statement implies $\D{X_{u'}}{X_w}=o(\rG)$ for every $u'\in\ring{u,v}$.
Also, since
\[
\D{X_v}{X_w}\le \delta_n^{1.7}\rG
\ll
{\rm s}_1\in[\delta_n^{1.6}\rG,\delta_n^{1.4}\rG],
\]
we have \({\rm s}_1-\D{X_v}{X_w}>0\) for all sufficiently large \(n\). Therefore
all arguments of \(\rmp\) below lie in \([0,\rmrp]\) for sufficiently large
\(n\). Applying the first statement together with the upper Lipschitz bound
of $\rmp$ on $[0,\rmrp]$, and then averaging over $u'\in\ring{u,v}$, gives
\begin{align}
    \label{eq:anchor-ring-trivial-bound-1}
\rmp\!\left({\rm s}_1-\D{X_v}{X_w}\right) + 
L_\rmp C\,
\delta_n^{-1}\frac{\xiq^2}{\rG}
\ge 
\acn{\ring{u,v}}{w}
\ge
\rmp\!\left({\rm s}_1+\D{X_v}{X_w}
\right)\,-\,
L_\rmp C\,
\delta_n^{-1}\frac{\xiq^2}{\rG}\,.
\end{align}
Also, by $\Evt{ray}$ and $\fe{\ring{u,v}}
\le
\tfrac{\xiq^2}{\rG}$ from Lemma~\ref{lem:angle-screened-subset-coverage},
\[
\Big|\cn{\ring{u,v}}{w}-\acn{\ring{u,v}}{w}\Big|
\le
\fe{\ring{u,v}}
\le 
\frac{\xiq^2}{\rG}\,.
\]
Therefore, the second statement follows from applying the above bound to
\eqref{eq:anchor-ring-trivial-bound-1} and the fluctuation bound above.
\end{proof}

\begin{lemma}[Projection stability from projection distance]
\label{lem:projection-stability-projdist}
Assume $\Evt{all}$ holds and $\rmp_1$ is a good anchor.
Fix $v,w\in\bfV_1$ such that
\(
\D{X_v}{X_w}\le \delta_n^{1.7}\rG,
\)
and let $u\in\Cs{v}$.
Let $\wX_{u}$ be the nearest-point projection of $X_u$ onto $\gamma_{w,v}$, and set
\[
\eta:=\D{X_u}{\wX_{u}}\le \delta_n^{1.7}\rG.
\]
Then, for every $u'\in\ring{u,v}$, letting $\wX_{u'}$ denote the nearest-point
projection of $X_{u'}$ onto $\gamma_{w,v}$, we have
\begin{align*}
\D{X_{u'}}{\wX_{u'}}
\le \eta + C_{\tref{lem:delta-cn-proxy}}\,\delta_n\xiq,\quad \mbox{and} \quad
\D{\wX_{u}}{\wX_{u'}}
\le
\delta_n^{-1.5}\frac{\eta^2}{\rG}
+ 3C_{\tref{lem:delta-cn-proxy}}\,\delta_n^{-1}\frac{\xiq^2}{\rG}.
\end{align*}
In particular, this condition includes all pairs \((v,w)\in\tar\), since on
\(\Evt{all}\), Lemma~\ref{lem:delta-cn-proxy} gives
\(\D{X_v}{X_w}\lesssim \delta_n^3\rG \ll \delta_n^{1.7}\rG\).
\end{lemma}
\begin{rem}
    We remark that the assumption $\D{X_u}{\wX_{u}}\le \delta_n^{1.7}\rG$ ensures that $\D{X_u}{\wX_u} \ll {\rm s}_1$.  
\end{rem}

\begin{proof}
\step{1. Projection distances}
Fix $u\in\Cs{v}$ and $u' \in\ring{u,v}$. Since $\rmp_1$ is a good anchor,
Corollary~\ref{cor:mesoscopic-pairs-v1} and
Corollary~\ref{cor:screened-ring-distance-stability} give
\[
\D{X_u}{X_{u'}} \le C_{\tref{lem:delta-cn-proxy}} \delta_n \xiq
\ll \delta_n^{1.6}\rG
\lesssim
  {\rm s}_1 \lesssim \delta_n^{1.4}\rG
\]
and therefore Corollary~\ref{cor:local-proj-geodesic} gives unique projections
$\wX_{u}$ and $\wX_{u'}$ of $X_u$ and $X_{u'}$ onto $\gamma_{w,v}$.
By nearest-point projection,
\[
\D{X_{u'}}{\wX_{u'}}
\le
\D{X_{u'}}{\wX_{u}}
\le
\D{X_{u'}}{X_u} + \D{X_u}{\wX_{u}}
\le
\eta + C_{\tref{lem:delta-cn-proxy}}\,\delta_n\xiq.
\]
Since $u\in\Cs{v}$ and $u'\in\ring{u,v}$, we also have
\[
\D{X_u}{X_v}\asymp {\rm s}_1,
\qquad
\D{X_{u'}}{X_v}\asymp {\rm s}_1,
\]
and hence
\[
\D{\wX_{u}}{X_v}\asymp {\rm s}_1,
\qquad
\D{\wX_{u'}}{X_v}\asymp {\rm s}_1.
\]

\step{2. Distance between the projections}
By Lemma \ref{lem: M-right-angle} applied to the right triangle $(X_u, \wX_{u}, X_v)$, we have
\begin{align*}
\big|\D{X_v}{\wX_{u}}-\D{X_v}{X_u}\big| &\le 4\frac{\D{X_u}{\wX_{u}}^2}{\D{\wX_{u}}{X_v}}
\lesssim
\frac{\eta^2}{{\rm s}_1}
\lesssim
\frac{\eta^2}{ \delta_n^{1.6}\rG }
\end{align*}
Similarly, 
\begin{align*}
    \big|\D{X_v}{\wX_{u'}}-\D{X_v}{X_{u'}}\big| &\lesssim  
\frac{(\eta + C_{\tref{lem:delta-cn-proxy}} \delta_n \xiq)^2}{{\rm s}_1}
\lesssim 
\frac{(\eta + C_{\tref{lem:delta-cn-proxy}} \delta_n \xiq)^2}{ \delta_n^{1.6}\rG }
\asymp
\delta_n^{-1.6}
\frac{\eta^2}{ \rG }
+ 
\frac{\delta_n^{0.4} \xiq^2}{\rG}.
\end{align*}
Since $\wX_{u}$ and $\wX_{u'}$ lie on the same geodesic segment $\gamma_{w,v}$,
\begin{align*}
\D{\wX_{u}}{\wX_{u'}}
=
\big|\D{X_v}{\wX_{u'}}-\D{X_v}{\wX_{u}}\big|
&\le
\big|\D{X_v}{\wX_{u'}}-\D{X_v}{X_{u'}}\big| + \big|\D{X_v}{X_{u'}}-\D{X_v}{X_u}\big| + \big|\D{X_v}{X_u}-\D{X_v}{\wX_{u}}\big|\,.
\end{align*}
The intermediate term above is controlled by
Corollary~\ref{cor:screened-ring-distance-stability}:
\[
\big|\D{X_v}{X_{u'}}-\D{X_v}{X_u}\big|
\le
2C_{\tref{lem:delta-cn-proxy}}\,
\delta_n^{-1}\frac{\xiq^2}{\rG}.
\]
Therefore, with $\delta_n^{0.4}\xiq^2/\rG \ll
\delta_n^{-1}\xiq^2/\rG$, we get
\[
\D{\wX_{u}}{\wX_{u'}}
\le
\delta_n^{-1.5}\frac{\eta^2}{\rG}
+ 3C_{\tref{lem:delta-cn-proxy}}\,\delta_n^{-1}\frac{\xiq^2}{\rG},
\]
where we changed $\delta_n^{-1.6}$ to $\delta_n^{-1.5}$ to absorb the
implicit constant in front of $\eta^2/\rG$.
\end{proof}

\begin{proof}[Proof of Lemma~\ref{lem:anchor-ring-dist}]

\step{Good candidate $u \in \Cs{v}$}
Let $u \in \bfU_1$ be any point such that $X_u$ is contained in the ring slab $\mathcal S(v,w,{\rm s}_1)$, which exists by $\Evt{locRing}$ and the definition of good anchor.
By Lemma~\ref{lem:moving-ring-distance-calibration}, we have $u \in \Cs{v}$.  Further, Remark~\ref{rem:Svwt-geometry} gives
\[
\eta := \D{X_u}{\wX_u}\le \tfrac98\delta_n^{1.5}\xiq \ll \delta_n^{1.7}\rG,
\]
which allows us to apply Lemma~\ref{lem:projection-stability-projdist} to $u$ and conclude that  
\begin{align}
    \label{eq:projection-stability-application}
\D{X_{u'}}{\wX_{u'}}
\le 2C_{\tref{lem:delta-cn-proxy}}\,\delta_n\xiq
\quad \mbox{and} \quad
\D{\wX_{u}}{\wX_{u'}}
\le 4C_{\tref{lem:delta-cn-proxy}}\,\delta_n^{-1}\frac{\xiq^2}{\rG}.
\end{align}

Further, since
\[
\D{X_v}{X_w}\le \delta_n^{1.7}\rG \ll {\rm s}_1,
\]
it follows that  
$X_w$ lies between $\wX_u$ and $X_v$ along the geodesic $\gamma_{v,w}$, and hence
$$
 \D{\wX_u}{X_w} = \D{\wX_u}{X_v} - \D{X_w}{X_v} \ge \frac{{\rm s}_1}{2}.
$$
For a precise estimate, we invoke Lemma \ref{lem:moving-ring-distance-calibration} to get
$$
 |\D{\wX_u}{X_w} - ({\rm s}_1 - \D{X_v}{X_w})| \le \delta_n \frac{\xiq^2}{\rG}.
$$

Next, consider $u' \in\ring{u,v}$. Applying  Lemma~\ref{lem: M-right-angle} to the triangle $(X_{u'}, \wX_{u'}, X_w)$ gives
\begin{align*}
\big|\D{X_w}{\wX_{u'}}-\D{X_w}{X_{u'}}\big| &\le 4\frac{\D{X_{u'}}{\wX_{u'}}^2}{\D{\wX_{u'}}{X_w}}
\lesssim
\frac{\delta_n^2\xiq^2}{\delta_n^{1.6}\rG} \ll \delta_n^{0.3}\frac{\xiq^2}{\rG},
\end{align*}
we we used \eqref{eq:projection-stability-application} and ${\rm s}_1 \ge \delta_n^{1.6}\rG$.
Combining the above estimates, we have 
\begin{align*}
\Big|\D{X_w}{X_{u'}} - \big(
{\rm s}_1 - \D{X_v}{X_w}\big)\Big|
\le&  \Big|\D{X_w}{\wX_{u'}} - \D{X_w}{\wX_{u'}}\Big| + \Big|\D{X_w}{\wX_{u'}} - \D{X_w}{\wX_{u}}\Big| + \Big|\D{X_w}{\wX_{u}} - \big(
{\rm s}_1 - \D{X_v}{X_w}\big)\Big| \\
\le&
\delta_n^{0.3}\frac{\xiq^2}{\rG} + 4C_{\tref{lem:delta-cn-proxy}}\,\delta_n^{-1}\frac{\xiq^2}{\rG} + \delta_n \frac{\xiq^2}{\rG} \\
\le&
5C_{\tref{lem:delta-cn-proxy}}\,\delta_n^{-1}\frac{\xiq^2}{\rG}\,,
\end{align*}
where we used $\Big|\D{X_w}{\wX_{u'}} - \D{X_w}{\wX_{u}}\Big| \le \D{\wX_{u}}{\wX_{u'}}$ and the bound on $\D{\wX_{u}}{\wX_{u'}}$ from \eqref{eq:projection-stability-application}.

As it is true for every $u' \in \ring{u,v}$, and since
\[
{\rm s}_1-\D{X_v}{X_w}>0,
\]
all arguments of \(\rmp\) below lie in \([0,\rmrp]\) for sufficiently large
\(n\). Therefore, together with the upper Lipschitz bound of $\rmp$ on
$[0,\rmrp]$, $\Enavi{\ring{u,v}}{w}$ from $\Evt{ray}$, and
$\fe{\ring{u,v}} \le \tfrac{\xiq^2}{\rG}$ from
Lemma~\ref{lem:angle-screened-subset-coverage}, we conclude that  
\begin{align*}
    \cn{\ring{u,v}}{w} 
    \ge& 
    \acn{\ring{u,v}}{w} - \fe{\ring{u,v}} \\
    \ge&
    \rmp\!\left({\rm s}_1-\D{X_v}{X_w}\right)
    - L_\rmp\cdot 5C_{\tref{lem:delta-cn-proxy}}\,\delta_n^{-1}\frac{\xiq^2}{\rG}
    - \frac{\xiq^2}{\rG} \,.
\end{align*}
Combining with Lemma \ref{lem:anchor-ring-upper-bound} gives
$$
    \max_{\tilde u \in \Cs{v}} \cn{\ring{\tilde u,v}}{w}  
    \le 
\rmp\!\left({\rm s}_1-\D{X_v}{X_w}\right)
\,+\,
C_{\tref{lem:anchor-ring-upper-bound}}\,
\delta_n^{-1}\frac{\xiq^2}{\rG}\,.
$$
Given that $ L_\rmp\cdot 5C_{\tref{lem:delta-cn-proxy}}+1$ and $C_{\tref{lem:anchor-ring-upper-bound}}$ are both constants depending only on $L_\rmp$ and $\ell_\rmp$, the lemma follows by adjusting the constant in front of $\delta_n^{-1}\xiq^2/\rG$.
\end{proof}

\section{Extending distance estimate to all pairs}
\label{sec:all-pairs-extension}

\begin{lemma}
\label{lem:subset-local-to-global-local}

Let \(0<\lambda\le 1\) and let \(\rho_0,\varepsilon>0\). Let \((\mathcal X,\rho)\) be a finite metric space, write \(n:=|\mathcal X|\), and let \(\mathcal F(\mathcal X,\rho)\) be any observable of \((\mathcal X,\rho)\). Assume that \(\lambda n\ge 4\).
Suppose there exists an algorithm \(A\) and a number \(p_0\in[0,1]\) with the following property: for every fixed subset
\[
\mathcal X_1\subseteq \mathcal X,
\qquad
|\mathcal X_1|\le \lambda n,
\]
when \(A\) is given \(\mathcal F(\mathcal X,\rho)\) and \(\mathcal X_1\) as input, it outputs, with probability at least \(1-p_0\), a symmetric set
\(
\mathcal P\subseteq \mathcal X_1\times \mathcal X_1
\)
and a function
\(
\hat d:\mathcal P\to \mathbb R
\)
such that
\[
|\hat d(x,y)-\rho(x,y)|\le \varepsilon
\qquad\text{for every }(x,y)\in \mathcal P,
\]
and
\[
\rho(x,y)\le \rho_0,\ \ x,y\in \mathcal X_1
\qquad\Longrightarrow\qquad
(x,y)\in \mathcal P.
\]

Here the probability is taken over the randomness of the input \((\mathcal X,\rho)\) under its underlying distribution, and also over the internal randomness of \(A\), if any.

Then there exists an algorithm \(B\) which takes \(\mathcal F(\mathcal X,\rho)\) as input, with no prescribed subset, and outputs, with probability at least
\(
1- 2\binom{\lceil 2/\lambda\rceil}{2}p_0
\)
a symmetric set
\(
\mathcal P\subseteq \mathcal X\times \mathcal X
\)
and a function
\(
\hat d:\mathcal P\to \mathbb R
\)
such that
\[
|\hat d(x,y)-\rho(x,y)|\le \varepsilon
\qquad\text{for every }(x,y)\in \mathcal P,
\]
and
\[
\rho(x,y)\le \rho_0
\qquad\Longrightarrow\qquad
(x,y)\in \mathcal P.
\]
\end{lemma}

The next lemma isolates the purely geometric shortest-path argument behind the
all-pairs extension.

\begin{lemma}[Metric extension from local estimates]
\label{lem:metric-extension-local-estimates}
Let \((\mathcal X,\rho)\) be a finite metric space with
\(0<r\le {\rm diam}(\mathcal X)\). Let \(C>0\) and \(\eta,\varepsilon\ge 0\),
and let
\(\mathcal P\subseteq \mathcal X\times\mathcal X\) be symmetric. Assume the
following.
\begin{itemize}
\item For every \((x,y)\in\mathcal P\), we are given an estimate
\(\widehat\rho(x,y)\) such that
\[
\big|\widehat\rho(x,y)-\rho(x,y)\big|\le \varepsilon.
\]
\item Whenever \(\rho(x,y)\le r\), we have \((x,y)\in\mathcal P\).
\item For every \(x,y\in\mathcal X\) with \(\rho(x,y)>r\), there exists a chain
\[
p_0=x,\ p_1,\ \dots,\ p_k=y
\]
such that
\[
\rho(p_i,p_{i+1})\le r
\qquad\text{for all } i,
\]
\[
k\le C\frac{\rho(x,y)}{r},
\]
and
\[
\left|
\sum_{i=0}^{k-1}\rho(p_i,p_{i+1})-\rho(x,y)
\right|
\le
k\eta.
\]
\end{itemize}
Define the weighted graph \(G_{\mathcal P}\) on \(\mathcal X\) by placing an edge
between \(x\) and \(y\) whenever \((x,y)\in\mathcal P\), with edge weight
\[
\widehat\rho(x,y)+\varepsilon.
\]
Let \(\rho_{\rm sp}\) be the shortest-path metric on \(G_{\mathcal P}\). Then,
for every \(x,y\in\mathcal X\),
\[
\rho(x,y)
\le
\rho_{\rm sp}(x,y)
\le
\rho(x,y)+C'\frac{{\rm diam}(\mathcal X)}{r}\,(\eta+\varepsilon),
\]
where one may take \(C'=2+3C\). In particular,
\[
\big|\rho_{\rm sp}(x,y)-\rho(x,y)\big|
\le
C'\frac{{\rm diam}(\mathcal X)}{r}\,(\eta+\varepsilon).
\]
\end{lemma}

\begin{lemma}
    \label{lem:geodesic-chain}
    There exists a constant $C_{\tref{lem:geodesic-chain}}$ depending only on $L_\rmp$ and $\ell_\rmp$ such that the following holds.
    Given $\Evt{locRing}$, for every \(v,w \in \bfV\)
    with 
    $$
\D{X_v}{X_w} > \tfrac14 C_{\tref{lem:delta-cn-proxy}}^{-1}\delta_n^3\rG,
    $$
    there are vertices \(v=v_0,v_1,\dots,v_k=w\) such that
    \begin{align*}
0 \le \big|\sum_{i=0}^{k-1}\D{X_{v_i}}{X_{v_{i+1}}}-\D{X_v}{X_w}\big| \le 
Ck \frac{\xiq^2}{\rG}\,,
    \end{align*}
    where $k$ satisfies 
    \begin{align*}
        \frac{k}{4} C_{\tref{lem:delta-cn-proxy}}^{-1}\delta_n^3\rG
    \le \D{X_v}{X_w} \le\frac{k}{2} C_{\tref{lem:delta-cn-proxy}}^{-1}\delta_n^3\rG\,.
    \end{align*}
\end{lemma}
\begin{proof}
    Consider the prefixed geodesic $\gamma_{v,w}$. Let  
    $$
        \eta := \tfrac12 C_{\tref{lem:delta-cn-proxy}}^{-1}\delta_n^3\rG
    $$

    Partition \(\gamma\) by points
\[
p_0=X_v,\ p_1,\ \dots,\ p_k=X_w
\]
with $p_i= \gamma(\lambda \eta i)$ for $i=0,1,\dots,k$, where $\lambda$ is a constant in $[1/2,1]$.  This is possible, by first choose $k$ so that $\frac{\eta}{2}k < \D{X_v}{X_w} \le k\eta$. This implies from continuity a constant $\lambda \in [1/2,1]$ such that 
$\lambda \eta k = \D{X_v}{X_w}$. Then, we simply choose $p_i = \gamma(\lambda \eta i)$ for $i=0,1,\dots,k$.

Now, for each $i=1,2,\dots,k-1$, consider the ring slab $\mathcal S(w,v,t_i)$, where $t_i = \lambda \eta i$. Let $v_i \in \bfV$ be a vertex such that  
$$
    X_{v_i} \in \mathcal S(w,v,t_i).
$$
The event $\Evt{locRing}$ guarantees the existence of such a vertex $v_i$ for each $i$. 
By Remark \ref{rem:Svwt-geometry}, the projection of $X_{v_i}$ to $\gamma$, $\wX_{v_i}$ satisfies
\begin{align*}
\D{X_{v_i}}{\wX_{v_i}} \le \tfrac98 \delta_n^{1.5}\xiq \quad \mbox{and} \quad 
\D{\wX_{v_i}}{p_i} \le \delta_n^4 \xiq^2\,. 
\end{align*}
If we extend the above definition with $v_0 =v$ and $v_k = w$, the above definition is again satisfied for $i=0$ and $i=k$.  Now, we invoke Lemma \ref{lem:two-orthogonal-perturbations} to obtain 
\begin{align*}
\big|\D{X_{v_i}}{X_{v_{i+1}}}-\D{p_i}{p_{i+1}}\big|
\le &
\big|\D{X_{v_i}}{X_{v_{i+1}}}-\D{\wX_{v_i}}{\wX_{v_{i+1}}}\big|
+ \big|\D{\wX_{v_i}}{\wX_{v_{i+1}}}-\D{p_i}{p_{i+1}}\big|\\
\lesssim &
\frac{\delta_n^3\xiq^2}{\eta} + \delta_n^4\xiq^2
\simeq \frac{\xiq^2}{\rG}\,.
\end{align*}
Therefore, 
$$
0 \le \big|\sum_{i=0}^{k-1}\D{X_{v_i}}{X_{v_{i+1}}}-\D{X_v}{X_w}\big| \le 
Ck \frac{\xiq^2}{\rG}\,,
$$
for some constant $C \ge 1$ depending only on $L_\rmp$ and $\ell_\rmp$. 
\end{proof}

\begin{proof}[Proof of Theorem~\ref{thm:bootstrap}]
Here we apply the local recovery result, Theorem~\ref{thm:local-distance-recovery-known-rmp}, as a black box, and then extend the resulting local estimates to all pairs using Lemma~\ref{lem:subset-local-to-global-local} and Lemma~\ref{lem:metric-extension-local-estimates}. In the unknown \(\rmp\) case, the local recovery result is Theorem~\ref{thm:reconstruct-q} in the appendix, but the same argument applies with only notational changes.

Let \(N:=|\bfV|\). We first construct a local recovery routine on arbitrary
prescribed subsets of \(\bfV\) of size at most \(N/5\).

Fix \(S\subseteq \bfV\) with \(|S|\le N/5\), and set
\[
m:=\max\!\left\{|S|,\left\lfloor \frac{N}{10}\right\rfloor\right\}.
\]
Then \(m\le N/5\), so we may choose a superset \(S'\supseteq S\) with
\(|S'|=m\), and also a disjoint set \(Y\subseteq \bfV\setminus S'\) with
\(|Y|=4m\). Let
\[
W:=S'\sqcup Y.
\]
Conditional on the latent sample, the induced graph \(G[W]\) is again a sampled
graph from the same model, now on \(5m\) vertices. Hence
Theorem~\ref{thm:local-distance-recovery-known-rmp} applies to \(G[W]\) with
prescribed subset \(S'\). It yields, with probability at least
\[
1-p_0-m^{-\omega(1)},
\]
a symmetric set \(\mathcal P_W\subseteq S'\times S'\) and a function
\(\hat d_W:\mathcal P_W\to \mathbb R\) such that
\[
\big|\hat d_W(v,w)-\D{X_v}{X_w}\big|
\le
C_{\tref{lem:anchor-ring-dist}}\,
\delta_m^{-1}\frac{\xiq^2}{\rG}
\qquad\text{for all }(v,w)\in\mathcal P_W,
\]
and \(\mathcal P_W\) contains every pair \((v,w)\in S'\times S'\) with
\[
\D{X_v}{X_w}
\le
C_{\tref{lem:delta-cn-proxy}}^{-1}\delta_m^3\rG.
\]
Restricting this output to \(S\times S\) gives an algorithm \(A\) in the sense
of Lemma~\ref{lem:subset-local-to-global-local}, with
\[
\lambda=\frac15,
\qquad
\rho_0:=C_{\tref{lem:delta-cn-proxy}}^{-1}\delta_N^3\rG,
\qquad
\varepsilon:=
C_{\tref{lem:anchor-ring-dist}}\,
\delta_N^{-1}\frac{\xiq^2}{\rG},
\]
after increasing the constant \(C_{\tref{lem:anchor-ring-dist}}\) by an
absolute factor, since \(m\in[N/10,N/5]\) and therefore
\(\delta_m\asymp \delta_N\), the corresponding \(\xiq\)-quantities at scales
\(m\) and \(N\) are comparable, and \(\delta_m\ge \delta_N\).

Applying Lemma~\ref{lem:subset-local-to-global-local}, we obtain an algorithm
\(B_{\rm loc}\) which takes the full graph \(G\) as input and returns, with
probability at least
\[
1-2\binom{10}{2}\bigl(p_0+N^{-\omega(1)}\bigr)
\ge
1-90p_0-N^{-\omega(1)},
\]
a symmetric set \(\mathcal P_{\rm loc}\subseteq \bfV\times\bfV\) and a function
\[
\hat d_{\rm loc}:\mathcal P_{\rm loc}\to \mathbb R
\]
such that
\[
\big|\hat d_{\rm loc}(v,w)-\D{X_v}{X_w}\big|
\le
\varepsilon
\qquad\text{for every }(v,w)\in\mathcal P_{\rm loc},
\]
and every pair \((v,w)\in\bfV\times\bfV\) with
\[
\D{X_v}{X_w}\le \rho_0
\]
belongs to \(\mathcal P_{\rm loc}\).

Now set
\[
r:=\frac12\,C_{\tref{lem:delta-cn-proxy}}^{-1}\delta_N^3\rG.
\]
On the event \(\Evt{locRing}\), whose probability is at least
\(1-N^{-\omega(1)}\) by Lemma~\ref{lem:U1-moving-ring-occupancy}, the chain
lemma, Lemma~\ref{lem:geodesic-chain}, implies that every pair
\(v,w\in\bfV\) with \(\D{X_v}{X_w}>r\) admits a chain
\[
v=v_0,v_1,\dots,v_k=w
\]
such that
\[
\D{X_{v_i}}{X_{v_{i+1}}}\le r
\qquad\text{for all }i,
\qquad
k\le 2\frac{\D{X_v}{X_w}}{r},
\]
and
\[
\left|
\sum_{i=0}^{k-1}\D{X_{v_i}}{X_{v_{i+1}}}-\D{X_v}{X_w}
\right|
\le
k\,C_{\tref{lem:geodesic-chain}}\frac{\xiq^2}{\rG}.
\]
Thus, on the event
\[
E:=\{\text{\(B_{\rm loc}\) succeeds}\}\cap \Evt{locRing},
\]
all assumptions of Lemma~\ref{lem:metric-extension-local-estimates} are
satisfied for the metric space \((\bfV,\rho)\), where
\[
\rho(v,w):=\D{X_v}{X_w},
\qquad
\mathcal P:=\mathcal P_{\rm loc},
\qquad
C=2,
\qquad
\eta:=C_{\tref{lem:geodesic-chain}}\frac{\xiq^2}{\rG}.
\]
Let \(\hat d\) be the shortest-path metric on the weighted graph on \(\bfV\)
with edge set \(\mathcal P_{\rm loc}\) and edge weights
\[
\hat d_{\rm loc}(v,w)+\varepsilon.
\]
Since \({\rm diam}(\bfV,\rho)\le {\rm diam}(M)\),
Lemma~\ref{lem:metric-extension-local-estimates} gives, for every
\(v,w\in\bfV\),
\[
\big|\hat d(v,w)-\D{X_v}{X_w}\big|
\le
8\frac{{\rm diam}(M)}{r}
\left(
C_{\tref{lem:geodesic-chain}}\frac{\xiq^2}{\rG}
+C_{\tref{lem:anchor-ring-dist}}\delta_N^{-1}\frac{\xiq^2}{\rG}
\right).
\]
Since \(r\asymp \delta_N^3\rG\), this is of order
\[
\frac{{\rm diam}(M)}{\delta_N^3\rG}\,
\delta_N^{-1}\frac{\xiq^2}{\rG} \le \frac{{\rm diam}(M)}{\rG} \delta_N^{C_{\tref{def:xi-q}}-2} (\sp n)^{-2/(d+5)}
\le  \delta_N^{C} (\sp n)^{-2/(d+5)}
\]
for some universal constant \(C\), when $n$ is large enough. 
Finally,
\[
\Pr(E)\ge 1-90p_0-N^{-\omega(1)}.
\]

For the unknown-\(\rmp\) case, the same reduction applies formally after
replacing Theorem~\ref{thm:local-distance-recovery-known-rmp} by
Theorem~\ref{thm:reconstruct-q}. Namely, for each padded set
\(
W=S'\sqcup Y
\)
as above, Theorem~\ref{thm:reconstruct-q} yields a symmetric set
\(
\mathcal P_W\subseteq S'\times S'
\)
and a function
\(
\hat d_W:\mathcal P_W\to\mathbb R
\)
such that
\[
\big|\hat d_W(v,w)-R\D{X_v}{X_w}\big|
\le
C\,\delta_m^{-7}\frac{\xiq^2}{\rG}
\qquad\text{for all }(v,w)\in\mathcal P_W,
\]
and \(\mathcal P_W\) contains every pair \((v,w)\in S'\times S'\) with
\[
\D{X_v}{X_w}
\le
C_{\tref{lem:delta-cn-proxy}}^{-1}\delta_m^3\rG,
\]
where \(R\) is the dilation factor from Theorem~\ref{thm:reconstruct-q}.
Thus, applying Lemma~\ref{lem:subset-local-to-global-local} to the scaled
metric
\[
\rho_R(v,w):=R\D{X_v}{X_w},
\]
we obtain, with the same probability bookkeeping as above, a global local
estimate \(\hat d_{\rm loc}\) on a symmetric set
\(\mathcal P_{\rm loc}\subseteq\bfV\times\bfV\) satisfying
\[
\big|\hat d_{\rm loc}(v,w)-\rho_R(v,w)\big|
\le
C\,\delta_N^{-7}\frac{\xiq^2}{\rG}
\qquad\text{for every }(v,w)\in\mathcal P_{\rm loc},
\]
and containing every pair with
\[
\rho_R(v,w)\le
\frac{R}{2}C_{\tref{lem:delta-cn-proxy}}^{-1}\delta_N^3\rG.
\]
Now apply Lemma~\ref{lem:metric-extension-local-estimates} to the metric
\((\bfV,\rho_R)\), with
\[
r:=\frac{R}{2}C_{\tref{lem:delta-cn-proxy}}^{-1}\delta_N^3\rG,
\qquad
\eta:=RC_{\tref{lem:geodesic-chain}}\frac{\xiq^2}{\rG}.
\]
Since
\(
{\rm diam}(\bfV,\rho_R)\le R\,{\rm diam}(M)\lesssim {\rm diam}(M)
\)
and \(R\asymp 1\), this yields
\[
\big|\hat d(v,w)-R\D{X_v}{X_w}\big|
\lesssim
\frac{{\rm diam}(M)}{\delta_N^3\rG}
\left(
\delta_N^{-7}\frac{\xiq^2}{\rG}
+
\frac{\xiq^2}{\rG}
\right)
\lesssim
\frac{{\rm diam}(M)}{\delta_N^3\rG}\,
\delta_N^{-7}\frac{\xiq^2}{\rG},
\]
for all \(v,w\in\bfV\). The same covering argument gives success probability
\(
1-90p_0-N^{-\omega(1)}
\).
For the running time, each invocation of the local recovery routine costs the
running time of the chosen cluster generating algorithm plus \(O(m^3)\), with
\(m\asymp N\). Lemma~\ref{lem:subset-local-to-global-local} uses only a constant
number of such calls, and the final metric extension is obtained by computing a
shortest-path metric on a graph with \(N\) vertices, which costs \(O(N^3)\).
Thus the total running time is the running time of the chosen cluster
generating algorithm plus \(O(N^3)\).
\end{proof}

\newpage

\bibliographystyle{alpha}
\bibliography{ref}

\appendix

\section{On the maximizing candidate in Lemma~\ref{lem:anchor-ring-dist}}
\label{sec:ring-maximizer-geometry}
When $\rmp$ is unknown, we need to understand the geometry of a maximizing candidate in the set $\Cs{v}$.

The main goal of this section is to establishing the following lemma:
\begin{lemma}[Geometry of the maximizing candidate]
\label{lem:ring-maximizer}
There exists a constant $C_{\tref{lem:ring-maximizer}}$ depending only on $L_\rmp$ and $\ell_\rmp$ so that the following holds.
Assume $\Evt{all}$ holds, and let $\rmp_1$ be a good anchor.
Fix \(v,w\in\bfV_1\) such that
\[
\delta_n^4\rG
\le
\D{X_v}{X_w}
\le
\delta_n^{2.3}\rG.
\]
Choose
\[
u^*\in \operatorname*{arg\,max}_{u\in\Cs{v}}\cn{\ring{u,v}}{w},
\]
and define
\[
D
:=
\D{X_{u^*}}{X_w}
-
\Big(\D{X_v}{X_{u^*}}-\D{X_v}{X_w}\Big).
\]
Then the following hold for all sufficiently large \(n\):

\begin{enumerate}
    \item[\textup{(i)}] \(0\le D\le C_{\tref{lem:ring-maximizer}}\delta_n^{-1}\frac{\xiq^2}{\rG}\).
    
    \item[\textup{(ii)}]
    \[
    \ang{}{X_{u^*}}{X_v}{X_w}
    \le
    C_{\tref{lem:ring-maximizer}}\,\delta_n^{-1}\frac{\xiq}{\rG}.
    \]
    
    \item[\textup{(iii)}] For every \(u'\in\ring{u^*,v}\),
    \[
    \big|\D{X_{u'}}{X_w}-\D{X_{u^*}}{X_w}\big|
    \le
    C_{\tref{lem:ring-maximizer}}\,\delta_n^{-1}\frac{\xiq^2}{\rG}.
    \]
    
    \item[\textup{(iv)}]
    \[
    \pi-\ang{}{X_w}{X_{u^*}}{X_v}
    \le
    C_{\tref{lem:ring-maximizer}}\,\delta_n^{-3}\frac{\xiq}{\rG}.
    \]
\end{enumerate}

In particular, on \(\Evt{all}\), this condition includes every \((v,w)\in\tar\)
with \(\D{X_v}{X_w}\ge \delta_n^4\rG\), since
\(\D{X_v}{X_w}\lesssim \delta_n^3\rG\ll \delta_n^{2.3}\rG\) for such pairs.
\end{lemma}

\begin{rem}
The quantity \(D\) is nonnegative by the triangle inequality. If \(D=0\), then
\(X_{u^*},X_w,X_v\) lie on a common minimizing geodesic, with \(X_w\) between
\(X_v\) and \(X_{u^*}\). Thus the smallness of \(D\) quantifies how close this
configuration is to being exactly collinear.
\end{rem}
Throughout this section, set
\[
\eta:=\delta_n^{-1}\frac{\xiq^2}{\rG}.
\]
To prove the first three items, we use the following geometric estimate.

\begin{lemma}[Small-angle estimate from one-sided defect]
\label{lem:angle-estimate-geo}
Let \(a,b,c\) be the side lengths of a geodesic triangle in \(M\), all at most
\(\rG\), and let \(\theta\) be the angle opposite to side \(c\).
Assume
\[
0<c\le \frac{1}{100}\,a,
\qquad
0\le b-(a-c)\le \frac14\,c.
\]
Then
\[
b-(a-c)
\ge
c_{\tref{lem:angle-estimate-geo}}\,
\frac{a^2}{c}\,\theta^2,
\]
for some universal constant \(c_{\tref{lem:angle-estimate-geo}}>0\).
\end{lemma}

\begin{rem}[Intuition behind the inequality]
    Not to interrupt the flow of the main proof, we defer the proof of Lemma~\ref{lem:angle-estimate-geo} to the appendix. Here we illustrate the intuition assuming that the triangle is in the Euclidean space. In that case, the law of cosines gives 
    \[
c^2=a^2+b^2-2ab\cos\theta,
\]
so on the branch near $b=a-c$ ($\theta$ is $0$ when $b=a-c$ and $b=a+c$),
\[
b-(a-c)=\frac{a(a-c)}{2c}\,\theta^2+O(\theta^4)
\quad (\theta\to 0).
\]
Thus, when $b$ is close to $a-c$ and $\theta$ is near $0$ (which is why we need the second condition in the lemma), one expects
$b-(a-c)\asymp \frac{a^2}{c}\theta^2$ (since $c$ is much smaller than $a$ in the lemma).
In the manifold setting, the proof needs two extra ingredients:
triangle comparison to the model spaces, and justification that we are on
the branch corresponding to $b\approx a-c$.
\end{rem}

\begin{proof}[Proof of the first three items of Lemma~\ref{lem:ring-maximizer}]
Let
\[
\theta_*:=\ang{}{X_{u^*}}{X_v}{X_w}.
\]

\step{1. Basic bounds for the maximizer}
Since \(u^*\in\Cs{v}\subseteq\meso\), Corollary~\ref{cor:mesoscopic-pairs-v1}
and our assumption on \(v,w\) give
\[
\D{X_v}{X_{u^*}}\asymp {\rm s}_1 \gtrsim \delta_n^{1.6}\rG \gg \delta_n^{2.3} \rG \ge
\D{X_v}{X_w} \ge \delta_n^4\rG.
\]
Moreover, by the definition of \(\Cs{v}\) and Lemma~\ref{lem:ring-single-avg},
\begin{equation}
\label{eq:u-star-anchor-distance}
\big|\rmp(\D{X_v}{X_{u^*}})-\rmp_1\big|
\lesssim \eta \qquad \Rightarrow \qquad 
\big|\D{X_v}{X_{u^*}}-{\rm s}_1\big|
\lesssim
\eta,
\end{equation}
where the implication follows from the lower Lipschitz bound of \(\rmp\). 

\step{2. Coarse control of the defect}
Fix \(u'\in\ring{u^*,v}\). Since \(u'\in \ring{u^*}\),
Corollary~\ref{cor:mesoscopic-pairs-v1} gives
\[
\D{X_{u^*}}{X_{u'}}
\lesssim
\delta_n\xiq
\qquad \Rightarrow \qquad
\big|\D{X_{u'}}{X_w}-\D{X_{u^*}}{X_w}\big|
\lesssim
\delta_n\xiq.
\]
Averaging over \(u'\in\ring{u^*,v}\), using the \(L_\rmp\)-Lipschitz continuity
of \(\rmp\), and then using \(\Enavi{\ring{u^*,v}}{w} \supseteq \Evt{ray}\supseteq \Evt{all}\)
and $\fe{\ring{u^*,v}} \le \xiq^2/\rG \ll \delta_n\xiq$ from Lemma \ref{lem:angle-screened-subset-coverage},
we obtain
\[
\Big|\cn{\ring{u^*,v}}{w}-\rmp(\D{X_{u^*}}{X_w})\Big|
\le 
\Big|\cn{\ring{u^*,v}}{w}-\acn{\ring{u^*,v}}{w}\Big|
+ 
\Big|\acn{\ring{u^*,v}}{w}-\rmp(\D{X_{u^*}}{X_w})\Big|
\lesssim
\delta_n\xiq.
\]
On the other hand, Lemma~\ref{lem:anchor-ring-dist} gives
\[
\Big|\cn{\ring{u^*,v}}{w}-\rmp({\rm s}_1-\D{X_v}{X_w})\Big|
\lesssim
\eta.
\]
Since \(\eta\ll \delta_n\xiq\), it follows that
\[
\big|\rmp(\D{X_{u^*}}{X_w})-\rmp({\rm s}_1-\D{X_v}{X_w})\big|
\lesssim
\delta_n\xiq
\quad \Rightarrow \quad
\big|\D{X_{u^*}}{X_w}-({\rm s}_1-\D{X_v}{X_w})\big|
\lesssim
\delta_n\xiq.
\]
Therefore
\begin{equation}
\label{eq:coarse-D-bound}
\begin{aligned}
0\le D
=&
\D{X_{u^*}}{X_w}
- \big(\D{X_v}{X_{u^*}}-\D{X_v}{X_w}\big) \\
\le&
\big|\D{X_{u^*}}{X_w}-({\rm s}_1-\D{X_v}{X_w})\big|
+
\big|\D{X_v}{X_{u^*}}-{\rm s}_1\big|
\lesssim
\delta_n\xiq.
\end{aligned}
\end{equation}
From the assumption \(\D{X_v}{X_w}\ge \delta_n^4\rG\) \[
D \ll \D{X_v}{X_w}\,.
\]

\step{3. Small angle at \(X_{u^*}\)}
We now apply Lemma~\ref{lem:angle-estimate-geo} to the triangle
\((X_v,X_{u^*},X_w)\). Since
\[
\D{X_v}{X_w}\ll \D{X_v}{X_{u^*}},
\qquad
D \ll \D{X_v}{X_w}
\]
all hypotheses are satisfied for large \(n\). 
\begin{equation}
\label{eq:theta-star-bound}
D
\gtrsim
\frac{\D{X_v}{X_{u^*}}^2}{\D{X_v}{X_w}}\,\theta_*^2
\qquad \Rightarrow \qquad
\theta_*
\lesssim
\frac{\sqrt{D\,\D{X_v}{X_w}}}{\D{X_v}{X_{u^*}}}.
\end{equation}

\step{4. Refined fluctuation of \(\D{X_{u'}}{X_w}\) on the screened ring}
Fix again \(u'\in\ring{u^*,v}\).
Lemma~\ref{lem:angle-screened-subset}\textup{(i)}
gives
\[
\Big|\cos\!\big(\ang{}{X_{u^*}}{X_v}{X_{u'}}\big)\Big|
\le
\delta_n^{-1}\frac{\D{X_{u^*}}{X_{u'}}}{\D{X_v}{X_{u^*}}}.
\]
By the triangle inequality for angles at \(X_{u^*}\),
\[
\Big|\ang{}{X_{u^*}}{X_{u'}}{X_w}-\ang{}{X_{u^*}}{X_v}{X_{u'}}\Big|
\le
\theta_*.
\]
Hence, using the \(1\)-Lipschitz continuity of \(\cos\),
\[
\Big|\cos\!\big(\ang{}{X_{u^*}}{X_{u'}}{X_w}\big)\Big|
\le
\delta_n^{-1}\frac{\D{X_{u^*}}{X_{u'}}}{\D{X_v}{X_{u^*}}}+\theta_*.
\]
Applying Lemma~\ref{_lem: M-opposite-side} to the triangle
\((X_{u^*},X_w,X_{u'})\), and using
\(\D{X_{u^*}}{X_{u'}}\ll \D{X_{u^*}}{X_w}\asymp \D{X_v}{X_{u^*}}\), we obtain
\[
\big|\D{X_{u'}}{X_w}-\D{X_{u^*}}{X_w}\big|
\lesssim
\D{X_{u^*}}{X_{u'}}\left(
\delta_n^{-1}\frac{\D{X_{u^*}}{X_{u'}}}{\D{X_v}{X_{u^*}}}+ \theta_*
\right)
+
4\frac{\D{X_{u^*}}{X_{u'}}^2}{\D{X_{u^*}}{X_w}}.
\]
Since \(\D{X_{u^*}}{X_{u'}}\lesssim \delta_n\xiq\) and
\(\D{X_v}{X_{u^*}}\asymp \D{X_{u^*}}{X_w}\asymp {\rm s}_1\asymp \delta_n^{1.5}\rG\), the R.H.S. can be simplifed to 
\begin{equation}
\label{eq:distance-fluctuation-pre}
\big|\D{X_{u'}}{X_w}-\D{X_{u^*}}{X_w}\big|
\lesssim
\delta_n\xiq\,\theta_*+\eta.
\end{equation}
With this improved bound compared to Step 2, we repeat the same argument as in Step 2, obtaining
\[
\Big|\cn{\ring{u^*,v}}{w}-\rmp(\D{X_{u^*}}{X_w})\Big|
\lesssim
\delta_n\xiq\,\theta_*+\eta,
\]
and
\[
D
=
\D{X_{u^*}}{X_w}
- \big(\D{X_v}{X_{u^*}}-\D{X_v}{X_w}\big)
\lesssim
\delta_n\xiq\,\theta_*+\eta.
\]
Substituting \eqref{eq:theta-star-bound}, this becomes
\[
D
\lesssim
\delta_n\xiq\,
\frac{\sqrt{D\,\D{X_v}{X_w}}}{\D{X_v}{X_{u^*}}}
+\eta.
\]
Applying \(xy\le \frac12 x^2+\frac12 y^2\) to \(x=\sqrt{D}\) and
\(y\asymp \delta_n\xiq\,\sqrt{\D{X_v}{X_w}}/\D{X_v}{X_{u^*}}\), we obtain
\[
D
\lesssim
\frac{\delta_n^2\xiq^2\,\D{X_v}{X_w}}{\D{X_v}{X_{u^*}}^2}
+\eta.
\]
Since \(\D{X_v}{X_w}\le \delta_n^{2.3}\rG\) and
\(\D{X_v}{X_{u^*}}\gtrsim \delta_n^{1.6}\rG\), the first term is
\(o(\eta)\). Hence
\[
D\lesssim \eta,
\]
which proves item \textup{(i)}.

Returning to \eqref{eq:theta-star-bound}, we now obtain
\[
\theta_*
\lesssim
\frac{\sqrt{\eta\,\D{X_v}{X_w}}}{\D{X_v}{X_{u^*}}}
\lesssim
\frac{\sqrt{\delta_n^{-1}\frac{\xiq^2}{\rG} \delta_n^{2.3}\rG}}{\delta_n^{1.6}\rG}
\ll
\delta_n^{-1}\frac{\xiq}{\rG},
\]
which proves item \textup{(ii)}.

Finally, substituting this improved angle bound into
\eqref{eq:distance-fluctuation-pre} gives
\[
\big|\D{X_{u'}}{X_w}-\D{X_{u^*}}{X_w}\big|
\lesssim
\eta,
\]
for every \(u'\in\ring{u^*,v}\), proving item \textup{(iii)}.
\end{proof}

To prove item \textup{(iv)}, we use the following standard near-\(\pi\) defect estimate.

\begin{lemma}[Near-\(\pi\) defect estimate]
\label{lem:angle-near-pi-defect}
Let \(a,b,c\) be the side lengths of a geodesic triangle in \(M\), all at most
\(\rG\), and let \(\alpha\) be the angle opposite to side \(a\). Then
\[
(b+c)-a
\ge
c_{\tref{lem:angle-near-pi-defect}}\,
\frac{bc}{b+c}\,(\pi-\alpha)^2,
\]
for some universal constant \(c_{\tref{lem:angle-near-pi-defect}}>0\).
\end{lemma}
Again, the lemma is a consequence of triangle comparison to the model spaces, and we defer the proof to the appendix.

\begin{proof}[Proof of item \textup{(iv)} in Lemma~\ref{lem:ring-maximizer}]
Set
\[
a:=\D{X_{u^*}}{X_v},
\qquad
b:=\D{X_{u^*}}{X_w},
\qquad
c:=\D{X_v}{X_w},
\qquad
\alpha:=\ang{}{X_w}{X_{u^*}}{X_v}.
\]
By definition,
\[
D=(b+c)-a.
\]
Applying Lemma~\ref{lem:angle-near-pi-defect} to the triangle
\((X_w,X_{u^*},X_v)\), we obtain
\[
D
\gtrsim
\frac{bc}{b+c}\,(\pi-\alpha)^2.
\]
Since \(a\asymp {\rm s}_1\) and \(c\ll a\), we also have \(b\asymp a\gg c\), and
therefore
\[
\frac{bc}{b+c}\gtrsim c.
\]
Using item \textup{(i)}, we conclude that
\[
\pi-\alpha
\lesssim
\sqrt{\frac{D}{c}}
\lesssim
\sqrt{\frac{\eta}{\D{X_v}{X_w}}}.
\]
Since \(\D{X_v}{X_w}\ge \delta_n^4\rG\), this gives
\[
\pi-\ang{}{X_w}{X_{u^*}}{X_v}
\ll
\delta_n^{-3}\frac{\xiq}{\rG},
\]
which proves item \textup{(iv)}.
\end{proof}

\section{Recovering distance when $\rmp$ is unknown}
\label{sec:recovering-distance-unknown-rmp}

The goal of this section is to recover the latent distances
\(\D{X_v}{X_w}\) for all \((v,w)\in\tar\) when the function \(\rmp\) is not known
a priori. Since \(\rmp\) is unknown, distances can only be recovered up to an
overall scaling factor.

\begin{theor}[Local distance recovery when $\rmp$ is unknown]
\label{thm:reconstruct-q} 
Consider the Riemannian graph model satisfying Assumption~\ref{assump: Riemannian model}. Suppose moreover that there exists an $(r_0,\lambda_0,p_0)$ cluster generating algorithm with $r_0 \le c_{\rm cn}\rG$, where $c_{\rm cn}>0$ is the constant from Lemma~\ref{lem:delta-cn-proxy}, which depends only on $L_\rmp$ and $\ell_\rmp$.

Then there is an algorithm which, given as input the sampled graph $G$ on $5n$ vertices, the parameters $L_\rmp$,$\ell_\rmp$, $\rG$, $c_\mu$, and $d$, and a prescribed subset $\bfV_1$ of size $n$, outputs with probability at least $1-p_0-n^{-\omega(1)}$ a symmetric set ${\cal P}\subseteq \bfV_1\times \bfV_1$ and a function $\hat{d}:{\cal P}\to \mathbb{R}$ such that
$$
\text{for every } (v,w)\in {\cal P}, \qquad
|\hat{d}(v,w) - R\D{X_v}{X_w}| \le 
\delta_n^{-7}\frac{\xiq^2}{\rG}
$$
and ${\cal P}$ contains every pair $(v,w)\in \bfV_1^2$ such that 
$$
    \D{X_v}{X_w} \le  c \delta_n^{3}\rG,
$$
where $c = C_{\tref{lem:delta-cn-proxy}}^{-1}$ is the constant from Lemma~\ref{lem:delta-cn-proxy}, which depends only on $L_\rmp$ and $\ell_\rmp$, and $R$ is a unknown scaling factor such that $\max{R,1/R}$ is bounded by a constant depending only on $L_\rmp$ and $\ell_\rmp$.
Moreover, the running time of this algorithm is the running time of the $(r_0,\lambda_0,p_0)$ cluster generating algorithm plus $O(n^3)$.
\end{theor}

Clearly, in the statement of Theorem~\ref{thm:reconstruct-q}, the set ${\cal P}$ is simply $\tar$. We will show we can construct the corresponding estimator when $\Evt{all}$ holds. 
\paragraph{Choice of an anchor.}
We first choose a good anchor value \(\rmp_1\) as follows. Select any
\((u_*,v_*)\in\meso\) such that
\[
\Delta_{\rm cn}(u_*,v_*)\asymp \delta_n^{1.5}\rG,
\]
and define
\[
\rmp_1:=\cn{\ring{u_*,v_*}}{v_*}.
\]
By Lemma~\ref{lem:delta-cn-proxy} and Lemma~\ref{lem:ring-single-avg}, this
choice yields a good anchor in the sense of Definition~\ref{def:anchor}. We
therefore fix such a value \(\rmp_1\), and write
\[
{\rm s}_1:=\rmp^{-1}(\rmp_1).
\]

\begin{defi}[Reference geodesic]
\label{def:reference-geodesic}
Let $C_{\tref{def:reference-geodesic}}>0$ be a constant to be determined, which depends only on \(L_\rmp\) and \(\ell_\rmp\).
Choose a pair \((w^0,w^1)\in\bfV_1^2\) such that
\begin{align}
    \label{def:w0w1}
C_{\tref{def:reference-geodesic}}\delta_n^3 \rG
\le \Delta_{\rm cn}(w^0,w^1) \le  2C_{\tref{def:reference-geodesic}} C_{\tref{lem:delta-cn-proxy}}^2\delta_n^3\rG\,.
\end{align}

Let
\[
\gamma:=\gamma_{w^0,w^1},
\]
and for any \(p\in M\) let \(\widetilde p\) denote the nearest-point projection
of \(p\) onto \(\gamma\).
Define the deterministic function
\[
\rmq(t):=\rmp({\rm s}_1-t)-\rmp({\rm s}_1),
\qquad
t\in[0,{\rm s}_1].
\]
\end{defi}
By Lemma~\ref{lem:delta-cn-proxy}, this implies
\[
\D{X_{w^0}}{X_{w^1}}\asymp \delta_n^3\rG \ll \delta_n^{1.5}\rG \asymp {\rm s}_1.
\]
Since \(\rmp\) is non-increasing and bi-Lipschitz on the relevant interval,
\(\rmq\) is monotone increasing, with lower Lipschitz constant \(\ell_\rmp\) and
upper Lipschitz constant \(L_\rmp\). Together with $\rmq(0)=0$, it implies 
$$
    \ell_\rmp\,t \le \rmq(t) \le L_\rmp\,t,
\qquad
\forall t\in[0,{\rm s}_1].
$$
The choice of $C_{\tref{def:reference-geodesic}}$ will be made in the proof of Theorem \ref{thm:reconstruct-q} so that $\D{X_{w^0}}{X_{w^1}}$ is sufficiently large comparing to $\D{X_v}{X_w}$ for all $(v,w) \in \tar$, which has an uniform bound $C_{\tref{lem:delta-cn-proxy}}\delta_n^3\rG$ by definition of $\tar$. 

For each \((v,w)\in\bfV_1^2\),
\begin{equation}
\label{eq:rmq-def}
\rmq_{\rm obs}(v,w)
:=
\cn{\ring{u(v,w),v}}{w}
-
\cn{\ring{u(v,w),v}}{v}
\qquad\text{where}\quad
u(v,w)\in\operatorname*{arg\,max}_{u\in\Cs{v}}\cn{\ring{u,v}}{w},
\end{equation}
By Lemma \ref{lem:ring-single-avg} and Lemma~\ref{lem:anchor-ring-dist}, the observable quantity
\(\rmq_{\rm obs}(v,w)\) approximates \(\rmq(\D{X_v}{X_w})\) uniformly on
\(\tar\).

\begin{cor}[Monotonicity oracle]
\label{cor:rmq}
For every with $\D{X_v}{X_w} \le \delta_n^{1.7}\rG$,
\[
\big|\rmq_{\rm obs}(v,w)-\rmq(\D{X_v}{X_w})\big|
\le
\underbrace{(C_{\tref{lem:anchor-ring-dist}} + C_{\tref{lem:ring-single-avg}})}_{=:C_{\tref{cor:rmq}}}
\delta_n^{-1}\frac{\xiq^2}{\rG}.
\]
In particular, whenever \((v,w),(v,w')\in\tar\) and
\[
\D{X_v}{X_w}
>
\D{X_v}{X_{w'}}
+
C\,\delta_n^{-1}\frac{\xiq^2}{\rG},
\]
we have
\[
\rmq_{\rm obs}(v,w)\ge \rmq_{\rm obs}(v,w'),
\]
for a constant \(C>0\) depending only on \(L_\rmp\) and \(\ell_\rmp\).

\end{cor}

\paragraph{A discrete net along the reference geodesic.}
Set
\[
q_*:=\rmq_{\rm obs}(w^0,w^1)>0.
\]
Choose a mesh size
\[
\varepsilon_{\rm net}:=\delta_n^{-4.6}\frac{\xiq^2}{\rG},
\]
and define
\[
q_0:= \tfrac14 q_*,
\qquad
q_{s+1}:=q_s+\varepsilon_{\rm net},
\qquad
S:=\max\{s:\ q_s\le \tfrac34 q_* \}.
\]
For each \(s\in[0,S]\), let
\[
t_s:=\rmq^{-1}(q_s).
\]
and choose \(v_s\in\bfV_1\) by
\[
v_s
\in
\operatorname*{arg\,min}_{v}
\rmq_{\rm obs}(w^1,v)
\quad\text{subject to}\quad
\rmq_{\rm obs}(w^0,v)\le q_s,
\]
where the minimum is taken over all \(v\in\bfV_1\) such that
\[
(v,w^0)\in\tar
\qquad\text{and}\qquad
(v,w^1)\in\tar.
\]
The point \(X_{v_s}\) should be thought of as a sample point whose projection
\(\widetilde X_{v_s}\) lies near the point on \(\gamma\) at distance \(t_s\) from
\(X_{w^0}\).

The first main step is to show that these vertices indeed form a discrete net
along the geodesic \(\gamma\).

\begin{lemma}
\label{lem:net-candidates}
There exists a constant \(C_{\tref{lem:net-candidates}}>0\) depending only on \(L_\rmp\) and \(\ell_\rmp\) such that the following holds: 
For each \(s\in[0,S]\), the projection \(\widetilde X_{v_s}\) lies on the
geodesic segment from \(X_{w^0}\) to \(X_{w^1}\), and satisfies
\[
\D{X_{v_s}}{\widetilde X_{v_s}}
\le
C_{\tref{lem:net-candidates}}\delta_n\xiq,
\qquad
\Big|\D{\widetilde X_{v_s}}{X_{w^0}}-t_s\Big|
\le
C_{\tref{lem:net-candidates}}
\delta_n^{-1}\frac{\xiq^2}{\rG}.
\]
\end{lemma}
To continue our discussion, we postpone the proof of the above lemma to the end of this section, and first explain how it is used in the rest of the argument. 
\begin{rem}
\label{rem:net-candidates}
Since the mesh size is \(\varepsilon_{\rm net}\gg \delta_n^{-1}\frac{\xiq^2}{\rG} \), the projected points
\[
\widetilde X_{v_1},\widetilde X_{v_2},\dots,\widetilde X_{v_S}
\]
appear in increasing order along \(\gamma\), and consecutive projections are
separated by distance comparable to \(\varepsilon_{\rm net}\).
\end{rem}

This construction is analogous to the one-dimensional net procedure in \cite{FMR25}. The next step is to replace the original oracle \(\rmq_{\rm obs}\) by ring statistics built from candidates whose corresponding slabs lie almost on the geodesic \(\gamma\) and are nearly perpendicular to it. 

\paragraph{Ring maximizers attached to the net.}
For each \(s\in[0,S]\), define
\[
u_s^0\in \operatorname*{arg\,max}_{u\in\Cs{v_s}}\cn{\ring{u,v_s}}{w^0},
\qquad
u_s^1\in \operatorname*{arg\,max}_{u\in\Cs{v_s}}\cn{\ring{u,v_s}}{w^1}.
\]
These points should be regarded as two calibrated rays based at \(v_s\), one
pointing toward \(w^0\) and the other toward \(w^1\).

\begin{lemma}
\label{lem:ring-u-s-i-w-s'}
There exists a constant
\(
C_{\tref{lem:ring-u-s-i-w-s'}}>0
\)
depending only on \(L_\rmp\) and \(\ell_\rmp\) such that the following holds.
For \(0 \le s'\le s \le s'' \le S\),
\[
\left|
\cn{\ring{u_s^0,v_s}}{v_{s'}}
-
\rmp\!\bigl({\rm s}_1-\D{\widetilde X_{v_s}}{\widetilde X_{v_{s'}}}\bigr)
\right|
\le
C_{\tref{lem:ring-u-s-i-w-s'}}
\,\delta_n^{-4.5}\frac{\xiq^2}{\rG},
\]
and
\[
\left|
\cn{\ring{u_s^1,v_s}}{v_{s''}}
-
\rmp\!\bigl({\rm s}_1-\D{\widetilde X_{v_s}}{\widetilde X_{v_{s''}}}\bigr)
\right|
\le
C_{\tref{lem:ring-u-s-i-w-s'}}
\,\delta_n^{-4.5}\frac{\xiq^2}{\rG}.
\]
\end{lemma}

To prove the above lemma, we need the following geometric control on the ring maximizers \(u_s^0\) and \(u_s^1\).
\begin{lemma}
\label{lem:net-candidates-ring}
There exists a constant \(C_{\tref{lem:net-candidates-ring}}>0\) depending only on \(L_\rmp\) and \(\ell_\rmp\) such that the following holds: 
For each \(s\in[0,S]\) and \(i\in\{0,1\}\), the point \(\widetilde X_{u_s^i}\)
lies on the opposite side of \(\widetilde X_{w^i}\) from \(\widetilde X_{v_s}\)
along \(\gamma\), and
\[
\Big|\D{\widetilde X_{u_s^i}}{\widetilde X_{v_s}}-{\rm s}_1\Big|
\le C_{\tref{lem:net-candidates-ring}}
\delta_n^{-4.5}\frac{\xiq^2}{\rG}.
\]
Moreover, for every \(u'\in\ring{u_s^i,v_s}\),
\[
\D{X_{u'}}{\widetilde X_{u'}}
\le
C_{\tref{lem:net-candidates-ring}}
\delta_n^{-1.5}\xiq,
\qquad
\D{\widetilde X_{u'}}{\widetilde X_{u_s^i}}
\le
C_{\tref{lem:net-candidates-ring}}
\delta_n^{-4.5}\frac{\xiq^2}{\rG},
\]
where $\wX_{u'}$ is the nearest-point projection of $X_{u'}$ onto $\gamma$.
\end{lemma}

\begin{proof}[Proof of Lemma~\ref{lem:ring-u-s-i-w-s'}]
It suffices to prove the first estimate; the second follows by the same argument.

Let \(u'\in\ring{u_s^0,v_s}\). By
Lemma~\ref{lem:net-candidates-ring}, the point \(\widetilde X_{u_s^0}\) lies on the
opposite side of \(\widetilde X_{w^0}\) from \(\widetilde X_{v_s}\) along \(\gamma\), and
\[
\Big|\D{\widetilde X_{u_s^0}}{\widetilde X_{v_s}}-{\rm s}_1\Big|
\le
C_{\tref{lem:net-candidates-ring}}
\,\delta_n^{-4.5}\frac{\xiq^2}{\rG}.
\]
Also, by Remark~\ref{rem:net-candidates}, the projected points
\(\widetilde X_{v_{s'}}\) and \(\widetilde X_{v_s}\) lie on the geodesic segment from
\(X_{w^0}\) to \(X_{w^1}\) in increasing order. Therefore the points
\[
\widetilde X_{u_s^0},\ X_{w^0},\ \widetilde X_{v_{s'}},\ \widetilde X_{v_s}
\]
occur in this order along \(\gamma\), and hence
\[
\D{\widetilde X_{u_s^0}}{\widetilde X_{v_{s'}}}
=
\D{\widetilde X_{u_s^0}}{\widetilde X_{v_s}}
-
\D{\widetilde X_{v_s}}{\widetilde X_{v_{s'}}}.
\]
It follows that
\[
\Big|
\D{\widetilde X_{u_s^0}}{\widetilde X_{v_{s'}}}
-
\bigl({\rm s}_1-\D{\widetilde X_{v_s}}{\widetilde X_{v_{s'}}}\bigr)
\Big|
\le
C_{\tref{lem:net-candidates-ring}}
\,\delta_n^{-4.5}\frac{\xiq^2}{\rG}.
\]

Next, Lemma~\ref{lem:net-candidates-ring} gives
\[
\D{\widetilde X_{u'}}{\widetilde X_{u_s^0}}
\le
C_{\tref{lem:net-candidates-ring}}
\,\delta_n^{-4.5}\frac{\xiq^2}{\rG},
\]
and therefore
\[
\Big|
\D{\widetilde X_{u'}}{\widetilde X_{v_{s'}}}
-
\bigl({\rm s}_1-\D{\widetilde X_{v_s}}{\widetilde X_{v_{s'}}}\bigr)
\Big|
\le
2C_{\tref{lem:net-candidates-ring}}
\,\delta_n^{-4.5}\frac{\xiq^2}{\rG}.
\]
Moreover, by Lemma~\ref{lem:net-candidates-ring} and Lemma~\ref{lem:net-candidates},
\[
\D{X_{u'}}{\widetilde X_{u'}}
\le
C_{\tref{lem:net-candidates-ring}}\delta_n^{-1.5}\xiq,
\qquad
\D{X_{v_{s'}}}{\widetilde X_{v_{s'}}}
\le
C_{\tref{lem:net-candidates}}\delta_n\xiq.
\]
Since
\[
\D{\widetilde X_{u'}}{\widetilde X_{v_{s'}}}
\asymp {\rm s}_1\asymp \delta_n^{1.5}\rG,
\]
Lemma~\ref{lem:two-orthogonal-perturbations} applies to the pair
\((X_{u'},X_{v_{s'}})\), and yields
\[
\Big|
\D{X_{u'}}{X_{v_{s'}}}
-
\D{\widetilde X_{u'}}{\widetilde X_{v_{s'}}}
\Big|
\lesssim
\frac{
\bigl(\D{X_{u'}}{\widetilde X_{u'}}+\D{X_{v_{s'}}}{\widetilde X_{v_{s'}}}\bigr)^2
}{
\D{\widetilde X_{u'}}{\widetilde X_{v_{s'}}}
}
\lesssim
\delta_n^{-4.5}\frac{\xiq^2}{\rG}.
\]
Combining the last two estimates, we obtain
\[
\Big|
\D{X_{u'}}{X_{v_{s'}}}
-
\bigl({\rm s}_1-\D{\widetilde X_{v_s}}{\widetilde X_{v_{s'}}}\bigr)
\Big|
\lesssim
\delta_n^{-4.5}\frac{\xiq^2}{\rG}.
\]
Since this holds for every \(u'\in\ring{u_s^0,v_s}\), the \(L_\rmp\)-Lipschitz
continuity of \(\rmp\) implies
\[
\Big|
\acn{\ring{u_s^0,v_s}}{v_{s'}}
-
\rmp\!\bigl({\rm s}_1-\D{\widetilde X_{v_s}}{\widetilde X_{v_{s'}}}\bigr)
\Big|
\lesssim
\delta_n^{-4.5}\frac{\xiq^2}{\rG}.
\]
Finally, since \(\Evt{all}\subseteq \Evt{ray}\)
and $\fe{\ring{u_s^0,v_s}}
\le
\frac{\xiq^2}{\rG}$ from Lemma~\ref{lem:angle-screened-subset-coverage}
, we have
\[
\Big|
\cn{\ring{u_s^0,v_s}}{v_{s'}}
-
\acn{\ring{u_s^0,v_s}}{v_{s'}}
\Big|
\le 
\fe{\ring{u_s^0,v_s}}
\ll
\delta_n^{-4.5}\frac{\xiq^2}{\rG}.
\]
This proves the first estimate.
\end{proof}
\subsection{Proof of Theorem~\ref{thm:reconstruct-q}}
\label{subsec:proof-reconstruct-q}
As a consequence of Lemma \ref{lem:ring-u-s-i-w-s'} 
and the bi-Lipschitz continuity of \(\rmq\), 
for \(0 \le s'<s<s'' \le S\), we have
\begin{align*}
\left| \D{\widetilde X_{v_s}}{\widetilde X_{v_{s'}}} - \D{\widetilde X_{v_s}}{\widetilde X_{v_{s''}}}\right|
- \delta_n^{-4.5}\frac{\xiq^2}{\rG} 
\lesssim 
    \left|
\cn{\ring{u_s^0,v_s}}{v_{s'}}
-
\cn{\ring{u_s^1,v_s}}{v_{s''}}
    \right| 
\lesssim &
\left| \D{\widetilde X_{v_s}}{\widetilde X_{v_{s'}}} - \D{\widetilde X_{v_s}}{\widetilde X_{v_{s''}}}\right|
+ \delta_n^{-4.5}\frac{\xiq^2}{\rG} \,.
\end{align*}

Next, by Lemma \ref{lem:net-candidates}, 
$$
 \Big||t_s-t_{s'}| - |t_s-t_{s''}| \Big| - 2C_{\tref{lem:net-candidates}}\delta_n^{-1}\frac{\xiq^2}{\rG} 
\le \left| \D{\widetilde X_{v_s}}{\widetilde X_{v_{s'}}} - \D{\widetilde X_{v_s}}{\widetilde X_{v_{s''}}}\right| \le  \Big||t_s-t_{s'}| - |t_s-t_{s''}| \Big|
+ 2C_{\tref{lem:net-candidates}}\delta_n^{-1}\frac{\xiq^2}{\rG} \,.
$$

Let $t^*$ be the midpoint of $t_{s'}$ and $t_{s''}$, and let $s^*$ be such that $t_{s^*}$ is closest to $t^*$. 
By the bi-Lipschitz continuity of $\rmq$, we have
$0\le t_{s+1}-t_s \le \frac{1}{\ell_\rmp} (q_{s+1}-q_s) \le \frac{1}{\ell_\rmp}\varepsilon_{\rm net}$ for every $s \in [0,S-1]$. Hence, 
$$
 \Big||t_s-t_{s'}| - |t_s-t_{s''}| \Big| \le 2|t_{s^*}-t^*| \le \frac{1}{2\ell_\rmp}\varepsilon_{\rm net} \ll \delta_n^{-5} \frac{\xiq^2}{\rG}\,.$$

This implies for the minmizer $s$ of $\left|\cn{\ring{u_s^0,v_s}}{v_{s'}} - \cn{\ring{u_s^1,v_s}}{v_{s''}}\right|$, 
$$
    \left|
\cn{\ring{u_s^0,v_s}}{v_{s'}}
-
\cn{\ring{u_s^1,v_s}}{v_{s''}}
    \right| \ll \delta_n^{-5}\frac{\xiq^2}{\rG}.
$$

\begin{rem}[Midpoint rule]
For $0 \le s' < s'+1 < s'' \le S$, if we choose $s$ to minimize the quantity     
\[
\left|
\cn{\ring{u_s^0,v_s}}{v_{s'}}
-
\cn{\ring{u_s^1,v_s}}{v_{s''}}
\right|,
\]
then the following midpoint approximation holds:
\[
\Big|
\D{\widetilde X_{v_s}}{\widetilde X_{v_{s'}}}
-
\D{\widetilde X_{v_s}}{\widetilde X_{v_{s''}}}
\Big|
\le
\delta_n^{-5}\frac{\xiq^2}{\rG}\,,
\]
or equivalently,
$$
\Big|
\D{\widetilde X_{v_s}}{\widetilde X_{v_{0}}} - \frac{1}{2} \left( \D{\widetilde X_{v_{s'}}}{\widetilde X_{v_{0}}} + \D{\widetilde X_{v_{s''}}}{\widetilde X_{v_{0}}} \right)
\Big|
\le
\delta_n^{-5}\frac{\xiq^2}{\rG}\,.
$$
\end{rem}

\begin{proof}[Proof of Theorem \ref{thm:reconstruct-q}]
Let 
$$
    {\rm d}_1 :=\D{\wX_{v_0}}{\wX_{v_S}} \,.  
$$
By Lemma \ref{lem:net-candidates}, bi-Lipschitz continuity of \(\rmq\), Corollary \ref{cor:rmq}, and Lemma \ref{lem:delta-cn-proxy},  we have
\begin{align*}
    {\rm d}_1 \ge& |t_S-t_0| - 2C_{\tref{lem:net-candidates}}\delta_n^{-1}\frac{\xiq^2}{\rG} \\
    \ge& \frac{ \tfrac34 q_* - \tfrac14 q_*}{L_\rmp} - 2C_{\tref{lem:net-candidates}}\delta_n^{-1}\frac{\xiq^2}{\rG}\\
    \ge& \frac{1}{2L_\rmp} \left( 
        \rmq(\D{X_{w^0}}{X_{w^1}}) -    
    C_{\tref{cor:rmq}}  \delta_n^{-1}\frac{\xiq^2}{\rG} \right) 
- 2C_{\tref{lem:net-candidates}}\delta_n^{-1}\frac{\xiq^2}{\rG} \\
    \ge&  
    \frac{1}{(2L_\rmp)^3}
    \Delta_{\rm cn}(w^0,w^1) -
    \frac{C_{\tref{cor:rmq}}}{2L_\rmp}  \delta_n^{-1}\frac{\xiq^2}{\rG} 
- 2C_{\tref{lem:net-candidates}}\delta_n^{-1}\frac{\xiq^2}{\rG}\,,
\end{align*}
Therefore, if we set 
$$
C_{\tref{def:reference-geodesic}} :=  100 \cdot \frac{1}{(2L_\rmp)^3} C_{\tref{lem:delta-cn-proxy}}\,.
$$
then we can guarantee that 
$$
    {\rm d}_1 \ge 2 C_{\tref{lem:delta-cn-proxy}} \delta_n^3\rG \ge 2 \max_{(v,w)\in\tar} \D{X_v}{X_w} \,.
$$
On the other hand, Lemma~\ref{lem:net-candidates} and the bi-Lipschitz
continuity of \(\rmq\) give
\[
{\rm d}_1
\le
|t_S-t_0|
+
2C_{\tref{lem:net-candidates}}\delta_n^{-1}\frac{\xiq^2}{\rG}
\le
\frac{q_S-q_0}{\ell_\rmp}
+
2C_{\tref{lem:net-candidates}}\delta_n^{-1}\frac{\xiq^2}{\rG}
\lesssim
q_*.
\]
Together with the lower bound above and Corollary~\ref{cor:rmq}, this yields
\[
{\rm d}_1\asymp q_* \asymp \D{X_{w^0}}{X_{w^1}} \asymp \delta_n^3\rG.
\]

Let $N =  \left\lfloor\log_2\left( \frac{\delta_n^3\rG}{\delta_n^{-7}\frac{\xiq^2}{\rG}}\right) \right\rfloor
$, the choice of $N$ guarantees that 
$$\frac{1}{2^N}{\rm d}_1 \gtrsim 
\delta_n^{-7} \frac{\xiq^2}{\rG} \gg N\delta_n^{-5}\frac{\xiq^2}{\rG},$$
where the last inequality uses the fact that $N \lesssim \log(n) \ll \delta_n^{-2}$ by definition of $\delta_n$.

Our goal is to estimate 
$$
    \rmq\Big(\frac{k}{2^N}{\rm d}_1\Big) 
$$ 
for $k = 0,1,\ldots,2^N$. We define a map $s(t)$ on the dyadic grid $t \in \left\{\frac{k}{2^N}\right\}_{k=0}^{2^N}$ as follows.
Set $s(0)=0$ and $s(1)=S$. Suppose $s(k/2^{N'})$ has been defined for $k=1,2,\ldots,2^{N'}$ for some $N'<N$.  
For each $k=1,2,\ldots,2^{N'}-1$, define $s\left( \frac{2k+1}{2^{N'+1}} \right)$ by the midpoint approximation in the previous remark: 
let $s' = s(k/2^{N'})$ and $s'' = s((k+1)/2^{N'})$, and
choose $s'\le s \le s''$ to minimize  
$$
|\cn{\ring{u_s^0,v_s}}{v_{s'}} - \cn{\ring{u_s^1,v_s}}{v_{s''}}|,
$$ 
and set $s\left( \frac{2k+1}{2^{N'+1}} \right)=s$. This recursively defines $s(t)$ for all $t \in \left\{\frac{k}{2^N}\right\}_{k=0}^{2^N}$, provided that $s'+1<s''$ at each step.  

\step{Validity of the midpoint approximation}
We {\bf claim} that for $0 \le N' \le N$ and $k=0,1,\ldots,2^{N'}$, 
$s=s(k/2^{N'})$ is well-defined and satisfies
$$
    \left|\D{\wX_{v_{s}}}{\wX_{v_0}} - \frac{k}{2^{N'}}{\rm d}_1\right| \le N' \cdot \delta_n^{-5}\frac{\xiq^2}{\rG}\,.
$$
Clearly this holds for $N'=0$ by definition of ${\rm d}_1$. Assume it holds for $N'$, we will show it holds for $N'+1$. First, it is clear that it holds for even $k$ values at it coincides with the previous step. For odd values, let us assume we are looking at
$$
    s\left( \frac{2k+1}{2^{N'+1}} \right) 
$$  
Thus, we are applying the midpoint approximation to $s' = s(k/2^{N'})$ and $s'' = s((k+1)/2^{N'})$. By the induction hypothesis, we have
$$
    \Big|\D{\wX_{v_{s''}}}{\wX_{v_{s'}}}  - \frac{1}{2^{N'}}{\rm d}_1\Big| \le 2 N' \delta_n^{-5}\frac{\xiq^2}{\rG} \,.
$$
Thus, with $N' \le N$, we have  
$$
\D{\wX_{v_{s''}}}{\wX_{v_{s'}}}   
\ge \frac{1}{2^{N'}}{\rm d}_1 - 2 N' \delta_n^{-5}\frac{\xiq^2}{\rG} 
\ge \frac{1}{2^{N}}{\rm d}_1 - 2 N \delta_n^{-5}\frac{\xiq^2}{\rG} 
\gg  \varepsilon_{\rm net} \gtrsim \max_{s} |t_{s+1}-t_s|,\, 
$$
implying that $s'+1 < s''$ and the midpoint approximation is valid. Moreover, by the midpoint approximation, for the chosen $s = s((2k+1)/2^{N'+1})$, we have
$$
    \left|\D{\wX_{v_{s}}}{\wX_{v_0}} - \frac{1}{2}\left( \D{\wX_{v_{s'}}}{\wX_{v_0}} + \D{\wX_{v_{s''}}}{\wX_{v_0}} \right)\right| \le \delta_n^{-5}\frac{\xiq^2}{\rG} \,.
$$
Together with the induction hypothesis, we have  
$$
\left| \frac{1}{2}\left( \D{\wX_{v_{s'}}}{\wX_{v_0}} + \D{\wX_{v_{s''}}}{\wX_{v_0}} \right) - \frac{2k+1}{2^{N'+1}}{\rm d}_1\right| \le \frac{N' \delta_n^{-5}\frac{\xiq^2}{\rG}+N' \delta_n^{-5}\frac{\xiq^2}{\rG}}{2} \le N'\delta_n^{-5}\frac{\xiq^2}{\rG} \,. 
$$
Combining the last two estimates, we conclude that
$$
\left|\D{\wX_{v_{s}}}{\wX_{v_0}} - \frac{2k+1}{2^{N'+1}}{\rm d}_1\right| \le (N'+1) \delta_n^{-5}\frac{\xiq^2}{\rG} \,.
$$
Therefore, the claim holds for $N'+1$. By induction, the claim holds for all $N' \le N$.

\step{Approximating $\rmq$ by}
For $k=0,1,\ldots,2^N$, define  
$$
    \bar \rmq\left(\frac{k}{2^N} \right) := \cn{\ring{u_0^0,v_0}}{v_0} - \cn{\ring{u_0^0,v_0}}{v_{s(k/2^N)}} \,.
$$
By Lemma \ref{lem:ring-u-s-i-w-s'} and the above claim, we have
$$
    \left| \bar \rmq\left(\frac{k}{2^N} \right) - \rmq\left(\frac{k}{2^N}{\rm d}_1\right)
    \right|
\lesssim N\delta_n^{-5}\frac{\xiq^2}{\rG}   
$$
Extending \(\bar \rmq\) to \([0,1]\) by linear interpolation, the
\(L_\rmp\)-Lipschitz continuity of \(\rmq\) gives
\[
\big|\bar \rmq(t)-\rmq(t{\rm d}_1)\big|
\lesssim
2^{-N}{\rm d}_1
+
N\delta_n^{-5}\frac{\xiq^2}{\rG},
\qquad
\forall t\in[0,1].
\]
Moreover, for every \(k=0,1,\ldots,2^N-1\),
\begin{align*}
\bar \rmq\left(\frac{k+1}{2^N}\right)-\bar \rmq\left(\frac{k}{2^N}\right)
\ge&
\rmq\left(\frac{k+1}{2^N}{\rm d}_1\right)
-
\rmq\left(\frac{k}{2^N}{\rm d}_1\right)
-
C N\delta_n^{-5}\frac{\xiq^2}{\rG} \\
\ge&
\ell_\rmp\frac{{\rm d}_1}{2^N}
-
C N\delta_n^{-5}\frac{\xiq^2}{\rG}
>0
\end{align*}
for all sufficiently large \(n\), since \(\frac{1}{2^N}{\rm d}_1 \gg N\delta_n^{-5}\frac{\xiq^2}{\rG}\). Hence \(\bar \rmq\) is strictly increasing on \([0,1]\).

\step{Approximating $\rmq_{\rm obs}(v,w)$ for $(v,w)\in\tar$}
By Corollary \ref{cor:rmq}, for every $(v,w)\in\tar$, we have
$$
\big|\rmq_{\rm obs}(v,w) - \rmq(\D{X_v}{X_w})\big|
\le
C_{\tref{cor:rmq}} \delta_n^{-1}\frac{\xiq^2}{\rG}.
$$
For such a pair, set
\[
\tau_{v,w}:=\frac{\D{X_v}{X_w}}{{\rm d}_1}\in[0,1],
\]
where the inclusion uses \({\rm d}_1\ge 2\max_{(v,w)\in\tar}\D{X_v}{X_w}\).
Choose
\[
\widehat \tau(v,w)
\in
\operatorname*{arg\,min}_{t\in[0,1]}
\big|\bar \rmq(t)-\rmq_{\rm obs}(v,w)\big|,
\]
and define
\[
{\rm d}(v,w):=q_*\,\widehat \tau(v,w).
\]
Then
\begin{align*}
\big|\bar \rmq(\tau_{v,w})-\rmq_{\rm obs}(v,w)\big|
\le&
\big|\bar \rmq(\tau_{v,w})-\rmq(\D{X_v}{X_w})\big|
+
\big|\rmq(\D{X_v}{X_w})-\rmq_{\rm obs}(v,w)\big| \\
\lesssim&
2^{-N}{\rm d}_1
+
N\delta_n^{-5}\frac{\xiq^2}{\rG}.
\end{align*}
By the definition of \(\widehat \tau(v,w)\), we also have
\[
\big|\bar \rmq(\widehat \tau(v,w))-\rmq_{\rm obs}(v,w)\big|
\le
\big|\bar \rmq(\tau_{v,w})-\rmq_{\rm obs}(v,w)\big|.
\]
Using the lower Lipschitz bound of \(\rmq\), it follows that
\begin{align*}
\ell_\rmp {\rm d}_1\,
\big|\widehat \tau(v,w)-\tau_{v,w}\big|
\le&
\big|\rmq(\widehat \tau(v,w){\rm d}_1)-\rmq(\tau_{v,w}{\rm d}_1)\big| \\
\le&
\big|\rmq(\widehat \tau(v,w){\rm d}_1)-\bar \rmq(\widehat \tau(v,w))\big|
+
\big|\bar \rmq(\widehat \tau(v,w))-\rmq_{\rm obs}(v,w)\big| \\
&+
\big|\rmq_{\rm obs}(v,w)-\bar \rmq(\tau_{v,w})\big|
+
\big|\bar \rmq(\tau_{v,w})-\rmq(\tau_{v,w}{\rm d}_1)\big| \\
\lesssim&
2^{-N}{\rm d}_1
+
N\delta_n^{-5}\frac{\xiq^2}{\rG}.
\end{align*}
Finally, set
\[
R:=\frac{q_*}{{\rm d}_1}.
\]
Since \({\rm d}_1\asymp q_*\), we have \(\max\{R,R^{-1}\}\lesssim 1\), and
\[
\big|{\rm d}(v,w)-R\D{X_v}{X_w}\big|
=
q_*\big|\widehat \tau(v,w)-\tau_{v,w}\big|
\lesssim
2^{-N}{\rm d}_1
+
N\delta_n^{-5}\frac{\xiq^2}{\rG}
\lesssim
\delta_n^{-7}\frac{\xiq^2}{\rG}.
\]
This completes the proof.
\end{proof}

\subsection{Proof of  Lemma~\ref{lem:net-candidates} and Lemma~\ref{lem:net-candidates-ring}}
\label{subsec:proof-net-candidates}
\begin{proof}[Proof of Lemma~\ref{lem:net-candidates}]
Set
\[
d_0:=\D{X_{w^0}}{X_{w^1}},
\qquad
\eta:=\delta_n^{-1}\frac{\xiq^2}{\rG}.
\]
By the choice of \((w^0,w^1)\in\tar\) and Lemma~\ref{lem:delta-cn-proxy},
\[
d_0\asymp \delta_n^3\rG.
\]

\step{1. The range of \(t_s\)}
Since
\[
q_*=\rmq_{\rm obs}(w^0,w^1),
\]
Corollary~\ref{cor:rmq} gives
\[
q_*=\rmq(d_0)+O(\eta).
\]
Here \(d_0\lesssim \delta_n^3\rG\ll {\rm s}_1\), so indeed \(d_0\in[0,{\rm s}_1]\), and hence
\[
\rmq(d_0)\asymp d_0
\]
by the bi-Lipschitz bounds on \([0,{\rm s}_1]\). Since \(\eta=o(d_0)\), it follows that
\[
q_*\asymp d_0.
\]
By definition,
\[
 \tfrac14 q_* \le q_s\le \tfrac34 q_*,
\]
and therefore, by bi-Lipschitz continuity of \(\rmq\) together with \(q_*\asymp d_0\), we obtain
\[
\frac{\tfrac14 q_* - 0}{L_\rmp} \le t_s \le \frac{\tfrac34 q_* - 0}{\ell_\rmp}
\qquad\mbox{and}\qquad 
\frac{q_* - \tfrac34 q_* }{L_\rmp} \le d_0 - t_s \le \frac{q_* - \tfrac14 q_*}{\ell_\rmp}\,,
\]
implying 
$$
    t_s \asymp d_0 - t_s \asymp q_* \asymp \delta_n^3\rG.
$$
\step{2. A comparison point near \(t_s\)}
Fix \(s\in[0,S]\), and choose a large absolute constant \(A>0\). Define
\[
t_s':=t_s-A\eta \simeq t_s.
\]
Since \(\Evt{all}\subseteq \Evt{locRing}\), there exists \(z_s\in\bfV_1\) such that
\[
X_{z_s}\in \mathcal S(w^0,w^1,t_s').
\]
Let \(\widetilde X_{z_s}\) be the nearest-point projection of \(X_{z_s}\) onto
\(\gamma=\gamma_{w^0,w^1}\). 
By Remark \ref{rem:Svwt-geometry}, we have  
\[
\D{X_{z_s}}{\widetilde X_{z_s}} \le \frac{9}{8}  \delta_n^{1.5}\xiq,
\qquad
\Big|\D{\widetilde X_{z_s}}{X_{w^0}}-t_s'\Big|\le \delta_n^4\xiq^2.
\]
Since \(t_s' \asymp \delta_n^3\rG\), the second bound implies 
$$
\D{\widetilde X_{z_s}}{X_{w^0}} \simeq t_s'.
$$
Now, we can apply Lemma~\ref{lem: M-right-angle} to the right triangle
\((X_{z_s},\widetilde X_{z_s},X_{w^0})\) yields
\[
\Big|\D{X_{z_s}}{X_{w^0}}-\D{\widetilde X_{z_s}}{X_{w^0}}\Big|
\lesssim
\frac{\D{X_{z_s}}{\widetilde X_{z_s}}^2}{\D{\widetilde X_{z_s}}{X_{w^0}}}
\lesssim
\eta\,.
\]
Thus
\[
\D{X_{z_s}}{X_{w^0}}
\le
t_s-\frac{A}{2}\eta
\]
for \(A\) sufficiently large. Similarly,
\[
\D{X_{z_s}}{X_{w^1}}
=
d_0-t_s+O(\eta).
\]

Now Corollary~\ref{cor:rmq} and the lower Lipschitz bound of \(\rmq\) imply
\[
\rmq_{\rm obs}(w^0,z_s)\le 
\rmq\Big(t_s - \frac{A}{2}\eta\Big) + C_{\tref{cor:rmq}}\eta 
\le q_s,
\]
provided \(A\) is chosen large enough. Hence \(z_s\) is an admissible candidate
in the definition of \(v_s\). Since \((z_s,w^0)\) and \((z_s,w^1)\) have latent
distances of order at most \(\delta_n^3\rG\), they both lie in \(\tar\) for all
sufficiently large \(n\).

\step{3. Endpoint distance bounds for the minimizer \(v_s\)}
Since \(v_s\) is admissible,
\[
\rmq_{\rm obs}(w^0,v_s)\le q_s=\rmq(t_s).
\]
Applying Corollary~\ref{cor:rmq} again, we obtain
\begin{align}
    \label{eq:vs-endpoint-bound-1}
\D{X_{w^0}}{X_{v_s}}
\le
t_s+O(\eta).
\end{align}
On the other hand, by minimality of \(v_s\) and admissibility of \(z_s\),
\[
\rmq_{\rm obs}(w^1,v_s)\le \rmq_{\rm obs}(w^1,z_s).
\]
Using Corollary~\ref{cor:rmq} and the bound on \(\D{X_{z_s}}{X_{w^1}}\), we get
\[
\D{X_{v_s}}{X_{w^1}}
\le
d_0-t_s+O(\eta).
\]
Combining the above inequality with \eqref{eq:vs-endpoint-bound-1} we have
\[
d_0\le \D{X_{w^0}}{X_{v_s}}+\D{X_{v_s}}{X_{w^1}} \le d_0 + O(\eta).
\]
we conclude that
\begin{equation}
\label{eq:vs-endpoint-bounds}
\Big|\D{X_{w^0}}{X_{v_s}}-t_s\Big|
\lesssim
\eta,
\qquad
\D{X_{w^0}}{X_{v_s}}+\D{X_{v_s}}{X_{w^1}}-d_0
\lesssim
\eta.
\end{equation}
A simple consequence of the above bound is that 
$$
\D{X_{w^0}}{X_{v_s}} \asymp \D{X_{v_s}}{X_{w^1}} \asymp \delta_n^3\rG.
$$

\step{4. Projection onto the reference geodesic}
Let \(\widetilde X_{v_s}\) be the nearest-point projection of \(X_{v_s}\) onto
\(\gamma\).
We first claim that \(\widetilde X_{v_s}\) lies on the geodesic segment from
\(X_{w^0}\) to \(X_{w^1}\). Indeed, if \(\widetilde X_{v_s}\) lay beyond
\(X_{w^1}\), then
\[
\D{X_{w^0}}{X_{v_s}}\ge \D{X_{w^0}}{\widetilde X_{v_s}}>d_0,
\]
contradicting \eqref{eq:vs-endpoint-bounds}, since \(t_s\le d_0-c\delta_n^4\rG\).
The case where \(\widetilde X_{v_s}\) lies beyond \(X_{w^0}\) is ruled out
similarly using the bound on \(\D{X_{v_s}}{X_{w^1}}\). Thus
\[
\D{X_{w^0}}{\widetilde X_{v_s}}+\D{\widetilde X_{v_s}}{X_{w^1}}=d_0.
\]

Next, apply Lemma~\ref{lem: M-right-angle} to the right triangles
\((X_{w^0},\widetilde X_{v_s},X_{v_s})\) and
\((\widetilde X_{v_s},X_{w^1},X_{v_s})\). Since
\[
\D{X_{w^0}}{X_{v_s}}+\D{X_{v_s}}{X_{w^1}}-d_0\lesssim \eta
\]
and
\[
\D{X_{w^0}}{\widetilde X_{v_s}},\ \D{\widetilde X_{v_s}}{X_{w^1}}
\gtrsim \delta_n^3\rG,
\]
the lower bound in Lemma~\ref{lem: M-right-angle}
implies
\[
\D{X_{v_s}}{\widetilde X_{v_s}}^2\left(
\frac1{\D{X_{w^0}}{\widetilde X_{v_s}}}
+
\frac1{\D{\widetilde X_{v_s}}{X_{w^1}}}
\right)\lesssim \eta.
\]
Since
\[
\frac1{\D{X_{w^0}}{\widetilde X_{v_s}}}
+
\frac1{\D{\widetilde X_{v_s}}{X_{w^1}}}
\ge
\frac{4}{\D{X_{w^0}}{\widetilde X_{v_s}}+\D{\widetilde X_{v_s}}{X_{w^1}}}
=
\frac{4}{d_0},
\]
this yields
\[
\D{X_{v_s}}{\widetilde X_{v_s}}
\lesssim
\sqrt{d_0\eta}
\asymp
\sqrt{\delta_n^3\rG\cdot \delta_n^{-1}\frac{\xiq^2}{\rG}} = \delta_n \xiq\,.
\]

Finally, since
\[
\D{X_{v_s}}{\widetilde X_{v_s}}
\lesssim \delta_n \xiq
\ll
\delta_n^3 \rG \asymp 
\D{X_{w^0}}{\widetilde X_{v_s}} \asymp \D{\widetilde X_{v_s}}{X_{w^1}},
\]
the upper bound in Lemma~\ref{lem: M-right-angle}
gives
\[
\Big|\D{X_{w^0}}{X_{v_s}}-\D{X_{w^0}}{\widetilde X_{v_s}}\Big|
\lesssim
\frac{\D{X_{v_s}}{\widetilde X_{v_s}}^2}
{\D{X_{w^0}}{\widetilde X_{v_s}}}
\lesssim
\frac{\delta_n^2\xiq^2}{\delta_n^3 \rG} = \eta\,. 
\]
Combining this with \eqref{eq:vs-endpoint-bounds}, we obtain
\[
\Big|\D{X_{w^0}}{\widetilde X_{v_s}}-t_s\Big|\lesssim \eta.
\]
This completes the proof.
\end{proof}

\begin{proof}[Proof of Lemma~\ref{lem:net-candidates-ring}]
Let us remark it is enough to show the desired bounds for \(i=0\), since the case \(i=1\) can be treated identically. We will focus on the case \(i=0\) in the proof.

\step{Step 1: angle control at \(X_{w^0}\)}
Let \(\theta_s\) be the angle at \(X_{w^0}\) between \(\gamma\) and \(\gamma_{v_s,w^0}\), equivalently
\[
\theta_s:=\ang{}{X_{w^0}}{X_{v_s}}{\wX_{v_s}}.
\]
By Lemma~\ref{lem: M-right-angle} applied to the triangle formed by \(X_{w^0},X_{v_s},\wX_{v_s}\),
\begin{align}
    \label{eq:theta-i-sin}
    \sin\theta_s
    \le
    C_{\tref{lem: M-right-angle}}
    \frac{\D{X_{v_s}}{\wX_{v_s}}}{\D{X_{w^0}}{X_{v_s}}}.
\end{align}
From Lemma \ref{lem:net-candidates}, we have
\[\D{X_{v_s}}{\wX_{v_s}} \lesssim \delta_n\xiq
\quad \mbox{and}\quad
\D{X_{w^0}}{X_{v_s}} \simeq t_s \asymp \delta_n^3\rG\,.
\]
Therefore, together with $\sin\theta_s \simeq \theta_s$ for small $\theta_s$, we have
\begin{align}
    \label{eq:theta-i-sin-final}
    \theta_s
    \lesssim
    \delta_n^{-2}\frac{\xiq}{\rG}.
\end{align}

\step{Step 2: transfer to the maximizer direction}
Since \((v_s,w^0)\in\tar\), applying the last item of Lemma \ref{lem:ring-maximizer} with
$(v,w)=(v_s,w^0)$ and $u^*=u_s^0$ gives
\begin{align}
    \label{eq:ui-angle-with-wi}
    \pi-\ang{}{X_{w^0}}{X_{u_s^0}}{X_{v_s}}
    \le
    C_{\tref{lem:ring-maximizer}}\,\delta_n^{-3}\frac{\xiq}{\rG}.
\end{align}
By the triangle inequality for angles at \(X_{w^0}\), 
$\ang{}{X_{w^0}}{X_{u_s^0}}{\wX_{v_s}}$ is at most 
$$
    C_{\tref{lem:ring-maximizer}}\,\delta_n^{-3}\frac{\xiq}{\rG}
    +\delta_n^{-2}\frac{\xiq}{\rG}
$$
from $\pi$. Since both \(\wX_{v_s}\) and \(\wX_{u_s^0}\) lie on the same geodesic \(\gamma\) issuing from \(X_{w^0}\), a point of \(\gamma\) on the same side of \(X_{w^0}\) as \(\wX_{v_s}\) would make the angle at \(X_{w^0}\) with \(\wX_{v_s}\) equal to \(0\), whereas a point on the opposite side makes this angle equal to \(\pi\). Therefore \(\wX_{u_s^0}\) must lie on the opposite side of \(X_{w^0}\) from \(\wX_{v_s}\) along \(\gamma\). In other words, 
$$
    \ang{}{X_{w^0}}{X_{u_s^0}}{\wX_{u_s^0}} 
    \lesssim  \delta_n^{-3}\frac{\xiq}{\rG}+\delta_n^{-2}\frac{\xiq}{\rG}
    \lesssim \delta_n^{-3}\frac{\xiq}{\rG}.
$$
By the first item of Lemma~\ref{lem:ring-maximizer}, a rough estimate shows that  
$$
    \D{X_{u_s^0}}{X_{w^0}}  
    \lesssim \D{X_{v_s}}{X_{u_s^0}} \simeq  {\rm s}_1 \asymp \delta_n^{1.5}\rG.
$$
Then, by the sin-angle formula in Lemma \ref{lem: M-right-angle}, we have
\begin{align*}
    \D{X_{u_s^0}}{\wX_{u_s^0}}
    &\le
    C_{\tref{lem: M-right-angle}}
    \D{X_{u_s^0}}{X_{w^0}}
    \sin\!\Big(\ang{}{X_{w^0}}{X_{u_s^0}}{\wX_{u_s^0}}\Big)
    \lesssim
    \delta_n^{1.5}\rG \cdot 
    \delta_n^{-3}\frac{\xiq}{\rG}\,
    \lesssim \delta_n^{-1.5}\xiq.
\end{align*}

Using $\D{\wX_{v_s}}{X_{v_s}} \lesssim \delta_n\xiq$ from Lemma \ref{lem:net-candidates}  and 
$$
    \D{X_{u_s^0}}{X_{v_s}} \simeq \D{\wX_{u_s^0}}{\wX_{v_s}} \simeq {\rm s}_1 \asymp \delta_n^{1.5}\rG,
$$
we apply Lemma \ref{lem:two-orthogonal-perturbations} to get 
$$
|\D{\wX_{u_s^0}}{\wX_{v_s}} - \D{X_{u_s^0}}{X_{v_s}}| \lesssim 
\frac{\delta_n^{2}\xiq^2 + \delta_n^{-3}\xiq^2}{\delta_n^{1.5}\rG} \lesssim \delta_n^{-4.5}\frac{\xiq^2}{\rG}.
$$
Since \(u_s^0\in \Cs{v_s}\), the definition of \(\Cs{v_s}\) together with
Lemma~\ref{lem:ring-single-avg} gives
\[
\big|\rmp(\D{X_{u_s^0}}{X_{v_s}})-\rmp_1\big|
\lesssim \eta.
\]
As \(\rmp_1=\rmp({\rm s}_1)\), the lower Lipschitz bound of \(\rmp\) yields
\[
\big|\D{X_{u_s^0}}{X_{v_s}}-{\rm s}_1\big|
\lesssim \eta.
\]
Hence triangle inequalities give the first item of the lemma. 

\step{Step 3: stability control on $u' \in \ring{u_s^0,v_s}$}
Fix any $u' \in \ring{u_s^0,v_s}$. 
First, from triangle inequalities, we have
$$
    \D{\wX_{u'}}{X_{u'}} \le \D{\wX_{u_s^0}}{X_{u'}} 
    \le \D{\wX_{u_s^0}}{X_{u_s^0}} + \D{X_{u_s^0}}{X_{u'}} 
    \lesssim \delta_n^{-1.5}\xiq + \delta_n \xiq \lesssim \delta_n^{-1.5}\xiq,  
$$
where we used the fact that $u' \in \ring{u_s^0,v_s} \subseteq \ring{u_s^0}$ implies $\D{X_{u_s^0}}{X_{u'}} \lesssim \delta_n \xiq$ from Corollary~\ref{cor:mesoscopic-pairs-v1}.
To figure out the distance between $\wX_{u'}$ and $\wX_{u_s^0}$, we first claim that  
\begin{align*}
   \D{\wX_{u'}}{\wX_{u_s^0}} = |\D{\wX_{u'}}{X_{w^0}} - \D{\wX_{u_s^0}}{X_{w^0}}|.
\end{align*}
The only way this identity could fail is if \(X_{w^0}\) lay on the geodesic segment between \(\wX_{u'}\) and \(\wX_{u_s^0}\). But then monotonicity of arclength along \(\gamma\) would give
\[
\D{X_{w^0}}{\wX_{u_s^0}}
\le
\D{\wX_{u'}}{\wX_{u_s^0}}.
\]
On the other hand, by the triangle inequality,
\[
\D{\wX_{u'}}{\wX_{u_s^0}}
\le
\D{\wX_{u'}}{X_{u'}} + \D{X_{u'}}{X_{u_s^0}} + \D{X_{u_s^0}}{\wX_{u_s^0}}
\lesssim \delta_n^{-1.5}\xiq,
\]
whereas \(\D{X_{w^0}}{\wX_{u_s^0}}\gtrsim {\rm s}_1\asymp \delta_n^{1.5}\rG\). Since \(\delta_n^{-1.5}\xiq\ll \delta_n^{1.5}\rG\), this is impossible.  

Now, we estimate  
\begin{align}
\label{eq: ui-stability-triangle-ineq}
&|\D{\wX_{u'}}{X_{w^0}} - \D{\wX_{u_s^0}}{X_{w^0}}| \\
\nonumber
\le& \Big|\D{\wX_{u'}}{X_{w^0}} - \D{X_{u'}}{X_{w^0}}\Big| +  
\Big|\D{X_{u'}}{X_{w^0}} - \D{X_{u_s^0}}{X_{w^0}}\Big|
+ \Big|\D{X_{u_s^0}}{X_{w^0}} - \D{\wX_{u_s^0}}{X_{w^0}}\Big|.
\end{align}
We will handle the second summand in the above display first, and then the first and the third summands together.
Before we proceed, let us remark that by a coarse bounds based on triangle inequalities, 
each term 
$$\D{\wX_{u'}}{X_{w^0}} \asymp \D{X_{u'}}{X_{w^0}} \asymp \D{X_{u_s^0}}{X_{w^0}} \asymp \D{\wX_{u_s^0}}{X_{w^0}} 
\asymp \D{X_{u_s^0}}{X_{v_s}}
\asymp \delta_n^{1.5}\rG.$$

Due to $u' \in \ring{u_s^0,v_s}$, 
we apply Lemma \ref{lem:angle-screened-subset} to get 
the angle $\theta_1$ at $X_{u_s^0}$ between $X_{v_s}$ and $X_{u'}$ satisfies 
$$
    |\cos(\theta_1)| \le \delta_n^{-1} \frac{\D{X_{u'}}{X_{u_s^0}}}{\D{X_{v_s}}{X_{u_s^0}}} \lesssim \delta_n^{-1} \frac{\delta_n\xiq}{\delta_n^{1.5}\rG} \lesssim \delta_n^{-1.5} \frac{\xiq}{\rG},
$$ 
where we used $\D{X_{v_s}}{X_{u_s^0}} \simeq {\rm s}_1 \asymp \delta_n^{1.5}\rG$. 

On the other hand, by (ii) of Lemma \ref{lem:ring-maximizer}, the angle $\theta_2$ at $X_{u_s^0}$ between $X_{w^0}$ and $X_{v_s}$ is bounded by $\delta_n^{-1}\xiq/\rG$. Therefore, let $\theta$ be the angle at $X_{u_s^0}$ between $X_{w^0}$ and $X_{u'}$, by triangle inequality on angles, and the fact that $\cos$ has lipschitz constant at most $1$, we conclude that  
$$
    |\cos(\theta)| \lesssim \delta_n^{-1.5} \frac{\xiq}{\rG} + \delta_n^{-1}\frac{\xiq}{\rG} \lesssim \delta_n^{-1.5} \frac{\xiq}{\rG}.
$$
Clearly, $\D{X_{u'}}{X_{u_s^0}} = o(\D{X_{u_s^0}}{X_{w^0}})$, so we invoke Lemma \ref{_lem: M-opposite-side} to conclude that 
\begin{align*}
|\D{X_{u'}}{X_{w^0}} - \D{X_{u_s^0}}{X_{w^0}}| \le &
4\!\left(
\D{X_{u'}}{X_{u_s^0}}\,|\cos(\theta)|
\,+\,
\frac{\D{X_{u'}}{X_{u_s^0}}^2}{\D{X_{u_s^0}}{X_{w^0}}}
\right) \\
\lesssim&
\delta_n\xiq \cdot \delta_n^{-1.5} \frac{\xiq}{\rG} + \frac{\delta_n^{2}\xiq^2}{\delta_n^{1.5}\rG}
\lesssim \delta_n^{-0.5}\frac{\xiq^2}{\rG}.
\end{align*}
For the other two summands in \eqref{eq: ui-stability-triangle-ineq}, 
they are corresponding right triangles with the bounds 
$$
    \D{\wX_{u'}}{X_{u'}} \lesssim \delta_n^{-1.5}\xiq, \quad \mbox{and} \quad \D{\wX_{u_s^0}}{X_{u_s^0}} \lesssim \delta_n^{-1.5}\xiq\,.
$$
Applying the upper bound in Lemma \ref{lem: M-right-angle} to these two right triangles, both of the two summands are at most $\delta_n^{-4.5}\frac{\xiq^2}{\rG}$. Combining all the bounds together we obtain the lemma.  
\end{proof}

\section{Probability estimates}
\label{sec:probability-estimates}
For the proof of Lemma~\ref{lem:Ept}, we need the following estimate on the cardinality of nets in metric spaces with lower $\phi$-regularity:
\begin{lemma}[Lower $\phi$-regularity implies a bound on net cardinality]
Let $(M,d,\mu)$ be a metric probability space, and assume that there exist
$r_\mu>0$ and a nonnegative nondecreasing function
\[
\phi:(0,r_\mu]\to[0,\infty)
\]
such that
\[
\mu(B(x,r)) \geq \phi(r)
\qquad \text{for all } x\in M,\; 0<r\le r_\mu.
\]
Let $\delta\in(0,r_\mu)$, and let $\mathcal N\subset M$ be a maximal
$\delta$-separated set. Equivalently, $\mathcal N$ may be obtained by the
greedy construction: choose $x_1\in M$ arbitrarily, and for $i\ge2$ choose
$x_i\in M$ such that $d(x_i,x_j)>\delta$ for all $j<i$, until no such point
remains.

Then $\mathcal N$ is a $\delta$-net of $M$, and
\[
|\mathcal N| \le \frac{1}{\phi(\delta/2)}.
\]
\end{lemma}

\begin{proof}
Since $\mathcal N$ is maximal $\delta$-separated, it is a $\delta$-net: if
there were some $x\in M$ such that $d(x,y)>\delta$ for all $y\in\mathcal N$,
then one could adjoin $x$ to $\mathcal N$ and still obtain a
$\delta$-separated set, contradicting maximality.

Now let $\mathcal N=\{x_1,\dots,x_N\}$. Because $\mathcal N$ is
$\delta$-separated, we have
\[
d(x_i,x_j)>\delta \qquad \text{whenever } i\neq j.
\]
Hence the balls $B(x_i,\delta/2)$ are pairwise disjoint. Therefore,
using that $\mu$ is a probability measure and the assumed lower
$\phi$-regularity,
\[
1=\mu(M)\ge \sum_{i=1}^N \mu\bigl(B(x_i,\delta/2)\bigr)
   \ge \sum_{i=1}^N \phi(\delta/2)
   = N\,\phi(\delta/2).
\]
It follows that
\[
|\mathcal N| = N \le \frac{1}{\phi(\delta/2)}.
\]
This proves the claim.
\end{proof}

\begin{proof}[Proof of Lemma~\ref{lem:Ept}]
Define
\[
\delta_k:=\Bigl(\frac32\Bigr)^k\rho_n,
\qquad k\ge 1,
\]
and let \(K\) be maximal such that \(\delta_K\le r_\mu/2\). For each
\(k\in\{1,\dots,K\}\), let \(\mathcal N_k\) be a maximal
\(\delta_k\)-separated set.

Since \(\delta_k>\rho_n\) for every \(k\ge 1\), and the function
\(\delta\mapsto \phi(\delta/2)\) is nondecreasing, the definition of
\(\rho_n\) implies that
\[
\phi(\delta_k/2)\ge \frac{\log^2 n}{n}
\qquad\text{for all } k\ge 1.
\]

For each \(k\in\{1,\dots,K\}\), let \(E_k\) be the event that
\[
|X_{\bfV}\cap B(x_i,\delta_k)|
\ge \frac{n\phi(\delta_k)}{2}
\qquad \text{for all } x_i\in\mathcal N_k.
\]
Fix \(k\) and \(x_i\in\mathcal N_k\). Then
\[
|X_{\bfV}\cap B(x_i,\delta_k)|
= \sum_{v\in \bfV} \mathbf 1_{\{X_v\in B(x_i,\delta_k)\}}
\]
is a binomial random variable with parameters \(n\) and
\[
p_i := \mu(B(x_i,\delta_k)).
\]
By assumption,
\[
p_i \ge \phi(\delta_k).
\]
Hence, by the multiplicative Chernoff bound,
\[
\mathbb P\!\left(
|X_{\bfV}\cap B(x_i,\delta_k)|
\le \frac{n\phi(\delta_k)}{2}
\right)
\le
\mathbb P\!\left(
|X_{\bfV}\cap B(x_i,\delta_k)|
\le \frac{np_i}{2}
\right)
\le
\exp\!\left(-\frac{np_i}{8}\right)
\le
\exp\!\left(-\frac{n\phi(\delta_k)}{8}\right).
\]
Applying the union bound over all \(x_i\in\mathcal N_k\), and using the
lemma on net cardinality, we obtain
\[
\mathbb P(E_k^c)
\le
|\mathcal N_k| \exp\!\left(-\frac{n\phi(\delta_k)}{8}\right)
\le
\frac{1}{\phi(\delta_k/2)}
\exp\!\left(-\frac{n\phi(\delta_k)}{8}\right).
\]
Since \(\phi\) is nondecreasing and \(\delta_k\ge \delta_1>\rho_n\), we have
\[
\phi(\delta_k/2)\ge \phi(\delta_1/2)\ge \frac{\log^2 n}{n}
\qquad\text{and}\qquad
\phi(\delta_k)\ge \phi(\delta_k/2)\ge \frac{\log^2 n}{n}.
\]
Therefore
\[
\mathbb P(E_k^c)
\le
\frac{n}{\log^2 n}\exp\!\left(-\frac{\log^2 n}{8}\right)
=
n^{-\omega(1)}.
\]
Moreover, by the definition of \(\rho_n\), we have \(\rho_n\ge r_\mu e^{-n}\),
and therefore
\[
K = O(n),
\]
so the union bound over \(k=1,\dots,K\) yields
\[
\mathbb P\Bigl(\bigcap_{k=1}^K E_k\Bigr)
\ge
1-Kn^{-\omega(1)}
=
1-n^{-\omega(1)}.
\]

It therefore remains to show that
\[
\bigcap_{k=1}^K E_k \subseteq \Ept{\bfV}.
\]
To this end, suppose that \(\bigcap_{k=1}^K E_k\) holds, and fix \(x\in M\)
and \(r\in[4\rho_n,r_\mu]\). Choose \(k\in\{1,\dots,K\}\) maximal such that
\[
\delta_k \le \frac r2.
\]
Such a \(k\) exists because
\[
\delta_1=\frac32\rho_n \le 2\rho_n \le \frac r2.
\]
By maximality,
\[
\delta_{k+1}=\frac32\delta_k > \frac r2,
\]
and hence
\[
\delta_k > \frac r3.
\]
Since \(\mathcal N_k\) is a maximal \(\delta_k\)-separated set, it is also a
\(\delta_k\)-net. Thus there exists \(x_i\in\mathcal N_k\) such that
\[
d(x,x_i)\le \delta_k.
\]
If \(y\in B(x_i,\delta_k)\), then
\[
d(y,x)\le d(y,x_i)+d(x_i,x)\le \delta_k+\delta_k=2\delta_k\le r,
\]
and therefore
\[
B(x_i,\delta_k)\subset B(x,r).
\]
Hence
\[
|X_{\bfV}\cap B(x,r)|
\ge
|X_{\bfV}\cap B(x_i,\delta_k)|.
\]
Since \(E_k\) holds, we have
\[
|X_{\bfV}\cap B(x_i,\delta_k)|
\ge \frac{n\phi(\delta_k)}{2}.
\]
Finally, because \(\delta_k>r/3\) and \(\phi\) is nondecreasing,
\[
\phi(\delta_k)\ge \phi(r/3).
\]
Combining the above inequalities yields
\[
|X_{\bfV}\cap B(x,r)|
\ge \frac{n\phi(r/3)}{2}.
\]
Thus \(\Ept{\bfV}\) holds, proving
\[
\bigcap_{k=1}^K E_k \subseteq \Ept{\bfV}.
\]
Consequently,
\[
\mathbb P\bigl(\Ept{\bfV}\bigr)
\ge
\mathbb P\Bigl(\bigcap_{k=1}^K E_k\Bigr)
\ge
1-n^{-\omega(1)}.
\]
\end{proof}

The proof of Lemma \ref{_lem: navi} is a consequence of the following statement together with the union bound over $v \in V$. 
\begin{lemma}[Fluctuation of normalized averages for a fixed vertex]
\label{lem:fixed_v_navi}
Assume the graph model of Definition~\ref{def:graph_model}, and suppose moreover
that
\[
M_\rmp:=\sup_{0 \le t \le {\rm diam}(M)} |\rmp(t)| < \infty.
\]
Let \(U\subseteq \bfV\), let \(v\in \bfV\setminus U\), and set \(m:=|U|\).
Assume that
\[
\sp m \ge \log^2 n.
\]
Then for every fixed
realization \(X_U=x_U\) and \(X_v=x_v\),
\[
\Pr\!\left(
\left|\cn{U}{v}-\acn{U}{v}\right|
>
\frac{\log n}{\sqrt{\sp\,m}}
\ \middle|\ X_U=x_U,\ X_v=x_v
\right)
\le n^{-\omega(1)}.
\]
\end{lemma}

\begin{proof}
Fix realizations \(X_U=x_U\) and \(X_v=x_v\), and write
\[
a_u:=\rmp\!\bigl(\D{x_u}{x_v}\bigr),
\qquad
\xi_u:=\widetilde Z_{u,v}-a_u,
\qquad
B_u:=B_{u,v},
\qquad u\in U.
\]
Then, conditional on \(X_U=x_U\) and \(X_v=x_v\),

\begin{itemize}
    \item \(\{B_u\}_{u\in U}\) are independent Bernoulli random variables with
    parameter \(\sp\);
    \item \(\{\xi_u\}_{u\in U}\) are independent, mean-zero, \(\Ksg\)-subgaussian
    random variables;
    \item the families \(\{B_u\}_{u\in U}\) and \(\{\xi_u\}_{u\in U}\) are
    independent;
    \item \(|a_u|\le M_\rmp\) for all \(u\in U\).
\end{itemize}

Since \(Z_{u,v}=B_u(a_u+\xi_u)\), we may write
\[
\cn{U}{v}-\acn{U}{v}
=
\frac{1}{\sp m}\sum_{u\in U}\bigl(Z_{u,v}-\sp a_u\bigr)
=
T_1+T_2,
\]
where
\[
T_1:=\frac{1}{\sp m}\sum_{u\in U}(B_u-\sp)a_u,
\qquad
T_2:=\frac{1}{\sp m}\sum_{u\in U} B_u\xi_u.
\]
It therefore suffices to show that
\[
\Pr\!\left(
|T_1|>\frac{\log n}{2\sqrt{\sp m}}
\ \middle|\ X_U=x_U,\ X_v=x_v
\right)
\le n^{-\omega(1)}
\]
and
\[
\Pr\!\left(
|T_2|>\frac{\log n}{2\sqrt{\sp m}}
\ \middle|\ X_U=x_U,\ X_v=x_v
\right)
\le n^{-\omega(1)}.
\]

\medskip
\noindent
\emph{Step 1: control of \(T_1\).}
For each \(u\in U\), define
\[
Y_u:=(B_u-\sp)a_u.
\]
Then \(\{Y_u\}_{u\in U}\) are independent, mean-zero random variables, and
\[
|Y_u|\le M_\rmp,
\qquad
\Var(Y_u)\le \sp a_u^2\le \sp M_\rmp^2.
\]
Hence
\[
\sum_{u\in U}\Var(Y_u)\le \sp m M_\rmp^2.
\]
Set
\[
t:=\frac12\,\log n\,\sqrt{\sp m}.
\]
By Bernstein's inequality,
\[
\Pr\!\left(
\left|\sum_{u\in U}Y_u\right|>t
\ \middle|\ X_U=x_U,\ X_v=x_v
\right)
\le
2\exp\!\left(
-\frac{t^2}{2\bigl(\sp m M_\rmp^2+M_\rmp t/3\bigr)}
\right).
\]
Since \(\sp m\ge \log^2 n\), we have
\[
t=\frac12\,\log n\,\sqrt{\sp m}\le \frac12\,\sp m.
\]
Using that \(M_\rmp\) is an absolute constant, it follows that
\[
\sp m M_\rmp^2+M_\rmp t/3 = O(\sp m),
\]
and therefore
\[
\Pr\!\left(
\left|\sum_{u\in U}Y_u\right|>t
\ \middle|\ X_U=x_U,\ X_v=x_v
\right)
\le
2\exp(-c\log^2 n)
=
n^{-\omega(1)}
\]
for some absolute constant \(c>0\). Since
\[
T_1=\frac{1}{\sp m}\sum_{u\in U}Y_u,
\]
this yields
\[
\Pr\!\left(
|T_1|>\frac{\log n}{2\sqrt{\sp m}}
\ \middle|\ X_U=x_U,\ X_v=x_v
\right)
\le n^{-\omega(1)}.
\]

\medskip
\noindent
\emph{Step 2: control of \(T_2\).}
Let
\[
N:=\sum_{u\in U} B_u.
\]
By Chernoff's inequality,
\[
\Pr\!\left(
N>2\sp m
\ \middle|\ X_U=x_U,\ X_v=x_v
\right)
\le
\exp(-c\sp m)
\le
\exp(-c\log^2 n)
=
n^{-\omega(1)}.
\]
Let
\[
E_B:=\{N\le 2\sp m\}.
\]

Now condition additionally on the realization of \(\{B_u\}_{u\in U}\). On the
event \(E_B\), if
\[
S:=\{u\in U:B_u=1\},
\]
then \(|S|=N\le 2\sp m\), and
\[
\sum_{u\in U} B_u\xi_u = \sum_{u\in S}\xi_u.
\]
Since the variables \(\{\xi_u\}_{u\in S}\) are independent and \(\Ksg\)-subgaussian,
their sum is \(\Ksg\sqrt{|S|}\)-subgaussian. Hence, on \(E_B\),
\[
\Pr\!\left(
\left|\sum_{u\in U} B_u\xi_u\right|>t
\ \middle|\ X_U=x_U,\ X_v=x_v,\ \{B_u\}_{u\in U}
\right)
\le
2\exp\!\left(-\frac{t^2}{2\Ksg^2|S|}\right)
\le
2\exp(-c'\log^2 n)
\]
for some absolute constant \(c'>0\), since
\[
t^2=\frac14\,\log^2 n\,\sp m
\qquad\text{and}\qquad
|S|\le 2\sp m.
\]
Averaging over \(\{B_u\}_{u\in U}\), we conclude that
\[
\Pr\!\left(
\left|\sum_{u\in U} B_u\xi_u\right|>t
\ \middle|\ X_U=x_U,\ X_v=x_v
\right)
\le
\Pr(E_B^c\mid X_U=x_U,\ X_v=x_v)+2e^{-c'\log^2 n}
=
n^{-\omega(1)}.
\]
Since
\[
T_2=\frac{1}{\sp m}\sum_{u\in U} B_u\xi_u,
\]
this gives
\[
\Pr\!\left(
|T_2|>\frac{\log n}{2\sqrt{\sp m}}
\ \middle|\ X_U=x_U,\ X_v=x_v
\right)
\le n^{-\omega(1)}.
\]

\medskip
\noindent
\emph{Step 3: conclusion.}
By the triangle inequality,
\[
\left|\cn{U}{v}-\acn{U}{v}\right|
\le |T_1|+|T_2|.
\]
Hence, by the union bound,
\[
\Pr\!\left(
\left|\cn{U}{v}-\acn{U}{v}\right|
>
\frac{\log n}{\sqrt{\sp\,m}}
\ \middle|\ X_U=x_U,\ X_v=x_v
\right)
\le n^{-\omega(1)}.
\]
This proves the claim.
\end{proof}

\begin{proof}[Proof of Lemma~\ref{lem:Evt-all-prob}]
Unwrapping the definitions of \(\Evt{base}\), \(\Evt{R}\), and \(\Evt{ray}\), we
may equivalently write
\begin{align}
\Evt{all}
=&
\Evt{cluster}
\cap
\Ept{\bfV_0}
\cap
\Ept{\bfV_2}
\cap
\Evt{locRing}
\cap
\bigcap_{v\in\bfV_0}
\Enavi{B_v}{\bfV_1\sqcup \bfV_2\sqcup \bfU_1 \sqcup \bfU_2}
\nonumber\\
&\cap
\bigcap_{v\in\bfV_1}
\Enavi{{\cal B}_v^{(2)}}{\bfU_1\sqcup\bfU_2}
\cap
\bigcap_{(u,v)\in\meso}\ \bigcap_{w\in\bfV_1}
\Enavi{\ring{u,v}}{w}.
\label{eq:Evt-all-expanded}
\end{align}

We refactor the event \(\Evt{all}\) into the following stages:
\begin{align*}
\Evt{all}^{(0)}
:=&
\Evt{cluster}\cap \Ept{\bfV_0},
\\
\Evt{all}^{(1)}
:=&
\Ept{\bfV_2}\cap \Evt{locRing},
\\
\Evt{all}^{(2)}
:=&
\bigcap_{v\in\bfV_0}
\Enavi{B_v}{\bfV_1\sqcup \bfV_2\sqcup \bfU_1 \sqcup \bfU_2},
\\
\Evt{all}^{(3)}
:=&
\bigcap_{v\in\bfV_1}
\Enavi{{\cal B}_v^{(2)}}{\bfU_1\sqcup\bfU_2},
\\
\Evt{all}^{(4)}
:=&
\bigcap_{(u,v)\in\meso}\ \bigcap_{w\in\bfV_1}
\Enavi{\ring{u,v}}{w}.
\end{align*}
Then
\[
\Evt{all}=\bigcap_{j=0}^4 \Evt{all}^{(j)}.
\]

We reveal the randomness in the following order:
\begin{align*}
\mathcal F_0&:=\sigma\!\big(X_{\bfV_0},\,\mathcal U_{\bfV_0,\bfV_0},\,B_{\bfV_0,\bfV_0}\big),\\
\mathcal F_1&:=\sigma\!\big(\mathcal F_0,\,X_{\bfV}\big),\\
\mathcal F_2&:=\sigma\!\big(\mathcal F_1,\,\mathcal U_{\bfV_0,\bfV\setminus \bfV_0},\,B_{\bfV_0,\bfV\setminus \bfV_0}\big),\\
\mathcal F_3&:=\sigma\!\big(\mathcal F_2,\,\mathcal U_{\bfV_2,\bfU_1\sqcup\bfU_2},\,B_{\bfV_2,\bfU_1\sqcup\bfU_2}\big),\\
\mathcal F_4&:=\sigma\!\big(\mathcal F_3,\,\mathcal U_{\bfV_1,\bfU_2},\,B_{\bfV_1,\bfU_2}\big).
\end{align*}
We verify, in this order, that \(\Evt{all}^{(j)}\) is \(\mathcal F_j\)-measurable
and has conditional probability at least \(1-n^{-\omega(1)}\) given the
success of the preceding stages.

\step{Step 0: the induced subgraph on \(\bfV_0\)}
By Definition~\ref{def:graph_model}, the induced subgraph on \(\bfV_0\) is
exactly \(Z_{\bfV_0,\bfV_0}\), so \(\Evt{all}^{(0)}\) is \(\mathcal F_0\)-measurable.
Moreover, on \(\Evt{all}^{(0)}\), all clusters \(B_v\) are determined. By the
success guarantee of the cluster generating algorithm and
Lemma~\ref{lem:Ept},
\[
\Pr\!\big(\Evt{all}^{(0)}\big)\ge 1-p_0-n^{-\omega(1)}.
\]

\step{Step 1: the latent-location event}
This is simply the union bound, relying on the individual estimates in 
in Lemma~\ref{lem:Ept} and Lemma~\ref{lem:U1-moving-ring-occupancy}, we have
\[
\Pr\!\big(\Evt{all}^{(1)} \,\big|\, \Evt{all}^{(0)}\big)\ge 1-n^{-\omega(1)}.
\]

\step{Step 2: concentration from \(B_v\) to the other blocks}
On \(\Evt{all}^{(0)}\), all clusters \(B_v\subseteq \bfV_0\) are fixed, so
\(\Evt{all}^{(2)}\) is \(\mathcal F_2\)-measurable. Moreover, on
\(\Evt{cluster}\subseteq \Evt{all}^{(0)}\), each such cluster satisfies
\[
|B_v|\ge \lambda_0 c_\mu r_0^d n.
\]
Hence
\[
\sp |B_v|
\ge
\lambda_0 c_\mu r_0^d\,\sp n
\gg
\log^2 n,
\]
by the standing sparsity assumption \(\sp n \ge (\log n)^{\log\log n}\).
Conditional on \(\mathcal F_1\), the only additional randomness entering
\(\Evt{all}^{(2)}\) is
\[
\big(\mathcal U_{\bfV_0,\bfV\setminus \bfV_0},\,B_{\bfV_0,\bfV\setminus \bfV_0}\big),
\]
which is independent of \(\mathcal F_1\). Therefore Lemma~\ref{_lem: navi},
applied to each pair \((B_v,\bfV_1\sqcup \bfV_2\sqcup \bfU_1 \sqcup \bfU_2)\),
and a union bound over \(v\in\bfV_0\), give
\[
\Pr\!\big(\Evt{all}^{(2)} \,\big|\, \mathcal F_1\big)\ge 1-n^{-\omega(1)}
\qquad\text{on } \Evt{all}^{(0)}.
\]
Consequently,
\[
\Pr\!\big(\Evt{all}^{(2)} \,\big|\, \Evt{all}^{(0)}\cap \Evt{all}^{(1)}\big)
\ge 1-n^{-\omega(1)}.
\]

\step{Step 3: concentration from \(\mathcal B_v^{(2)}\) to \(\bfU_1\sqcup\bfU_2\)}
Once \(\mathcal F_2\) is revealed, all quantities \(\Delta_{\rm cn}(v,v')\) with
\(v\in\bfV_1\) and \(v'\in\bfV_2\) are determined, so all proxy sets
\(\mathcal B_v^{(2)}\), all rings \(\ring{u}\), and the set \(\meso\) are
\(\mathcal F_2\)-measurable. Hence \(\Evt{all}^{(3)}\) is
\(\mathcal F_3\)-measurable.

On \(\Evt{all}^{(0)}\cap \Evt{all}^{(1)}\cap \Evt{all}^{(2)}\), we have
\(\Evt{base}\cap \Ept{\bfV_2}\), so Corollary~\ref{cor:mesoscopic-pairs-v1}
implies
\[
|\mathcal B_v^{(2)}|\gtrsim (\delta_n^3\xiq)^d n,
\qquad
\forall v\in\bfV_1.
\]
Therefore
\[
\sp |\mathcal B_v^{(2)}|
\gtrsim
\sp n\,(\delta_n^3\xiq)^d
=
(\sp n)^{\frac{5}{d+5}}(\log n)^{d(C_{\tref{def:xi-q}}-3)}
\gg
\log^2 n,
\]
where we used \({\frak q}=5\), the definition of \(\xiq\), and again
\(\sp n \ge (\log n)^{\log\log n}\).
Conditional on \(\mathcal F_2\), the only new randomness needed for
\(\Evt{all}^{(3)}\) is
\[
\big(\mathcal U_{\bfV_2,\bfU_1\sqcup\bfU_2},\,B_{\bfV_2,\bfU_1\sqcup\bfU_2}\big),
\]
which is independent of \(\mathcal F_2\). Applying Lemma~\ref{_lem: navi} to
each pair \((\mathcal B_v^{(2)},\bfU_1\sqcup\bfU_2)\) and taking a union bound
over \(v\in\bfV_1\), we obtain
\[
\Pr\!\big(\Evt{all}^{(3)} \,\big|\, \mathcal F_2\big)\ge 1-n^{-\omega(1)}
\qquad\text{on } \Evt{all}^{(0)}\cap \Evt{all}^{(1)}\cap \Evt{all}^{(2)}.
\]
Hence
\[
\Pr\!\big(\Evt{all}^{(3)} \,\big|\, \Evt{all}^{(0)}\cap \Evt{all}^{(1)}\cap \Evt{all}^{(2)}\big)
\ge 1-n^{-\omega(1)}.
\]

\step{Step 4: concentration from \(\ring{u,v}\) to \(\bfV_1\)}
Once \(\mathcal F_3\) is revealed, all screened ring sets \(\ring{u,v}\) are
determined, because their definition uses only the already constructed sets
\(\ring{u}\), \(\mathcal B_v^{(2)}\), and the edge variables in
\(
Z_{\bfV_2,\bfU_1\sqcup\bfU_2}.
\)
Thus \(\Evt{all}^{(4)}\) is \(\mathcal F_4\)-measurable.

On \(\Evt{all}^{(0)}\cap \Evt{all}^{(1)}\cap \Evt{all}^{(2)}\cap \Evt{all}^{(3)}\),
Lemma~\ref{lem:angle-screened-subset-coverage} gives
\[
|\ring{u,v}|\ge n\,\delta_n^{5d}\xiq^{d+1},
\qquad
\forall (u,v)\in\meso.
\]
Consequently,
\[
\sp |\ring{u,v}|
\gtrsim
\sp n\,\delta_n^{5d}\xiq^{d+1}
=
(\sp n)^{\frac{4}{d+5}}(\log n)^{C_{\tref{def:xi-q}}(d+1)-5d}
\gg
\log^2 n,
\]
again by \({\frak q}=5\), the definition of \(\xiq\), and the standing
sparsity assumption \(\sp n \ge (\log n)^{\log\log n}\).
Moreover, by Remark~\ref{rem:edge-revelation}, no edges between \(\bfV_1\) and
\(\bfU_2\) are revealed in the construction of the sets
\(\mathcal B_v^{(2)}\), \(\ring{u}\), and \(\ring{u,v}\). Therefore,
conditional on \(\mathcal F_3\), the only new randomness entering
\(\Evt{all}^{(4)}\) is
\[
\big(\mathcal U_{\bfV_1,\bfU_2},\,B_{\bfV_1,\bfU_2}\big),
\]
which is independent of \(\mathcal F_3\). Applying Lemma~\ref{_lem: navi} to
each pair \((\ring{u,v},\{w\})\) with \((u,v)\in\meso\) and \(w\in\bfV_1\), and
taking a union bound, we get
\[
\Pr\!\big(\Evt{all}^{(4)} \,\big|\, \mathcal F_3\big)\ge 1-n^{-\omega(1)}
\]
on \(\Evt{all}^{(0)}\cap \Evt{all}^{(1)}\cap \Evt{all}^{(2)}\cap \Evt{all}^{(3)}\).
Hence
\[
\Pr\!\big(\Evt{all}^{(4)} \,\big|\, \Evt{all}^{(0)}\cap \Evt{all}^{(1)}\cap \Evt{all}^{(2)}\cap \Evt{all}^{(3)}\big)
\ge 1-n^{-\omega(1)}.
\]

Combining the above bounds by the chain rule for conditional probabilities,
\begin{align*}
\Pr(\Evt{all})
\ge&
\Pr(\Evt{all}^{(0)})
\prod_{j=1}^4
\Pr\!\Big(
\Evt{all}^{(j)}
\ \Big|\
\bigcap_{i<j}\Evt{all}^{(i)}
\Big) \\
\ge&
\bigl(1-p_0-n^{-\omega(1)}\bigr)\bigl(1-n^{-\omega(1)}\bigr)^4
\ge
1-p_0-n^{-\omega(1)}.
\end{align*}
This proves the lemma.
\end{proof}

\section{Deferred proofs of auxiliary results regarding distance recovery from local estimates}  

\begin{proof}[Proof of Lemma~\ref{lem:subset-local-to-global-local}]
Let
\[
m:=\Big\lfloor \frac{\lambda n}{2}\Big\rfloor.
\]
Since \(0<\lambda\le 1\) and \(\lambda n\ge 2\), we have \(1\le m\le n/2\), hence \(m<n\).
Choose a partition
\[
\mathcal X=P_1\sqcup \cdots \sqcup P_r,
\qquad
r=\left\lceil \frac{n}{m}\right\rceil
=
\left\lceil \frac{n}{\lfloor \lambda n/2\rfloor}\right\rceil,
\]
into nonempty blocks with \(|P_i|\le m\) for every \(i\). In particular \(r\ge 2\).

For each \(1\le i<j\le r\), set
\[
\mathcal X_{ij}:=P_i\sqcup P_j.
\]
Then
\[
|\mathcal X_{ij}|\le 2m\le \lambda n,
\]
so the hypothesis applies to every \(\mathcal X_{ij}\). Run \(A\) on each \(\mathcal X_{ij}\), and write the output as
\[
\mathcal P^{ij}\subseteq \mathcal X_{ij}\times \mathcal X_{ij},
\qquad
\hat d^{\,ij}:\mathcal P^{ij}\to \mathbb R.
\]

Define
\[
\mathcal P:=\bigcup_{1\le i<j\le r}\mathcal P^{ij}.
\]
This set is symmetric. For each \((x,y)\in \mathcal P\), let \(\iota(x,y)\) be the lexicographically first pair \((i,j)\) such that \((x,y)\in \mathcal P^{ij}\), and set
\[
\hat d(x,y):=\hat d^{\,\iota(x,y)}(x,y).
\]
This defines the output of \(B\).

Let \(E_{ij}\) be the event that the run of \(A\) on \(\mathcal X_{ij}\) succeeds. Then by the union bound,
\[
\Pr\Big(\bigcap_{1\le i<j\le r}E_{ij}\Big)\ge 1-\binom{r}{2}p_0 \ge 1- 2\binom{\lceil 2/\lambda\rceil}{2}p_0\,.
\]

Assume \(E\). If \((x,y)\in \mathcal P\), then \((x,y)\in \mathcal P^{\,\iota(x,y)}\), and the success of that run gives
\[
|\hat d(x,y)-\rho(x,y)|\le \varepsilon.
\]

Now let \(x,y\in \mathcal X\) satisfy \(\rho(x,y)\le \rho_0\). If \(x\in P_i\) and \(y\in P_j\) with \(i\neq j\), then \(x,y\in \mathcal X_{ij}\), so the success of the \(ij\)-th run yields \((x,y)\in \mathcal P^{ij}\subseteq \mathcal P\). If \(x,y\in P_i\), choose any \(j\neq i\); this is possible because \(r\ge 2\). Then \(x,y\in \mathcal X_{ij}\), and again \((x,y)\in \mathcal P^{ij}\subseteq \mathcal P\). Thus every pair at distance at most \(\rho_0\) belongs to \(\mathcal P\).

\end{proof}

\begin{proof}[Proof of Lemma~\ref{lem:metric-extension-local-estimates}]
For the lower bound, let
\[
x=x_0,\ x_1,\ \dots,\ x_m=y
\]
be any path in \(G_{\mathcal P}\). Since
\[
\widehat\rho(x_i,x_{i+1})+\varepsilon \ge \rho(x_i,x_{i+1}),
\]
the triangle inequality gives
\[
\sum_{i=0}^{m-1}\bigl(\widehat\rho(x_i,x_{i+1})+\varepsilon\bigr)
\ge
\sum_{i=0}^{m-1}\rho(x_i,x_{i+1})
\ge
\rho(x,y).
\]
Taking the infimum over all paths yields
\[
\rho_{\rm sp}(x,y)\ge \rho(x,y).
\]

For the upper bound, first consider the case \(\rho(x,y)\le r\). Then
\((x,y)\in\mathcal P\), so the direct edge gives
\[
\rho_{\rm sp}(x,y)
\le
\widehat\rho(x,y)+\varepsilon
\le
\rho(x,y)+2\varepsilon.
\]
Since \(r\le {\rm diam}(\mathcal X)\), this implies
\[
\rho_{\rm sp}(x,y)
\le
\rho(x,y)+2\frac{{\rm diam}(\mathcal X)}{r}\,(\eta+\varepsilon).
\]

Now assume \(\rho(x,y)>r\), and let \(p_0,\dots,p_k\) be the chain from the
assumption. Since each \(\rho(p_i,p_{i+1})\le r\), every edge
\((p_i,p_{i+1})\) belongs to \(\mathcal P\). Therefore
\begin{align*}
\rho_{\rm sp}(x,y)
\le&
\sum_{i=0}^{k-1}\bigl(\widehat\rho(p_i,p_{i+1})+\varepsilon\bigr) \\
\le&
\sum_{i=0}^{k-1}\rho(p_i,p_{i+1})+2k\varepsilon \\
\le&
\rho(x,y)+k\eta+2k\varepsilon.
\end{align*}
Using \(k\le C\rho(x,y)/r\), we obtain
\[
\rho_{\rm sp}(x,y)
\le
\rho(x,y)+C\frac{\rho(x,y)}{r}\,(\eta+2\varepsilon).
\]
Since \(\rho(x,y)\le {\rm diam}(\mathcal X)\),
\[
\rho_{\rm sp}(x,y)
\le
\rho(x,y)+3C\frac{{\rm diam}(\mathcal X)}{r}\,(\eta+\varepsilon).
\]
Combining the two cases proves the lemma.
\end{proof}

\section{Geometric estimates}

\label{app:deferred-geom-proofs}

\begin{proof}[Proof of Lemma~\ref{lem: M-right-angle}]
\step{Case $c \simeq a\sin(\theta)$}
We first bound $c$ from above by $c_{-\kappa}$, the opposite side length in $M_{-\kappa}$, and then bound $c_{-\kappa}$ by $\theta$. This bound relies on Lemma \ref{_lem: tri_lem}. Instead of using the hyperbolic law of cosines, we use the hyperbolic law of sines  
$$
\frac{\sin\alpha}{\sinh \sqrt{\kappa} a} = \frac{\sin\beta}{\sinh \sqrt{\kappa} b} = \frac{\sin\theta}{\sinh \sqrt{\kappa} c_{-\kappa}}.
$$
Using the fact that $\alpha=\pi/2$, we get  
\begin{align*}
    \sin(\theta) = \frac{\sinh \sqrt{\kappa} c_{-\kappa}}{\sinh \sqrt{\kappa} a}\,. 
\end{align*}
Since $c_{-\kappa} \le a+c \le 2\rG$, we are in a regime that one can approxiamte $\sinh$ by a linear function, 
\begin{align*}
    c \le c_{-\kappa} \le  
    Ca \sin(\theta) \,,
\end{align*}
for some universal constant $C$.

The lower bound of $c$ follows the same approach, but using $c_\kappa$ and the spherical law of sines: 
$\tfrac{\sin\alpha}{\sin \sqrt{\kappa} a} = \tfrac{\sin\beta}{\sin \sqrt{\kappa} b} = \tfrac{\sin\theta}{\sin \sqrt{\kappa} c_\kappa}.$
This gives 
$$
    c \ge c_\kappa \qquad \text{and}\qquad
     \sin(\theta) = \frac{\sin \sqrt{\kappa} c_\kappa}{\sin \sqrt{\kappa} a},
$$
which leads to 
\begin{align*}
    c \ge c_\kappa \ge  
    \frac{a}{C} \sin(\theta) \,,
\end{align*}
for some universal constant $C$.

\step{Case $b \simeq a\cos(\theta)$}
Let us first state the hyperbolic law of cosines for a triangle with sides $a,b,c_{-\kappa}$ with $\alpha$ being the angle opposite to side $a$ and $\theta$ being the angle opposite to side $c_{-\kappa}$:
\[
\cosh \sqrt{\kappa} a=\cosh \sqrt{\kappa} b\,\cosh \sqrt{\kappa} c_{-\kappa}-\sinh \sqrt{\kappa} b\,\sinh \sqrt{\kappa} c_{-\kappa}\,\cos \alpha
\]
and 
\[
\cosh \sqrt{\kappa} c_{-\kappa}=\cosh \sqrt{\kappa} a\,\cosh \sqrt{\kappa} b-\sinh \sqrt{\kappa} a\,\sinh \sqrt{\kappa} b\,\cos \theta.
\]

If $\alpha=\frac{\pi}{2}$, then \(\cos\alpha=0\) and the first equation becomes:
\[
\cosh \sqrt{\kappa} a=\cosh \sqrt{\kappa} b\,\cosh \sqrt{\kappa} c_{-\kappa}.
\]

Substitute into the second equation:
\begin{align*}
 \cosh \sqrt{\kappa} c_{-\kappa}
 =& (\cosh \sqrt{\kappa} b\,\cosh \sqrt{\kappa} c_{-\kappa})\cosh \sqrt{\kappa} b-\sinh \sqrt{\kappa} a\,\sinh \sqrt{\kappa} b\,\cos \theta\\
      =& \cosh^2 \sqrt{\kappa} b\,\cosh \sqrt{\kappa} c_{-\kappa}-\sinh \sqrt{\kappa} a\,\sinh \sqrt{\kappa} b\,\cos \theta.
\end{align*}
Then, rearranging gives 
\[
\sinh \sqrt{\kappa} a\,\sinh \sqrt{\kappa} b\,\cos \theta=\cosh \sqrt{\kappa} c_{-\kappa}(\cosh^2 \sqrt{\kappa} b-1)=\cosh \sqrt{\kappa} c_{-\kappa}\,\sinh^2 \sqrt{\kappa} b.
\]
Therefore, 
\[
\cos \theta=\frac{\cosh \sqrt{\kappa} c_{-\kappa}\,\sinh \sqrt{\kappa} b}{\sinh \sqrt{\kappa} a} \le \frac{\cosh \sqrt{\kappa} c\, \sinh \sqrt{\kappa} b}{\sinh \sqrt{\kappa} a},
\]
where the last inequality follows from $c \le c_{-\kappa}$ and the fact that $\cosh$ is increasing on $[0,\infty)$. 

Examining the spherical law of cosines with the same substitution trick, we obtain 
$$
    \cos \theta = \frac{\cos \sqrt{\kappa} c_\kappa\, \sin \sqrt{\kappa} b}{\sin \sqrt{\kappa} a} \ge \frac{\cos \sqrt{\kappa} c\, \sin \sqrt{\kappa} b}{\sin \sqrt{\kappa} a},
$$
and the last inequality follows from $c \ge c_\kappa$ and the fact that $\cos$ is decreasing on $[0,\pi]$ together with the fact that our assumption on $a,b,c$ implies $\sin(\sqrt{\kappa} a), \sin(\sqrt{\kappa} b)$ are all positive.

Again, since $a,b,c\le \rG$, we are in a regime that one can approximate $\sinh$ and $\sin$ by linear functions, and the $\cos$ terms are bounded by constants, so we conclude the desired bound. 

\step{Estimate of $a-b$}
For the lower bound, let
\[
a_\kappa:=\os^\kappa(\pi/2;b,c),
\]
with the convention that \(a_\kappa=\sqrt{b^2+c^2}\) when \(\kappa=0\). Since
\(\alpha=\pi/2\), Lemma~\ref{_lem: tri_lem} gives
\[
a_\kappa\le a.
\]
If \(\kappa=0\), then
\[
a-b
=
\frac{c^2}{a+b}.
\]
Using \(a\le b+c\), we obtain
\[
a-b
\ge
\frac{c^2}{2(b+c)}
\ge
\frac18 \min\!\left\{\frac{c^2}{b},\,c\right\}.
\]
So it remains to consider the case \(\kappa>0\). Write
\[
\tilde a_\kappa:=\sqrt{\kappa}\,a_\kappa,\qquad
\tilde b:=\sqrt{\kappa}\,b,\qquad
\tilde c:=\sqrt{\kappa}\,c.
\]
Then \(\tilde b,\tilde c\le 1/16\), and by triangle inequality also
\[
\tilde a_\kappa\le \tilde b+\tilde c\le \frac18.
\]
The spherical law of cosines for the comparison triangle gives
\[
\cos\tilde a_\kappa=\cos\tilde b\,\cos\tilde c.
\]
Hence
\[
\cos\tilde b-\cos\tilde a_\kappa
=
\cos\tilde b\,(1-\cos\tilde c).
\]
By the mean value theorem, for some \(\xi\in(\tilde b,\tilde a_\kappa)\),
\[
\sin\xi\,(\tilde a_\kappa-\tilde b)
=
\cos\tilde b-\cos\tilde a_\kappa.
\]
Since \(\tilde b\le 1/16\), we have
\[
\cos\tilde b\ge 1-\frac{\tilde b^2}{2}\ge \frac34,
\]
and since \(\tilde c\le 1/16\), Taylor's theorem yields
\[
1-\cos\tilde c
\ge
\frac{\tilde c^2}{2}-\frac{\tilde c^4}{24}
\ge
\frac13\,\tilde c^2.
\]
Also,
\[
\sin\xi\le \xi\le \tilde a_\kappa\le \tilde b+\tilde c.
\]
Therefore
\[
\tilde a_\kappa-\tilde b
\ge
\frac14\,\frac{\tilde c^2}{\tilde b+\tilde c}.
\]
Rescaling back, we obtain
\[
a_\kappa-b
\ge
\frac14\,\frac{c^2}{b+c}
\ge
\frac18 \min\!\left\{\frac{c^2}{b},\,c\right\}.
\]
Since \(a\ge a_\kappa\), this proves
\[
a-b\ge \frac18 \min\!\left\{\frac{c^2}{b},\,c\right\}.
\]

Now assume \(c\le b/4\). Apply Lemma~\ref{_lem: M-opposite-side} to the same
triangle, but with the relabeling
\[
(a,b,c,\theta)\mapsto (b,c,a,\alpha).
\]
Since \(\alpha=\pi/2\), the lemma yields
\[
|b-a-c\cos\alpha|
\le
4\frac{c^2}{b}.
\]
Because \(a\ge b\) and \(\cos\alpha=0\), this gives
\[
a-b\le 4\frac{c^2}{b}.
\]

\end{proof}

\subsection{Geodesics and Fermi coordinates}
\label{subsec: geodesics-fermi-coords}

\begin{proof}[Sketch of proof of Lemma~\ref{lem:distance-to-nearby-geodesic}]
Since \(\gamma((a,b))\subset B_M(p,\tfrac12\rM)\subset B_M(p,\rinj(M))\) and
\(p\notin \gamma((a,b))\), the distance function \(r_p(x):=\D{p}{x}\) is smooth
along \(\gamma((a,b))\), and
\[
\nabla r_p(x)=-\frac{\exp_x^{-1}(p)}{\|\exp_x^{-1}(p)\|}.
\]
Hence
\[
f'(t)=\langle \nabla r_p(\gamma(t)),\dot\gamma(t)\rangle
=
-\Big\langle
\frac{\exp_{\gamma(t)}^{-1}(p)}{\|\exp_{\gamma(t)}^{-1}(p)\|},
\,\dot\gamma(t)
\Big\rangle.
\]
Since \(\gamma\) is a geodesic,
\[
f''(t)=\operatorname{Hess}r_p(\gamma(t))(\dot\gamma(t),\dot\gamma(t)).
\]
The standard Hessian comparison (see \cite[Theorem 11.7]{Lee18})
implies
\[
\operatorname{Hess}r_p(x)(v,v) \ge \sqrt{\kappa} \cot(\sqrt{\kappa}r_p(x))\|v^\perp\|^2,
\]
where \(v^\perp\) is the component of \(v\) orthogonal to \(\nabla r_p(x)\).
Thus, with $x=\gamma(t)$, we have $\cot(\sqrt{\kappa}r_p(x)) > \cot(\pi/2)=0$, so 
$f''(t) \ge 0$ unless $\dot\gamma(t) \parallel \nabla r_p(\gamma(t))$, i.e. $\dot\gamma(t)$ is radial with respect to $p$.
If $f''(t)=0$, then by uniqueness of minimizing geodesics inside the
injectivity ball this would force \(\gamma\) to pass through \(p\), a
contradiction. Thus \(f''(t)>0\) on \((a,b)\), so \(f\) is strictly convex.

Finally, maximality of \((a,b)\) implies
\[
f(t)\to \tfrac12\rM
\quad\text{as }t\to a^+\text{ or }t\to b^-,
\]
while \(f(0)=\D{p}{q}<\tfrac12\rM\). Hence \(f\) attains a unique minimum at an
interior point \(t_\star\), which is the unique critical point. The
orthogonality characterization follows from the formula for \(f'(t)\).
\end{proof}

\begin{proof}[Proof of Corollary~\ref{cor:local-proj-geodesic}]
Fix \(x\in B_M(q,\rM/4)\). Since \(\gamma([-\rM/2,\rM/2])\) is compact, the
function
\[
s\mapsto \D{x}{\gamma(s)}
\]
attains its minimum on \([-\rM/2,\rM/2]\).

If \(x\in \gamma([-\rM/2,\rM/2])\), then the unique minimizer is
\(\pi_\gamma(x)=x\), and there is nothing to prove. So assume
\[
x\notin \gamma([-\rM/2,\rM/2]).
\]

Let \((a,b)\) be the maximal open interval containing \(0\) such that
\[
\gamma((a,b))\subset B_M(x,\rM/2).
\]
Since \(\D{x}{q}<\rM/4\), this interval is nonempty. Moreover,
\[
\D{x}{\gamma(0)}=\D{x}{q}<\rM/4,
\]
whereas for any \(s\in[-\rM/2,\rM/2]\setminus (a,b)\) one has
\[
\D{x}{\gamma(s)}\ge \rM/2.
\]
Hence any global minimizer over \([-\rM/2,\rM/2]\) must lie in \((a,b)\).

Now apply Lemma~\ref{lem:distance-to-nearby-geodesic} with \(p=x\). It implies
that
\[
f(s):=\D{x}{\gamma(s)}
\]
is strictly convex on \((a,b)\). Therefore \(f\) has at most one critical
point, hence at most one minimizer on \((a,b)\). Since a minimizer exists, it is
unique. This proves the existence and uniqueness of \(\pi_\gamma(x)\).

Finally, by Lemma~\ref{lem:distance-to-nearby-geodesic}, the unique minimizer is
characterized by the orthogonality condition: the minimizing geodesic from
\(\pi_\gamma(x)\) to \(x\) meets \(\gamma\) orthogonally at \(\pi_\gamma(x)\).
\end{proof}

\begin{proof}[Proof of Lemma~\ref{lem:two-orthogonal-perturbations}]
Let
\[
a:=\D{p}{q},
\qquad
b_p:=\D{p}{p'},
\qquad
b_q:=\D{q}{q'},
\qquad
a':=\D{p'}{q}.
\]
Since \(p\) is a closest point of \(p'\) on \(\gamma\), the geodesic \(pp'\) is
orthogonal to \(\gamma\) at \(p\), so
\[
\ang{}{p}{p'}{q}=\pi/2.
\]
Apply Lemma~\ref{_lem: M-opposite-side} to triangle \((p,p',q)\):
\[
\big|a-a'-b_p\cos(\ang{}{p}{p'}{q})\big|
\le
4\frac{b_p^2}{a}.
\]
Hence
\begin{equation}
\label{eq:aprime_minus_a_two_orth}
|a'-a|
\le
4\frac{b_p^2}{a}.
\end{equation}
In particular, since \(b_p\le a/8\),
\[
\frac{7}{8}a\le a'\le \frac{9}{8}a.
\]

Now consider triangle \((q,p',q')\) and set
\[
\theta:=\ang{}{q}{p'}{q'}.
\]
Apply 
Lemma~\ref{_lem: M-opposite-side} again with base point \(q\), \(x=p'\),
and third vertex \(q'\):
\begin{equation}
\label{eq:uv_vs_aprime_two_orth}
\big|a'-\D{p'}{q'}-b_q\cos\theta\big|
\le
4\frac{b_q^2}{a'}.
\end{equation}

It remains to bound \(|\cos\theta|\). Since \(q\) is a closest point of \(q'\)
on \(\gamma\), the geodesic \(qq'\) is orthogonal to \(\gamma\) at \(q\), hence
\[
\ang{}{q}{q'}{p}=\pi/2.
\]
Let
\[
\phi:=\ang{}{q}{p'}{p}.
\]
Then $|\theta-\pi/2|\le \phi$, so
\[
|\cos\theta|\le \sin\phi\le \phi.
\]
In triangle \((q,p,p')\), the opposite side to angle \(\phi\) is
\(\D{p}{p'}=b_p\), and the adjacent side is \(\D{q}{p}=a\). By
Lemma~\ref{_lem: tri_lem} and Lemma~\ref{_lem:spherical-angle},
\(\D{p}{p'}=b_p\), and the adjacent side is \(\D{q}{p}=a\). By
Lemma~\ref{_lem: tri_lem} and Lemma~\ref{_lem:spherical-angle},
\[
\phi\le C_1\frac{b_p}{a},
\]
for an absolute constant $C_1$. Therefore
\[
|\cos\theta|\le C_1\frac{b_p}{a}.
\]
Plugging this into \eqref{eq:uv_vs_aprime_two_orth} and using \(a'\asymp a\)
gives
\[
|\D{p'}{q'}-a'|
\le
b_q|\cos\theta|+4\frac{b_q^2}{a'}
\le
C\left(\frac{b_pb_q}{a}+\frac{b_q^2}{a}\right)
\le
C\frac{(b_p+b_q)^2}{a}.
\]
Combining with \eqref{eq:aprime_minus_a_two_orth} yields
\[
|\D{p'}{q'}-a|
\le
C\frac{(b_p+b_q)^2}{a}.
\]
\end{proof}

An immediate consequence of Corollary~\ref{cor:local-proj-geodesic} is that the Fermi map \(\Phi_\gamma\) is a local diffeomorphism from \(I\times B_{\mathbb R^{d-1}}(0,\rM/4)\) to its image in \(M\), when $|I| \le \rM/2$. 

\begin{lemma}[Pointwise metric comparison in Fermi coordinates]
\label{lem:fermi-metric-comparison}
For every \(\varepsilon\in(0,1)\), there exists a constant
depending only on $\varepsilon$, \(c_{\rm fm}(\varepsilon)\in(0,1)\) such that the following holds.

Let \(\gamma:J\to M\) be a unit-speed geodesic, and let \(\Phi_\gamma\) be an
associated Fermi map. Then for every \(t\in J\), every
\(y\in\mathbb R^{d-1}\) with
\[
\|y\|_2\le c_{\rm fm}(\varepsilon)\rM,
\]
and every vector
\[
v=(a,b)\in T_{(t,y)}(J\times\mathbb R^{d-1})
\cong \mathbb R\times\mathbb R^{d-1},
\]
one has
\[
(1-\varepsilon)\bigl(a^2+\|b\|_2^2\bigr)
\le
(\Phi_\gamma^*g)_{(t,y)}(v,v)
\le
(1+\varepsilon)\bigl(a^2+\|b\|_2^2\bigr).
\]
Equivalently,
\[
(1-\varepsilon)\bigl(a^2+\|b\|_2^2\bigr)
\le
\bigl\|D\Phi_\gamma(t,y)(a,b)\bigr\|_g^2
\le
(1+\varepsilon)\bigl(a^2+\|b\|_2^2\bigr).
\]
\end{lemma}

\begin{proof}
Fix \(\varepsilon\in(0,1)\). We will show that the conclusion holds for a
sufficiently small choice of \(c_{\rm fm}(\varepsilon)\).

Fix \(t\in J\) and \(y\in\mathbb R^{d-1}\), and write
\[
u:=\sum_{\alpha=1}^{d-1} y^\alpha V_\alpha(t)\in T_{\gamma(t)}M,
\qquad
\rho:=\|u\|=\|y\|_2.
\]
If \(y=0\), then
\[
\Phi_\gamma(t,0)=\gamma(t),
\qquad
D\Phi_\gamma(t,0)(a,b)
=
a\,\dot\gamma(t)+\sum_{\alpha=1}^{d-1} b^\alpha V_\alpha(t),
\]
and since \(\dot\gamma(t),V_1(t),\dots,V_{d-1}(t)\) form an orthonormal basis of
\(T_{\gamma(t)}M\), we have
\[
\bigl\|D\Phi_\gamma(t,0)(a,b)\bigr\|_g^2
=
a^2+\|b\|_2^2.
\]
So the claim is trivial when \(y=0\). Henceforth assume \(y\neq 0\).

Consider the geodesic
\[
\sigma(r):=\Phi_\gamma(t,ry)=\exp_{\gamma(t)}(r u),
\qquad r\in[0,1].
\]
This is a geodesic with constant speed \(|\sigma'(r)|=\rho\).

Now fix \(v=(a,b)\in \mathbb R\times\mathbb R^{d-1}\). Define a two-parameter
variation by
\[
F(s,r)
:=
\Psi_\gamma(t+s a, r(y+s b))
=
\exp_{\gamma(t+s a)}\!\left(
  r\sum_{\alpha=1}^{d-1}(y^\alpha+s b^\alpha)\,V_\alpha(t+s a)
\right),
\qquad (s,r)\in(-\eta,\eta)\times[0,1],\ \eta>0.
\]
Then
\[
F(0,r)=\sigma(r),
\]
and
\[
J(r):=\partial_s F(s,r)\big|_{s=0}
\]
is a Jacobi field along \(\sigma\). By construction,
\[
J(1)=D\Phi_\gamma(t,y)(a,b).
\]

We next compute the initial data of \(J\). Since
\[
F(s,0)=\gamma(t+s a),
\]
we have
\[
J(0)=a\,\dot\gamma(t).
\]
Also,
\[
\partial_r F(s,0)
=
\sum_{\alpha=1}^{d-1}(y^\alpha+s b^\alpha)V_\alpha(t+s a),
\]
so using the torsion-free property of the Levi-Civita connection and the fact
that the frame \(V_\alpha\) is parallel along \(\gamma\),
\[
J'(0)
=
\nabla_r J(0)
= 
\nabla_r \partial_s F(s,r)\big|_{(s,r)=(0,0)}
=
\nabla_s \partial_r F(s,r)\big|_{(s,r)=(0,0)}
\]
Next, 
\begin{align*} 
\nabla_s \left(\sum_{\alpha=1}^{d-1}(y^\alpha+s b^\alpha)V_\alpha(t+s a)\right)\bigg|_{s=0}
=
\sum_{\alpha=1}^{d-1} b^\alpha V_\alpha(t)
+ \sum_{\alpha=1}^{d-1} y^\alpha \nabla_s V_\alpha(t+s a)\big|_{s=0}
= 
\sum_{\alpha=1}^{d-1} b^\alpha V_\alpha(t)\,.
\end{align*}

Thus
\[
\|J(0)\|=|a|,
\qquad
\|J'(0)\|=\|b\|_2,
\qquad
\langle J(0),J'(0)\rangle=0.
\]

Choose a parallel orthonormal frame
\[
E_0(r),E_1(r),\dots,E_{d-1}(r)
\]
along \(\sigma\) such that
\[
E_0(0)=\dot\gamma(t),
\qquad
E_1(0)=\frac{u}{\rho},
\]
and \(E_2(0),\dots,E_{d-1}(0)\) complete to an orthonormal basis of
\(T_{\gamma(t)}M\). Since \(\sigma'(0)=u=\rho E_1(0)\) and both \(E_1\) and
\(\sigma'/\rho\) are parallel along \(\sigma\), we have
\[
E_1(r)=\frac{\sigma'(r)}{\rho}
\qquad\text{for all }r.
\]

Write
\[
J(r)=\sum_{i=0}^{d-1} X_i(r)E_i(r),
\qquad
X(r)\in\mathbb R^d.
\]
Then \(X\) satisfies the Jacobi equation in matrix form
\[
X''(r)+A(r)X(r)=0,
\]
where
\[
A_{ij}(r)
=
\big\langle R(E_j(r),\sigma'(r))\sigma'(r),E_i(r)\big\rangle.
\]
Since the sectional curvature bound and \(|\sigma'(r)|=\rho\), the operator norm of
\(A(r)\) is bounded by
\[
\|A(r)\|\le \kappa\rho^2
\qquad
\text{for all }r\in[0,1].
\]
Moreover, because \(E_1(r)=\sigma'(r)/\rho\), the first row and column
corresponding to the radial direction vanish, but we will only use the above
operator norm bound.

Let
\[
R_0:=\sqrt{a^2+\|b\|_2^2}.
\]
With respect to the chosen frame, the initial data become
\[
X(0)=a e_0,
\qquad
X'(0)=\beta,
\]
where \(\beta\in\mathbb R^d\) satisfies
\[
\langle \beta,e_0\rangle=0,
\qquad
\|\beta\|=\|b\|_2.
\]
Integrating the Jacobi equation twice gives the Volterra equation
\[
X(r)= ae_0 +  r\beta
-\int_0^r (r-s)A(s)X(s)\,ds.
\]
Set
\[
M:=\sup_{0\le r\le 1}\|X(r)\|.
\]
Using the above identity and the bound on \(A\), we obtain
\[
M
\le
R_0+\int_0^1 (1-s)\|A(s)\|\,M\,ds
\le
R_0+\frac12\kappa\rho^2 M.
\]
Choose \(c_{\rm fm}(\varepsilon)>0\) so small that
\[
\kappa\rho^2\le \pi^2 c_{\rm fm}(\varepsilon)^2\le \frac{\varepsilon}{3}
\qquad
\text{whenever }\rho\le c_{\rm fm}(\varepsilon)\rM.
\]
(If \(\kappa=0\), then \(A\equiv 0\) and the argument is even simpler.)
Then
\[
M\le 2R_0.
\]

Substituting this back into the Volterra equation at \(r=1\) yields
\[
\|X(1)-ae_0 -\beta\|
\le
\int_0^1 (1-s)\|A(s)\|\,\|X(s)\|\,ds
\le
\frac12\kappa\rho^2 M
\le
\kappa\rho^2 R_0.
\]
By the choice of \(c_{\rm fm}(\varepsilon)\), we have
\[
\kappa\rho^2\le \frac{\varepsilon}{3},
\]
and hence
\[
\|X(1)-ae_0 -\beta\|\le \frac{\varepsilon}{3} R_0.
\]
Therefore
\[
\left(1-\frac{\varepsilon}{3}\right) R_0
\le
\|X(1)\|
\le
\left(1+\frac{\varepsilon}{3}\right) R_0.
\]
Since \(\varepsilon\in(0,1)\), we have
\[
\left(1-\frac{\varepsilon}{3}\right)^2\ge 1-\varepsilon,
\qquad
\left(1+\frac{\varepsilon}{3}\right)^2\le 1+\varepsilon.
\]
Therefore,
\[
(1-\varepsilon) R_0^2
\le
\|X(1)\|^2
\le
(1+\varepsilon)R_0^2.
\]

Finally, since \(J(1)=D\Phi_\gamma(t,y)(a,b)\), we have
\[
\|X(1)\|^2=\|J(1)\|_g^2
=
(\Phi_\gamma^*g)_{(t,y)}((a,b),(a,b)),
\]
while
\[
R_0^2=a^2+\|b\|_2^2.
\]
This proves the claim.
\end{proof}

\begin{proof}[Proof of Lemma~\ref{lem:fermi-short-window}]
Let \(c_{\rm fm}\in(0,1)\) be a sufficiently small absolute constant.

Let \(t_0\) be the midpoint of \(I\), and set \(q:=\gamma(t_0)\). Shrinking
\(c_{\rm fm}\) if necessary, we may assume that
\[
\Phi_\gamma(\Omega)\subset B_M(q,c_{\exp}\rM),
\]
where \(c_{\exp}\) is the constant from the normal-coordinate metric comparison
for \(\exp_q\). Indeed, for \((t,y)\in \Omega\),
\[
\D{q}{\Phi_\gamma(t,y)}
\le
\D{q}{\gamma(t)}+\D{\gamma(t)}{\Phi_\gamma(t,y)}
=
|t-t_0|+\|y\|_2
\lesssim c_{\rm fm}\rM.
\]

Thus the map
\[
F:=\log_q\circ \Phi_\gamma:\Omega\to T_qM\cong \mathbb R^d
\]
is well-defined. By the assumed differential estimate,
\[
\|DF(z)-I\|\le \frac{1}{10}
\qquad\forall z\in\Omega.
\]
Hence, for every \(v\in\mathbb R^d\),
\[
\frac{9}{10}\|v\|
\le
\|DF(z)v\|
\le
\frac{11}{10}\|v\|.
\]

Now let \(z,z'\in \Omega\). Since \(\Omega\) is convex,
\[
\ell(s):=z'+s(z-z'),
\qquad s\in[0,1],
\]
is contained in \(\Omega\). Therefore
\[
F(z)-F(z')
=
\int_0^1 DF(\ell(s))(z-z')\,ds,
\]
and so
\[
\frac{9}{10}\|z-z'\|
\le
\|F(z)-F(z')\|
\le
\frac{11}{10}\|z-z'\|.
\]
Thus \(F\) is bi-Lipschitz on \(\Omega\).

Finally, since \(\Phi_\gamma(\Omega)\subset B_M(q,c_{\exp}\rM)\), the normal
coordinate chart at \(q\) is \(2\)-bi-Lipschitz there:
\[
\frac12 \|F(z)-F(z')\|
\le
\D{\Phi_\gamma(z)}{\Phi_\gamma(z')}
\le
2 \|F(z)-F(z')\|.
\]
Combining the last two displays, and shrinking \(c_{\rm fm}\) slightly if
necessary, we obtain
\[
\frac12 \|z-z'\|
\le
\D{\Phi_\gamma(z)}{\Phi_\gamma(z')}
\le
2 \|z-z'\|.
\]

Together with the injectivity part proved earlier, this completes the proof of
Lemma~\ref{lem:fermi-short-window}.
\end{proof}

\subsection{Additional triangle comparison near \texorpdfstring{$\theta = 0,\pi$}{theta=0,pi} in Section }
\begin{proof}[Proof of Lemma~\ref{lem:angle-estimate-geo}]
    \step{Reduction to the model space}
    Let $\theta_\kappa$ be the comparison angle of $\theta$ in the model space $M_\kappa$ with curvature $\kappa$. By Lemma~\ref{_lem: tri_lem}, we have $\theta_\kappa \ge \theta$. Thus, an upper bound on $\theta_\kappa$ gives an upper bound on $\theta$.

    Fix $a,c$, and for $\varphi\in[0,\theta_\kappa]$ let $b_\kappa(\varphi)$ be the adjacent side length in $M_\kappa$ when the angle opposite to side $c$ is $\varphi$. Then
    \[
    b_\kappa(0)=a-c,\qquad b_\kappa(\theta_\kappa)=b.
    \]

    \[
    \tilde a:=\sqrt{\kappa}\,a,\qquad
    \tilde b(\varphi):=\sqrt{\kappa}\,b_\kappa(\varphi),\qquad
    \tilde c:=\sqrt{\kappa}\,c.
    \]

    \begin{equation}
    \label{eq:angle-estimate-geo-1}
    \cos\tilde c=\cos\tilde a\cos\tilde b(\varphi)+\sin\tilde a\sin\tilde b(\varphi)\cos\varphi.
    \end{equation}

    Differentiating \eqref{eq:angle-estimate-geo-1} with respect to $\varphi$, while fixing $a,c$ (hence
    $\tilde a,\tilde c$), gives
    \[
    0
    =
    -\cos\tilde a\sin\tilde b(\varphi)\,\tilde b'(\varphi)
    +\sin\tilde a\cos\tilde b(\varphi)\cos\varphi\,\tilde b'(\varphi)
    -\sin\tilde a\sin\tilde b(\varphi)\sin\varphi,
    \]
    hence
    \begin{equation}
    \label{eq:angle-estimate-geo-2}
    \tilde b'(\varphi)\bigl(\sin\tilde a\cos\tilde b(\varphi)\cos\varphi-\cos\tilde a\sin\tilde b(\varphi)\bigr)
    =
    \sin\tilde a\sin\tilde b(\varphi)\sin\varphi.
    \end{equation}
    Therefore
    \begin{equation}
    \label{eq:angle-estimate-geo-3}
    \frac{db_\kappa}{d\varphi}
    =
    \frac{1}{\sqrt{\kappa}}\frac{d\tilde b}{d\varphi}
    =
    \frac{1}{\sqrt{\kappa}}\,
    \frac{\sin\tilde a\,\sin\tilde b(\varphi)\,\sin\varphi}
    {\sin\tilde a\cos\tilde b(\varphi)\cos\varphi-\cos\tilde a\sin\tilde b(\varphi)}.
    \end{equation}

    Set
    \[
    D(\varphi):=\sin\tilde a\cos\tilde b(\varphi)\cos\varphi-\cos\tilde a\sin\tilde b(\varphi).
    \]
    From \eqref{eq:angle-estimate-geo-1},
    \[
    \cos\varphi
    =
    \frac{\cos\tilde c-\cos\tilde a\cos\tilde b(\varphi)}{\sin\tilde a\sin\tilde b(\varphi)},
    \]
    so
    \begin{equation}
    \label{eq:angle-estimate-geo-4}
    D(\varphi)
    =
    \sin\tilde a\cos\tilde b(\varphi)\cdot
    \frac{\cos\tilde c-\cos\tilde a\cos\tilde b(\varphi)}{\sin\tilde a\sin\tilde b(\varphi)}
    -\cos\tilde a\sin\tilde b(\varphi)
    =
    \frac{\cos\tilde b(\varphi)\cos\tilde c-\cos\tilde a}{\sin\tilde b(\varphi)}.
    \end{equation}

    Now apply the spherical law of cosines for side $a$ in $M_\kappa$:
    \begin{equation}
    \label{eq:angle-estimate-geo-5}
    \cos\tilde a
    =
    \cos\tilde b(\varphi)\cos\tilde c+\sin\tilde b(\varphi)\sin\tilde c\cos\alpha_\kappa(\varphi),
    \end{equation}
    where $\alpha_\kappa(\varphi)$ is the angle opposite to side $a$ in the model triangle at parameter $\varphi$.
    Thus
    \[
    \cos\tilde b(\varphi)\cos\tilde c-\cos\tilde a
    =
    -\sin\tilde b(\varphi)\sin\tilde c\cos\alpha_\kappa(\varphi),
    \]
    and \eqref{eq:angle-estimate-geo-4} becomes
    \begin{equation}
    \label{eq:angle-estimate-geo-6}
    D(\varphi)=-\sin\tilde c\cos\alpha_\kappa(\varphi).
    \end{equation}

    Substituting \eqref{eq:angle-estimate-geo-6} into \eqref{eq:angle-estimate-geo-3}, we obtain
    \[
    \frac{\partial b_\kappa}{\partial\varphi}
    =
    \frac{1}{\sqrt{\kappa}}\,
    \frac{\sin\tilde a\,\sin\tilde b(\varphi)\,\sin\varphi}
    {-\sin\tilde c\,\cos\alpha_\kappa(\varphi)}.
    \]

   \step{2. Estimate of $\alpha_\kappa(\varphi)$}
   Fix $\varphi\in[0,\theta_\kappa]$. We apply Lemma \ref{_lem: M-opposite-side} with
   $a' = b_\kappa(\varphi)$, $b' = c$, $c' = a$, and $\phi = \alpha_\kappa(\varphi)$.
   Since $c \le 0.01 a$ and $b_\kappa(\varphi)\ge a-c\ge 0.99a$, the condition
   $b'\le \tfrac14 a'$ is satisfied. Hence
   \begin{align*}
   \bigl|\,b_\kappa(\varphi)-a-c\cos\alpha_\kappa(\varphi)\,\bigr|
   \le
   4\frac{c^2}{b_\kappa(\varphi)}.
   \end{align*}
   Rearranging gives
   \[
   b_\kappa(\varphi)-(a-c)
   \ge
   c\bigl(1+\cos\alpha_\kappa(\varphi)\bigr)-4\frac{c^2}{b_\kappa(\varphi)}.
   \]
   Also,
   \[
   b_\kappa(\varphi)\ge a-c\ge 0.99a,
   \]
   hence
   \[
   4\frac{c^2}{b_\kappa(\varphi)}
   \le
   \frac{4}{0.99}\frac{c^2}{a}
   \le
   \frac{4}{99}c
   < 0.1\,c.
   \]
   Therefore
   \[
   b_\kappa(\varphi)-(a-c)\ge c\bigl(0.9+\cos\alpha_\kappa(\varphi)\bigr).
   \]
   On the regime $b_\kappa(\varphi)-(a-c)\le \tfrac14 c$, we obtain
   \[
   0.9+\cos\alpha_\kappa(\varphi)\le \tfrac14
   \qquad\Longrightarrow\qquad
   \cos\alpha_\kappa(\varphi)\le -0.65<0.
   \]
   So $\alpha_\kappa(\varphi)\in(\pi/2,\pi)$. Along this branch,
   \[
   \frac{\partial b_\kappa}{\partial\varphi}
   \ge
   \frac{1}{\sqrt{\kappa}}\,\frac{\sin\tilde a\,\sin(\tilde a-\tilde c)\,\sin\varphi}{\sin\tilde c}.
   \]
   Therefore, integrating from $\varphi=0$ to $\varphi=\theta_\kappa$ gives
   $$
   b-(a-c)
   \ge
   \frac{1}{\sqrt{\kappa}}\,\frac{\sin\tilde a\,\sin(\tilde a-\tilde c)}{\sin\tilde c}(1-\cos\theta_\kappa)
   \ge
   c_{\tref{lem:angle-estimate-geo}} \frac{a^2}{c}\theta^2.
   $$
\end{proof}

\begin{proof}[Proof of Lemma~\ref{lem:angle-near-pi-defect}]
Let
\[
\alpha_{-\kappa}:=\ang{-\kappa}{b,c,a}{}{}
\]
be the comparison angle in $M_{-\kappa}$. By Lemma~\ref{_lem: tri_lem},
$\alpha\ge \alpha_{-\kappa}$.
Fix $b,c$ and let $a=a(\phi)$ be the opposite side length in $M_{-\kappa}$
as a function of the opposite angle $\phi$. The hyperbolic law of cosines gives
\[
\cosh(\sqrt{\kappa}\,a(\phi))
=
\cosh(\sqrt{\kappa}\,b)\cosh(\sqrt{\kappa}\,c)
-\sinh(\sqrt{\kappa}\,b)\sinh(\sqrt{\kappa}\,c)\cos\phi.
\]
Differentiating:
\[
\sqrt{\kappa}\,\sinh\!\big(\sqrt{\kappa}\,a(\phi)\big)\,a'(\phi)
=
\sinh(\sqrt{\kappa}\,b)\sinh(\sqrt{\kappa}\,c)\sin\phi,
\]
so
\[
a'(\phi)
=
\frac{1}{\sqrt{\kappa}}\,
\frac{\sinh(\sqrt{\kappa}\,b)\sinh(\sqrt{\kappa}\,c)\sin\phi}
{\sinh\!\big(\sqrt{\kappa}\,a(\phi)\big)}.
\]
Hence $a'(\phi)\ge 0$ on $[0,\pi]$, so $a(\phi)$ is nondecreasing.
Also, $a(\phi)\le b+c$, thus
\[
a'(\phi)
\ge
\frac{1}{\sqrt{\kappa}}\,
\frac{\sinh(\sqrt{\kappa}\,b)\sinh(\sqrt{\kappa}\,c)}
{\sinh\!\big(\sqrt{\kappa}\,(b+c)\big)}
\sin\phi.
\]
Since $a(\alpha_{-\kappa})=a$, integrating from $\alpha_{-\kappa}$ to $\pi$ gives
\[
(b+c)-a
=
a(\pi)-a(\alpha_{-\kappa})
=
\int_{\alpha_{-\kappa}}^{\pi} a'(\phi)\,d\phi
\ge
\frac{1}{\sqrt{\kappa}}\,
\frac{\sinh(\sqrt{\kappa}\,b)\sinh(\sqrt{\kappa}\,c)}
{\sinh\!\big(\sqrt{\kappa}\,(b+c)\big)}
\int_{\alpha_{-\kappa}}^{\pi}\sin\phi\,d\phi
\]
\[
=
\frac{1}{\sqrt{\kappa}}\,
\frac{\sinh(\sqrt{\kappa}\,b)\sinh(\sqrt{\kappa}\,c)}
{\sinh\!\big(\sqrt{\kappa}\,(b+c)\big)}
\bigl(1+\cos\alpha_{-\kappa}\bigr).
\]
Now
\[
1+\cos\alpha_{-\kappa}
=
1-\cos(\pi-\alpha_{-\kappa})
\ge
\frac{2}{\pi^2}(\pi-\alpha_{-\kappa})^2,
\]
so
\[
(b+c)-a
\ge
\frac{2}{\pi^2\sqrt{\kappa}}\,
\frac{\sinh(\sqrt{\kappa}\,b)\sinh(\sqrt{\kappa}\,c)}
{\sinh\!\big(\sqrt{\kappa}\,(b+c)\big)}
(\pi-\alpha_{-\kappa})^2.
\]
The comparison inequality
$\alpha\ge\alpha_{-\kappa}$, i.e. $(\pi-\alpha_{-\kappa})^2\ge(\pi-\alpha)^2$,
gives
\[
(b+c)-a
\ge
\frac{2}{\pi^2\sqrt{\kappa}}\,
\frac{\sinh(\sqrt{\kappa}\,b)\sinh(\sqrt{\kappa}\,c)}
{\sinh\!\big(\sqrt{\kappa}\,(b+c)\big)}
(\pi-\alpha)^2.
\]
Finally, because $b,c\le \rG$, so that $\sqrt{\kappa}\,b,\sqrt{\kappa}\,c, \sqrt{\kappa}\,(b+c)$ are small constants bounded from above by $1/16$, together with $x \le \sinh x \le \frac{e}{2} x$ for $x\in[0,1]$, thus
\[
\frac{2}{\pi^2\sqrt{\kappa}}\,
\frac{\sinh(\sqrt{\kappa}\,b)\sinh(\sqrt{\kappa}\,c)}
{\sinh\!\big(\sqrt{\kappa}\,(b+c)\big)}
\ge
c_{\tref{lem:angle-near-pi-defect}}\,\frac{bc}{b+c},
\]
and the proof follows. 

\end{proof}
\begin{proof}[Sketch Lemma~\ref{lem:fermi-short-window}]
First, the injectivity part of the lemma simply follows from the uniqueness of projection to the geodesic, see Corollary \ref{cor:local-proj-geodesic}. So it works on  
$$
    I \times B_{\R^{d-1}}(0,\rho) 
$$
whenever $|I| \le \frac{1}{4}\rM$ and $\rho\le \frac{1}{4}\rM$.  

Now we prove the bi-Lipschitz estimate on the distance. 
Fix a small \(\varepsilon\in(0,1)\), say \(\varepsilon=\tfrac1{10}\), and let
\(c_{\rm fm}(\varepsilon)\) be given by
Lemma~\ref{lem:fermi-metric-comparison}. We choose a smaller constant
\(c_{\rm tub}>0\) so that
\[
10c_{\rm tub}\le c_{\rm fm}(\varepsilon).
\]

Let \(I\subset J\) be an interval of length at most \(c_{\rm tub}\rM\), and set
\[
\rho:=c_{\rm tub}\rM.
\]
Define the inner domain
\[
\Omega_{\rm in}:=I\times B_{\R^{d-1}}(0,\rho).
\]
Let
\[
H:=|I|+2\rho.
\]
Choose a larger interval \(I'\subset J\), with the same midpoint as \(I\), such
that
\[
\dist(I,\partial I')=H,
\]
and define
\[
\rho':=\rho+H,\qquad
\Omega_{\rm out}:=I'\times B_{\R^{d-1}}(0,\rho').
\]
By the choice of \(c_{\rm tub}\), both the longitudinal size of \(I'\) and the
transverse radius \(\rho'\) are bounded by \(c_{\rm fm}(\varepsilon)\rM\), so
Lemma~\ref{lem:fermi-metric-comparison} applies on \(\Omega_{\rm out}\).

Let
\[
x=\Phi_\gamma(z),\qquad x'=\Phi_\gamma(z'),
\qquad z,z'\in\Omega_{\rm in}.
\]
We first prove the upper bound. Since \(\Omega_{\rm in}\) is convex, the straight
segment \([z,z']\) lies in \(\Omega_{\rm in}\), and therefore
\[
\D{x}{x'}
\le
L_g\!\bigl(\Phi_\gamma([z,z'])\bigr)
\le
\sqrt{1+\varepsilon}\,|z-z'|.
\]

For the lower bound, let \(\sigma\) be a minimizing geodesic from \(x\) to
\(x'\). We claim that \(\sigma\subset \Phi_\gamma(\Omega_{\rm out})\). Indeed,
suppose \(\sigma\) leaves \(\Phi_\gamma(\Omega_{\rm out})\). Let \(x_-\) be the
first exit point and \(x_+\) the last re-entry point. Since \(\Phi_\gamma\) is
injective on \(\Omega_{\rm out}\), the portions of \(\sigma\) from \(x\) to
\(x_-\) and from \(x_+\) to \(x'\) can be pulled back to curves in
\(\Omega_{\rm out}\) joining \(z\) and \(z'\) to \(\partial\Omega_{\rm out}\),
respectively. By construction of \(\Omega_{\rm out}\), each such curve has
Euclidean length at least \(H\). Hence Lemma~\ref{lem:fermi-metric-comparison}
gives
\[
L_g(\sigma)\ge 2\sqrt{1-\varepsilon}\,H.
\]

On the other hand, there is an explicit path inside \(\Omega_{\rm in}\) joining
\(z\) to \(z'\) obtained by moving radially to the geodesic, then along the
geodesic, and then radially back out. Its Euclidean length is at most
\[
|I|+2\rho=H,
\]
so its image under \(\Phi_\gamma\) has \(g\)-length at most
\[
\sqrt{1+\varepsilon}\,H.
\]
Since
\[
2\sqrt{1-\varepsilon}>\sqrt{1+\varepsilon},
\]
this contradicts the minimality of \(\sigma\). Thus every minimizing geodesic
from \(x\) to \(x'\) stays inside \(\Phi_\gamma(\Omega_{\rm out})\).

We may therefore pull \(\sigma\) back via \(\Phi_\gamma^{-1}\) to a curve
\(\widetilde\sigma\subset\Omega_{\rm out}\) joining \(z\) to \(z'\). Applying
Lemma~\ref{lem:fermi-metric-comparison} to \(\widetilde\sigma\), we obtain
\[
\D{x}{x'}
=
L_g(\sigma)
\ge
\sqrt{1-\varepsilon}\,L_{\rm Eucl}(\widetilde\sigma)
\ge
\sqrt{1-\varepsilon}\,|z-z'|.
\]

Combining the two bounds yields
\[
\sqrt{1-\varepsilon}\,|z-z'|
\le
\D{\Phi_\gamma(z)}{\Phi_\gamma(z')}
\le
\sqrt{1+\varepsilon}\,|z-z'|.
\]
After fixing \(\varepsilon=\tfrac1{10}\) and weakening constants, this gives the
desired bi-Lipschitz estimate.
\end{proof}

\end{document}